\documentclass[twoside,11pt]{article}
\usepackage[preprint]{jmlr2e}
\usepackage{amsfonts,amsmath,amssymb,url,xspace,bm,mathtools,enumitem,subcaption,multirow}
\usepackage[colorinlistoftodos,prependcaption,textsize=tiny]{todonotes}
\allowdisplaybreaks

\newtheorem{assumption}[theorem]{Assumption}
\usepackage{dsfont}

\newcommand{\LD}{\langle}
\newcommand{\RD}{\rangle}

\newcommand{\mR}{\mathbb{R}}
\newcommand{\mE}{\mathbb{E}}

\newcommand{\mN}{\mathcal{N}}

\newcommand{\mA}{\mathcal{A}}
\newcommand{\mS}{\mathcal{S}}
\newcommand{\mF}{\mathcal{F}}

\newcommand{\0}{\boldsymbol{0}}

\newcommand{\bbeta}{\bm{\beta}}

\newcommand{\bTheta}{\bm{\Theta}}

\newcommand{\tbeta}{\bbeta^\star}

\newcommand{\tQ}{{\bQ^\star}}
\newcommand{\tS}{{\bS^\star}}
\newcommand{\tR}{{\bR^\star}}

\newcommand{\bari}{\bar{i}}

\newcommand{\TR}{{{\rm Trace}}}

\newcommand{\hUb}{\hat{\Ub}}
\newcommand{\hVb}{\hat{\Vb}}
\newcommand{\bxb}{\bar{\xb}}
\newcommand{\bzb}{\bar{\zb}}

\newcommand{\rbr}[1]{\left(#1\right)}
\newcommand{\sbr}[1]{\left[#1\right]}
\newcommand{\cbr}[1]{\left\{#1\right\}}

\newcommand{\abr}[1]{\left|#1\right|}

\def\B{\mathcal{B}}
\def\D{\mathcal{D}}
\def\I{\mathcal{I}}
\def\J{\mathcal{J}}
\def\T{\mathcal{T}}

\def\mL{\mathcal{L}}

\newcommand{\eb}{\mathbf{e}}

\newcommand{\xb}{\mathbf{x}}

\newcommand{\zb}{\mathbf{z}}

\newcommand{\ba}{\bm{a}}
\newcommand{\bb}{\bm{b}}

\newcommand{\bd}{\bm{d}}

\newcommand{\bp}{\bm{p}}
\newcommand{\bq}{\bm{q}}
\newcommand{\br}{\bm{r}}
\newcommand{\bs}{\bm{s}}
\newcommand{\bt}{\bm{t}}
\newcommand{\bu}{\bm{u}}
\newcommand{\bv}{\bm{v}}

\newcommand{\bx}{\bm{x}}

\newcommand{\bQ}{\bm{Q}}
\newcommand{\bR}{\bm{R}}
\newcommand{\bS}{\bm{S}}

\newcommand{\Ab}{\mathbf{A}}
\newcommand{\Bb}{\mathbf{B}}

\newcommand{\Hb}{\mathbf{H}}

\newcommand{\Jb}{\mathbf{J}}

\newcommand{\Qb}{\mathbf{Q}}
\newcommand{\Rb}{\mathbf{R}}

\newcommand{\Tb}{\mathbf{T}}
\newcommand{\Ub}{\mathbf{U}}
\newcommand{\Vb}{\mathbf{V}}

\newcommand{\tUb}{{\Ub^\star}}
\newcommand{\tVb}{{\Vb^\star}}
\newcommand{\tTheb}{{\bTheta^\star}}
\newcommand{\tThe}{\bTheta^\star}

\newcommand{\tThep}{\bThe^{\star\prime}}
\newcommand{\tuT}{\bu^{\star T}}
\newcommand{\tvT}{\bv^{\star T}}
\newcommand{\tu}{\bu^{\star}}
\newcommand{\tv}{\bv^{\star}}
\newcommand{\tdp}{\bd^{\star \prime}}
\newcommand{\tpp}{\bp^{\star \prime}}
\newcommand{\td}{\bd^{\star}}
\newcommand{\tdT}{\bd^{\star T}}
\newcommand{\tp}{\bp^{\star}}
\newcommand{\tpT}{\bp^{\star T}}
\newcommand{\bThe}{\bTheta}
\newcommand{\barkap}{\bar{\kappa}}
\def\b1{\pmb{1}}
\def\Var{\text{Var}}
\def\VEC{\text{vec}}

\newcommand{\bbt}{\bar{\bm{t}}}

\newcommand{\diag}{{\rm diag}}

\newcommand{\cbc}{c_{\cal B}}
\newcommand{\Cbc}{C_{\cal B}}

\jmlrheading{}{2020}{}{02/2020}{}{}{Sen Na, Yuwei Luo, Zhuoran Yang, Zhaoran Wang, and Mladen Kolar}

\ShortHeadings{Semiparametric Nonlinear Bipartite Graph Representation Learning}{S. Na et al.}
\firstpageno{1}

\begin{document}

\title{Semiparametric Nonlinear Bipartite Graph Representation Learning with Provable Guarantees}

\author{\name Sen Na \email senna@uchicago.edu\\
\name Yuwei Luo \email yuweiluo@uchicago.edu \\
\addr Department of Statistics\\
University of Chicago\\
Chicago, IL 60637, USA
\AND
\name Zhuoran Yang \email zy6@princeton.edu \\
\addr Department of Operations Research and Financial Engineering\\
Princeton University\\
Princeton, NJ 08544, USA
\AND
\name Zhaoran Wang \email zhaoran.wang@northwestern.edu \\
\addr Department of Industrial Engineering and Management Sciences\\
Northwestern University\\
Evanston, IL 60208, USA
\AND
\name Mladen Kolar \email mladen.kolar@chicagobooth.edu \\
\addr Booth School of Business \\
University of Chicago\\
Chicago, IL 60637, USA
}

\editor{}

\maketitle

\begin{abstract}
Graph representation learning is a ubiquitous task in machine learning where the goal is to embed each vertex into a low-dimensional vector space. We consider the bipartite graph and formalize its representation learning problem as a statistical estimation problem of parameters in a semiparametric exponential family distribution. The bipartite graph is assumed to be generated by a semiparametric exponential family distribution, whose parametric component is given by the proximity of outputs of two one-layer neural networks, while nonparametric (nuisance) component is the base measure. Neural networks take high-dimensional features as inputs and output embedding vectors. In this setting, the representation learning problem is equivalent to recovering the weight matrices. The main challenges of estimation arise from the nonlinearity of activation functions and the nonparametric nuisance component of the distribution. To overcome these challenges, we propose a pseudo-likelihood objective based on the rank-order decomposition technique and focus on its local geometry. We show that the proposed objective is strongly convex in a neighborhood around the ground truth, so that a gradient descent-based method achieves linear convergence rate. Moreover, we prove that the sample complexity of the problem is linear in dimensions (up to logarithmic factors), which is consistent with parametric Gaussian models. However, our estimator is robust to any model misspecification within the exponential family, which is validated in extensive experiments.
\end{abstract}


\section{Introduction}\label{sec:1}

Graphs naturally arise as models in a variety of applications, ranging from social networks \citep{Scott1988Social} and molecular biology \citep{Higham2008Fitting} to recommendation systems \citep{Ma2018Graph} and transportation \citep{Bell1997Transportation}. In a variety of problems, graphs tend to be high-dimensional and highly entangled, and hence difficult to directly learn from. As a prominent remedy, graph representation learning aims to learn a mapping that represents each vertex as low-dimensional vector such that structural properties of the original graph are preserved. Those learned low-dimensional representations, also called embeddings, are further used as the input features in downstream machine learning tasks, such as link prediction \citep{Taskar2004Link, AlHasan2011survey}, node classification \citep{Bhagat2011Node}, and community detection \citep{Fortunato2010Community}.

There are three major approaches to graph embedding: matrix factorization-based algorithms \citep{Belkin2002Laplacian, Ahmed2013Distributed}, random walk algorithms \citep{Perozzi2014Deepwalk, Grover2016node2vec}, and graph neural networks \citep{Scarselli2008graph, Zhou2018Graph, Wu2019comprehensivea}. These approaches can be unified via the encoder-decoder framework proposed in \citet{Hamilton2017Representation}. In this framework, the encoder is a mapping that projects each vertex or a subgraph to a low-dimensional vector, whereas the decoder is a probability model that infers the \textit{structural information} of the graph from the embeddings generated by the encoder. The structural information here depends on the specific downstream tasks of interest, which also determine the loss function of the decoder. The desired graph representations are hence obtained by minimizing the loss function as a function of embedding vectors. For example, in the link prediction task, the decoder predicts whether an edge between two vertices exists or not using a Bernoulli model and logistic loss function, and the model parameter is a function of embeddings \citep{Baldin2018Optimal}.

Such an encoder-decoder architecture motivates the study of graph representation learning through the lens of statistical estimation for generative models. In particular, suppose the observed graph is generated by a statistical model specified by the decoder with \textit{true} graph representations as its inputs. We can then assess the performance of a graph embedding algorithm by examining the difference between the learned representation and the ground truth. \citet{Baldin2018Optimal} adopted this perspective to study the performance of a linear embedding method for the link prediction problem. The validity of their results hinges on the condition that both the linear model of the encoder and the Bernoulli model of the decoder are correctly specified. When either of these assumptions are violated, they would incur large estimation error.  Recent advances in graph representation learning are attributed to more flexible decoders \citep{Cho2014Learning, Goodfellow2016Deep, Badrinarayanan2017Segnet}, which are based on deep neural networks and can handle graphs with edge attributes that can be categorical. These approaches are poorly understood from a theoretical point of view.

In the present paper, we focus on \textit{bipartite graphs}, where there are two distinct sets of vertices, $U$ and $V$, and only edges between two vertices in different sets are allowed.  We study the semiparametric nonlinear bipartite graph representation learning problem under the encoder-decoder framework. We assume that each vertex $u\in U$ is associated with a high-dimensional Gaussian vector $\xb_u\in\mR^{d_1}$. Similarly, each vertex $v\in V$ is associated with a high-dimensional Gaussian vector $\zb_v\in\mR^{d_2}$. The encoder maps them via one-layer neural networks to low-dimensional vectors $\phi_1(\tUb^T \xb_u), \phi_2(\tVb^T\zb_v) \in \mR^r$, where $\tUb\in\mR^{d_1\times r}$, $\tVb\in\mR^{d_2\times r}$ are weight matrices, $\{\phi_i\}_{i = 1, 2}$ are activation functions evaluated entrywise, and $r \ll (d_1\wedge d_2)$. Furthermore, in the decoder, we consider the link prediction task under a semiparametric model. In particular, we assume that the attribute of an edge follows a natural exponential family distribution parameterized by the proximity between two vertices, which is defined as the inner product $\LD \phi_1(\tUb^T \xb_u), \phi_2(\tVb^T\zb_v)\RD$ between the embedding vectors. Here, $\phi_1(\tUb^T \xb_u)$ is the embedding vector of $u$, while $\phi_2(\tVb^T\zb_v)$ is the embedding vector of $v$. We do not specify the base measure of the exponential family distribution but, instead, treat it as a nuisance parameter. This gives us a semiparametric model for the decoder and robustness to model misspecification within the exponential family.

In the above described semiparametric nonlinear model, our goal is to recover weight matrices $\tUb$ and $\tVb$. Based on these weight matrices, we can then compute embeddings for all vertices. There are two main obstacles that make the estimation problem challenging. First, while the activation functions $\{\phi_i\}_{i = 1,2}$ make the encoder model more flexible, their nonlinearity leads to a loss function that is nonconvex and nonsmooth. Second, while the unknown nonparametric nuisance component of the decoder model makes the graph representation learning robust to the model misspecification, it also makes the likelihood function not available. To overcome these obstacles, we propose a pseudo-likelihood objective, which is minimized at $(\tUb, \tVb)$ locally. We analyze the landscape of the empirical objective and show that, in a neighborhood around the ground truth, the objective is strongly convex. Therefore, the vanilla gradient descent (GD) achieves linear convergence rate. Moreover, we prove that the sample complexity is linear in dimensions $d_1\vee d_2$, up to logarithmic factors, which matches the best known result under the parametric model \citep{Zhong2019Provable}. Experiments on synthetic and real data corroborate our theoretical results and illustrate flexibility of the proposed representation learning model.

\paragraph{Notations.} For any positive integer $n$, $[n] = \{1,2, \ldots, n\}$ denotes the index set, and $\text{Unif}([n])$ is a uniform sampling over the indices. We write $a\lesssim b$ if $a\leq c\cdot b$ for some constant $c$, and $a\asymp b$ if $a\lesssim b$ and $b\lesssim a$. We define $\delta_{ij} = \pmb{1}_{i = j}$, which equals to $1$ if $i=j$ and $0$ otherwise. For any matrix $\Ub$, $\VEC(\Ub)$ denotes the column vector obtained by vectorizing $\Ub$ and $\|\Ub\|_{p,q} = (\sum_{j}(\sum_{i} |\Ub_{ij}|^p)^{q/p})^{1/q}$. As usual, $\|\Ub\|_F$, $\|\Ub\|_2$ refer to the Frobenius and operator norm, respectively, and $\sigma_p(\Ub)$ denotes the $p$-th singular value of $\Ub$. For a square matrix $\Ub$, $\diag(\Ub) = (\Ub_{11};\Ub_{22}; \ldots)$ is a vector including all diagonal entries of $\Ub$; when $\Ub$ is symmetric, $\lambda_{\max}(\Ub)$ ($\lambda_{\min}(\Ub)$) denotes its maximum (minimum) eigenvalue. We write $\Ab\succeq \Bb$ if $\Ab - \Bb$ is positive semidefinite and $\Ab \succ \Bb$ if it is positive definite. For any vector $\ba$, $\|\ba\|_{\min} = \min_i |\ba_i|$ is the minimal absolute value of its entries.

\paragraph{Structure of the paper.} In Section \ref{sec:2}, we
formalize the semiparametric graph representation learning problem and
introduce related work. In Section \ref{sec:3}, we present our
estimation method by proposing a pseudo-likelihood objective, and the
theoretical analysis of such objective is provided in Section
\ref{sec:4}. In Section \ref{sec:5} we show experimental results and
conclusions are summarized in Section \ref{sec:6}. Section~\ref{sec:7}
provides proofs of main technical results, while the proofs of
auxiliary results are given in the appendix.

\section{Preliminaries and related work}\label{sec:2}

We describe the setup of our problem and introduce the applications and related work. We particularly focus on the statistical literature on theory of semiparametric estimation and matrix completion, although bipartite graph representation learning has been routinely applied to varied deep neural networks \citep{Nassar2018Hierarchical, Wu2018Graph}.  We point reader to \citet{Zha2001Bipartite} for a survey on bipartite graph.

\subsection{Problem formulation}\label{sec:2:1}

Let $G = (U, V, E)$ be a bipartite graph where $U$ and $V$ are two sets of vertices and $E$ denotes the set of edges between two vertex sets. For each vertex $u\in U$, we assume it is associated with a Gaussian vector $\xb_u\in\mR^{d_1}$, while for each $v\in V$ we have a Gaussian vector $\zb_v\in\mR^{d_2}$. An edge between $u$ and $v$ has an attribute $y_{(u, v)}$ that follows the following semiparametric exponential family model \begin{align}\label{mod:semi}
P(y_{(u, v)}\mid \bTheta_{(u, v)}^\star, f) = \exp(y_{(u, v)}\cdot\bTheta_{(u, v)}^\star - b(\bTheta_{(u, v)}^\star, f) + \log f(y_{(u, v)})),
\end{align}
which is parameterized by the base measure function $f: \mR\rightarrow \mR$ and a scalar
\begin{align*}
\bTheta_{(u, v)}^\star = \LD \phi_1(\tUb^T\xb_{u}), \phi_2(\tVb^T\zb_{v})\RD.
\end{align*}
In model \eqref{mod:semi}, $b(\cdot, \cdot)$ is the log-partition
function (or normalizing function) that makes the density have unit integral. The parametric component of the exponential family, $\bTheta_{(u, v)}^\star = \tTheb(\xb_{u}, \zb_{v})$, depends on the covariate $\xb_{u}$ coming from the set $U$ and the covariate $\zb_{v}$ coming from the set $V$. The nonparametric component $f$ is treated as a nuisance parameter, which gives us flexibility in modeling the edge attributes. To make notation concise, we will drop the subscript of $\xb_u$ and $\zb_v$ hereinafter, and use $\xb$ and $\zb$ to denote covariates from set $U$ and $V$, respectively. In our analysis, the activation functions $\{\phi_i\}_{i = 1,2}$ have one of the following three forms:
\begin{align*}
\text{Sigmoid: }\phi(x) = \frac{\exp(x)}{1 + \exp(x)};\ 
\text{Tanh: }\phi(x) = \frac{\exp(x)-\exp(-x)}{\exp(x)+\exp(-x)};\ 
\text{ReLU: }\phi(x) = \max(0,x).
\end{align*}

We formalize the bipartite graph representation learning as a statistical parameter estimation problem of a generative model. In particular, suppose the graph is generated by the exponential family model \eqref{mod:semi} with some unknown base measure $f$, and we observe part of edge attributes, $y$, and associated covariates on two ends, $\xb$ and $\zb$. Thus, we obtain data set $\{(y_{ij}, \xb_i, \zb_j)\}_{i, j}$ where $i, j$ index the vertices of two sets. The graph representation learning in our setup is then equivalent to recovering $\tUb \in \mR^{d_1\times r}$ and $\tVb\in\mR^{d_2\times r}$, which can be used to compute parametric component of the decoder model and estimate embedding vectors, $\phi_1(\hUb^T\xb)$ and $\phi_2(\hVb^T\zb)$, for all vertices in two sets, since activation functions are user-chosen and known.

\subsection{Applications and related work}

Graph representation learning underlies a number of real world problems, including object recognition in image analysis \citep{Bunke1995Efficient, Fiorio1996topologically}, community
detection in social science \citep{Perozzi2014Deepwalk, Cavallari2017Learning}, and recommendation systems in machine
learning \citep{Kang2016Top, Jannach2016Recommender}. See
\citet{Bengio2013Representation}, \citet{Hamilton2017Inductive},
\citet{Hamilton2017Representation} for recent surveys and other
applications. The bipartite graph is of particular interest since it
classifies vertices into two types, which extensively appears in
modern applications.

For concreteness, in user-item recommendation systems, the attribute of an edge between a user node and an item node represents the rating, which is modeled by the proximity of projected features onto the latent space. Specifically, each user is represented by a high-dimensional feature vector $\xb$ and each item is represented by a high-dimensional feature vector $\zb$. A simple generative model for the rating $y$ that a user gives to an item is $y = \LD \tUb^T\xb, \tVb^T\zb \RD + \epsilon$ with $\epsilon \sim \mN(0, 1)$ independent from $\xb, \zb$. Such a model is studied in the inductive matrix completion (IMC) literature \citep{Abernethy2006Low, Jain2013Provable, Si2016Goal, Berg2017Graph}. \citet{Zhong2019Provable} studied nonlinear IMC problem, where a generalized model for the rating is $y = \LD \phi(\tUb^T\xb), \phi(\tVb^T\zb)\RD + \epsilon$, with $\phi(\cdot)$ being a common activation function. In this generalized nonlinear model, one-layer neural network compresses the high-dimensional features into low-dimensional
embeddings. \citet{Zhong2019Provable} proposed to minimize the squared loss to recover weight matrices $\tUb$ and $\tVb$, and established consistency for their minimizer, with linear sample complexity in dimension $d_1\vee d_2$, up to logarithmic factors.

Our work contributes to this line of research by enhancing the IMC
model from two aspects. First, we allow for two separate neural
networks to embed user and item covariates. Although this modification may seem minor, it makes theoretical analysis more challenging when two networks mismatch: one network has a smooth activation function while the other does not. Second, we consider an exponential family model with unknown base measure, which extends the applicability of the model and allows for model misspecification within the exponential family. In particular, the semiparametric setup makes our estimator independent of the specific form of $f$. For example, the model in \citet{Zhong2019Provable} is a special case of \eqref{mod:semi} with $f(y) = \exp(-y^2/2)$, while the link prediction problem in \citet{Liben-Nowell2007link} and \citet{Menon2011Link} is a special case with $f(y) = 1$.

Furthermore, our work contributes to the literature on graph embedding \citep{Qiu2018Network, Goyal2018Graph}. Our paper studies the bipartite graph and casts the graph representation learning as the problem of parameter estimation in a generative model. This setup allows us to analyze statistical properties, such as consistency and convergence rate, of the learned embedding features. To the best of our knowledge, statistical view of representation learning is missing although it was successfully used in real experiments \citep[see, e.g.,][]{Graepel2001Learning, Yang2015Network}. In addition, our work also contributes to a growing literature on semiparametric modeling \citep{Fengler2005Semiparametric, Li2008Variable, Fan2017High}, where the parametric component in \eqref{mod:semi} is given by $\tTheb = {\tbeta}^T\xb$ and the goal is to estimate $\tbeta$ by regressing $y$ on $\xb$, without knowing $f$. \citet{Fosdick2015Testing} formalized the representation learning as a latent space network model, where the parameter $\tTheb$ is given by the inner product of two latent vectors and $f(y) = \exp(-y^2/2)$, that is under a Gaussian noise setup, and proposed methodology for testing the dependence between nodal attributes and latent factors. \citet{Ma2019Universal} studied a similar model with $f(y)=1$ and proposed both convex and nonconvex approaches to recover latent factors.  However, our work is more challenging due to the nonlinearity of activation functions and the missing knowledge of $f$.

Lastly, several estimation methods for pairwise measurements have been studied in related, but simpler, models \citep{Chen2014Information, Chen2015Spectral, Chen2016Community, Pananjady2017Worst, Negahban2018Learning, Chen2018projected, Chen2019Spectral}. \citet{Chen2018Contrastive} studied model \eqref{mod:semi} by assuming the parameter matrix of the graph to be low-rank, and estimated $\tTheb$ as a whole. As a comparison, our model is more complicated since each entry of $\tTheb$ in our setup is given by the inner product of two embedding vectors, which measures the proximity of two vertices. Our task is to recover two underlying weight matrices $\tUb$, $\tVb$ that are convolved by activation functions to generate~$\tTheb$.

\section{Methodology}\label{sec:3}

We propose a pseudo-likelihood objective function to estimate the unknown weight matrices and discuss identifiability of the parameters. The objective function is minimized by the gradient descent with a constant step size. Theoretical analysis of the iterates is provided in Section \ref{sec:4}.

The likelihood of the model \eqref{mod:semi} is not available due to the presence of the infinite-dimensional nuisance parameter $f$. Using the rank-order decomposition technique \citep{Ning2017likelihood}, we focus on the pairwise differences and develop a pseudo-likelihood objective. Importantly, the differential pseudo-likelihood does not depend on $f$ and, as a result, our estimator is valid for a wide range of distributions, without having to explicitly specify them in advance.

We follow the setup described in Section \ref{sec:2:1}. To simplify
the presentation, suppose we have $2n_1$ vertices in $U$ and $2n_2$ vertices in $V$, denoted by $U = \{u_1, \ldots, u_{n_1}, u_1', \ldots, u_{n_1}'\}$ and $V = \{v_1, \ldots, v_{n_2}, v_1', \ldots, v_{n_2}'\}$, respectively. For $i\in[n_1]$ and $j\in[n_2]$, we let $\xb_i = \xb_{u_i}$, $\xb_i' = \xb_{u_i'}$, $\zb_j = \zb_{v_j}$, $\zb_j' = \zb_{v_j'}$, and suppose that $\xb_i, \xb_i' \stackrel{i.i.d}{\sim} \mN(0, I_{d_1})$ and $\zb_j, \zb_j' \stackrel{i.i.d}{\sim} \mN(0, I_{d_2})$, independent of each other. Further, we assume to observe $m$ edge attributes, $y$,
between vertices $\{u_1, \ldots, u_{n_1}\}$ and
$\{v_1, \ldots, v_{n_2}\}$, and another $m$ edge attributes, $y'$,
between $\{u_1', \ldots, u_{n_1}'\}$ and $\{v_1', \ldots, v_{n_2}'\}$,
both of which follow the distribution in \eqref{mod:semi} and are
sampled with replacement from the set of all possible $n_1n_2$
edges. We note that the sampling setup is commonly adopted in the literature on partially observed graphs and matrix completion problems \citep{Zhong2019Provable, Chen2018Contrastive}, which is equivalent to assuming edges on a graph are missing at random.

Denote sample sets
\begin{align*}
\Omega = \{(y_{u(k),v(k)}, \xb_{u(k)}, \zb_{v(k)})\}_{k = 1}^m
\quad\text{ and }\quad
\Omega' = \{(y_{u'(l),v'(l)}', \xb_{u'(l)}', \zb_{v'(l)}')\}_{l = 1}^m,
\end{align*}
where $u(k), u'(l) = \text{Unif}([n_1])$ and $v(k), v'(l) = \text{Unif}([n_2])$. While the observations within $\Omega$ or $\Omega'$ are not independent, as they may have common features $\xb$ or $\zb$, the observations between $\Omega$ and $\Omega'$ are independent. Such two independent sets of samples are obtained by sample splitting in practice. We stress that the sample splitting setup in our paper is used only to make the analysis concise without enhancing the order of sample complexity. In particular, it does not help us avoid the main difficulties of the problem.

Based on samples $\Omega$ and $\Omega'$, we consider $m^2$ pairwise differences and construct an empirical loss function. For $k\in[m]$, let $k_1 = u(k)$, $k_2 = v(k)$, and
\begin{align*}
\bTheta_{k_1k_2}^\star = \LD \phi_1(\tUb^T\xb_{k_1}), \phi_2(\tVb^T\zb_{k_2}) \RD
\end{align*}
denote the true parameter associated with the $k$-th sample (similarly for $\bTheta_{l_1l_2}^{\star\prime}$). Note that $\bTheta_{k_1k_2}^\star$ is the underlying parametric component of the model that generates $y_k = y_{k_1,k_2}$. The key idea in constructing the pseudo-likelihood objective is to use rank-order decomposition to extract a factor, that is independently from the base measure. Given a pair of independent samples, $y_k$ and $y_l'$, we denote their order statistics as $y_{(\cdot)}$ and rank statistics as $R$. Then we know $y_{(\cdot)} = (y_k, y_l')$ or $y_{(\cdot)} = (y_l', y_k)$, and $R = (1, 2)$ or $R = (2, 1)$. Thus, $(y_{(\cdot)}, R)$ fully characterizes the pair $(y_k, y_l')$, and is hence a sufficient statistics. Note that
\begin{align}\label{equ:rank:order:decom}
P(y_k, y_l' & \mid \bTheta_{k_1k_2}^\star, \bTheta_{l_1l_2}^{\star\prime}, f) \nonumber\\
& = P(y_{(\cdot)}, R \mid\bTheta_{k_1k_2}^\star, \bTheta_{l_1l_2}^{\star\prime}, f ) \nonumber\\
& = P(R \mid y_{(\cdot)}, \bTheta_{k_1k_2}^\star, \bTheta_{l_1l_2}^{\star\prime}, f)\cdot P(y_{(\cdot)} \mid \bTheta_{k_1k_2}^\star, \bTheta_{l_1l_2}^{\star\prime}, f) \nonumber\\
& = \frac{P(y_k \mid \bTheta_{k_1k_2}^\star, f)\cdot P(y'_l \mid \bTheta_{l_1l_2}^{\star\prime}, f)\cdot P(y_{(\cdot)} \mid \bTheta_{k_1k_2}^\star, \bTheta_{l_1l_2}^{\star\prime}, f)}{P(y_k \mid \bTheta_{k_1k_2}^\star, f) \cdot P(y'_l \mid \bTheta_{l_1l_2}^{\star\prime}, f) + P(y_k \mid \bTheta_{l_1l_2}^{\star\prime}, f)\cdot P(y'_l \mid \bTheta_{k_1k_2}^\star, f)} \nonumber\\
& \stackrel{\eqref{mod:semi}}{=} \frac{\exp(y_k\bTheta_{k_1k_2}^\star + y'_l\bTheta_{l_1l_2}^{\star\prime}) }{\exp(y_k\bTheta_{k_1k_2}^\star + y'_l\bTheta_{l_1l_2}^{\star\prime}) + \exp(y'_l\bTheta_{k_1k_2}^\star + y_k\bTheta_{l_1l_2}^{\star\prime})} \cdot P(y_{(\cdot)} \mid \bTheta_{k_1k_2}^\star, \bTheta_{l_1l_2}^{\star\prime}, f) \nonumber\\
& = \underbrace{\frac{1}{1 + \exp\big(-(y_l-y'_l)(\bTheta_{k_1k_2}^\star - \bTheta_{l_1l_2}^{\star\prime})\big) }}_{\text{local differential quasi-likelihood}}\cdot P(y_{(\cdot)} \mid \bTheta_{k_1k_2}^\star, \bTheta_{l_1l_2}^{\star\prime}, f).
\end{align}
The first term is the density of the rank statistics given order
statistics, which is only a function of unknown weight matrices $\tUb$ and $\tVb$. The second term is the density of order statistics, which relies on the specific base measure $f$. Thus, we omit the second term and sum over all $m^2$ paired samples for the first term to arrive at the following objective
\begin{align}\label{equ:loss}
\mL(\Ub, \Vb) =\frac{1}{m^2}\sum_{k,l=1}^{m}\log\rbr{1 + \exp\rbr{- (y_k - y_l')(\bTheta_{k_1k_2} - \bTheta_{l_1l_2}')}}.
\end{align}

The above loss function is similar to the logistic loss for the pairwise measurements. However, it is nonconvex in both components even for identity activation functions. When feature vectors $\xb$, $\zb$ follow the multinomial distribution and activation functions $\{\phi_i\}_{i=1}^2$ are not present, \citet{Chen2018Contrastive} estimated the rank-$r$ matrix $\tUb\tVb^T$ as a whole by minimizing \eqref{equ:loss} with an additional nuclear norm penalty. Our goal is to recover both components $\tUb$, $\tVb$, in the presence of nonlinear activation functions, resulting in a challenging nonconvex optimization problem.

\subsection{Gradient Descent}\label{sec:3.1}

We propose to minimize loss function \eqref{equ:loss} using the gradient descent with a constant step size. The iteration is given by
\begin{align}\label{equ:iter}
\begin{pmatrix}
\Ub^{t+1}\\
\Vb^{t+1}
\end{pmatrix} = \begin{pmatrix}
\Ub^t\\
\Vb^t
\end{pmatrix} - \eta \begin{pmatrix}
\nabla_\Ub \mL(\Ub^t, \Vb^t)\\
\nabla_\Vb \mL(\Ub^t, \Vb^t)
\end{pmatrix}.
\end{align}

For future references, we provide explicit formulas of the gradient and the Hessian for loss \eqref{equ:loss}. We introduce some definitions beforehand. Let us denote each column of weight matrices as $\Ub = (\bu_1, \ldots, \bu_r)$ and $\Vb = (\bv_1, \ldots, \bv_r)$ (similar for $\tUb$, $\tVb$). To simplify notations, for a sequence of vectors $\ba_1, \ldots, \ba_n$, we let $(\ba_i)_{i=1}^n = (\ba_1; \ldots; \ba_n)$ be the long vector by stacking them up; for a sequence of matrices $\Ab_1, \ldots, \Ab_n$, we let $\diag\big((\Ab_i)_{i=1}^n\big)$ be the block diagonal matrix with each block being specified by $\Ab_i$ sequentially. Moreover, we define the following quantities: $\forall k, l \in[m]$ and $\forall i \in [r]$,
\begin{align*}
\bd_{ki} = & \phi_1'(\bu_i^T\xb_{k_1})\phi_2(\bv_i^T\zb_{k_2})\xb_{k_1}, \quad\quad\quad\quad\ \
\bd_{li}' =  \phi_1'(\bu_i^T\xb_{l_1}')\phi_2(\bv_i^T\zb_{l_2}')\xb_{l_1}', \\
\bp_{ki} = & \phi_1(\bu_i^T\xb_{k_1})\phi_2'(\bv_i^T\zb_{k_2})\zb_{k_2},
\quad\quad\quad\quad\ \ \;
\bp_{li}' = \phi_1(\bu_i^T\xb_{l_1}')\phi_2'(\bv_i^T\zb_{l_2}')\zb_{l_2}',\\
\bQ_{ki} = & \phi_1''(\bu_i^T\xb_{k_1})\phi_2(\bv_i^T\zb_{k_2})\xb_{k_1}\xb_{k_1}^T,
\quad\quad\quad
\bQ_{li}' =  \phi_1''(\bu_i^T\xb_{l_1}')\phi_2(\bv_i^T\zb_{l_2}')\xb_{l_1}'\xb_{l_1}^{\prime T},\\
\bR_{ki} = & \phi_1(\bu_i^T\xb_{k_1})\phi_2''(\bv_i^T\zb_{k_2})\zb_{k_2}\zb_{k_2}^T,
\quad\quad\quad\
\bR_{li}' =  \phi_1(\bu_i^T\xb_{l_1}')\phi_2''(\bv_i^T\zb_{l_2}')\zb_{l_2}'\zb_{l_2}^{\prime T},\\
\bS_{ki} = & \phi_1'(\bu_i^T\xb_{k_1})\phi_2'(\bv_i^T\zb_{k_2})\xb_{k_1}\zb_{k_2}^T,
\quad\quad\quad\ \;
\bS_{li}' =  \phi_1'(\bu_i^T\xb_{l_1}')\phi_2'(\bv_i^T\zb_{l_2}')\xb_{l_1}'\zb_{l_2}^{\prime T}.
\end{align*}
The quantities on the left part are vectors or matrices calculated by using samples in $\Omega$, which is indexed by $k$, while the quantities on the right part are calculated by using samples in $\Omega'$, which is indexed by $l$. We should mention that $\phi_i'$, $\phi_i''$ are the first derivative and the second derivative of the activation function $\phi_i$ (if $\phi_i$ is ReLU then $\phi_i'' = 0$), while superscript of $\xb_{l_1}'$ (and $\zb_{l_2}'$) means the sample is from $\Omega'$ (i.e. the sample index $l$ is always used with superscript $(\cdot)'$). In addition, we define two scalars as
\begin{align*}
A_{kl} = & \frac{(y_k - y_l')^2\cdot \exp\big((y_k - y_l')(\bThe_{k_1k_2} - \bTheta_{l_1l_2}')\big)}{\big(1 + \exp\big((y_k - y_l')(\bThe_{k_1k_2} - \bTheta_{l_1l_2}')\big) \big)^2}, \;\; B_{kl}  = \frac{y_k - y_l'}{1 + \exp\big((y_k - y_l')(\bThe_{k_1k_2} - \bTheta_{l_1l_2}')\big)}.
\end{align*}

With above definitions and by simple calculations, one can show the gradient is given by
\begin{equation}\label{eq:gradient}
\begin{aligned}
\nabla_{\Ub}\mL(\Ub, \Vb) = &\rbr{\frac{\partial\mL(\Ub, \Vb)}{\partial \bu_1},\ldots, \frac{\partial\mL(\Ub, \Vb)}{\partial \bu_r}} \text{\ \ with\ \ }   \frac{\partial\mL(\Ub, \Vb)}{\partial \bu_i} = -\frac{1}{m^2}\sum_{k, l=1}^m B_{kl} \rbr{\bd_{ki} - \bd_{li}'},\\
\nabla_{\Vb}\mL(\Ub, \Vb) = &\rbr{\frac{\partial\mL(\Ub, \Vb)}{\partial \bv_1}, \ldots, \frac{\partial\mL(\Ub, \Vb)}{\partial \bv_r}} \text{\ \ with\ \ }  \frac{\partial\mL(\Ub, \Vb)}{\partial \bv_i} = -\frac{1}{m^2}\sum_{k, l=1}^m B_{kl} \rbr{\bp_{ki} - \bp_{li}'}.
\end{aligned}
\end{equation}
Furthermore, $\forall i, j\in[r]$, one can show
\begin{align*}
\frac{\partial^2\mL(\Ub, \Vb)}{\partial \bu_i\partial \bu_j}
& = \frac{1}{m^2}\sum_{k, l=1}^m A_{kl} \rbr{\bd_{ki} - \bd_{li}'}\rbr{\bd_{kj} - \bd_{lj}'}^T
- \frac{\delta_{ij}}{m^2}\sum_{k, l=1}^m B_{kl}\rbr{\bQ_{ki} - \bQ_{li}'},\\
\frac{\partial^2\mL(\Ub, \Vb)}{\partial \bu_i\partial \bv_j}
& = \frac{1}{m^2}\sum_{k, l=1}^m A_{kl} \rbr{\bd_{ki} - \bd_{li}'}\rbr{\bp_{kj} - \bp_{lj}'}^T
- \frac{\delta_{ij}}{m^2}\sum_{k, l=1}^m B_{kl}\rbr{\bS_{ki} - \bS_{li}'},\\
\frac{\partial^2\mL(\Ub, \Vb)}{\partial \bv_i\partial \bv_j}
& = \frac{1}{m^2}\sum_{k, l=1}^m A_{kl} \rbr{\bp_{ki} - \bp_{li}'}\rbr{\bp_{kj} - \bp_{lj}'}^T
- \frac{\delta_{ij}}{m^2}\sum_{k, l=1}^m B_{kl}\rbr{\bR_{ki} - \bR_{li}'}.
\end{align*}
To combine all blocks and form the Hessian matrix, we will vectorize weight matrices and further define long vectors $\bd_k = \rbr{\bd_{ki}}_{i = 1}^r$, $\bp_k = \rbr{\bp_{ki}}_{i = 1}^r$, $\bd_l' = \rbr{\bd_{li}'}_{i = 1}^r$, $\bp_l' = \rbr{\bp_{li}'}_{i = 1}^r$, and block diagonal matrices $\bQ_k = \diag\rbr{\rbr{\bQ_{ki}}_{i=1}^r}$, $\bR_k = \diag\rbr{\rbr{\bR_{ki}}_{i=1}^r}$, $\bS_k = \diag\rbr{\rbr{\bS_{ki}}_{i=1}^r}$ (similar for $\bQ_l'$, $\bR_l'$, $\bS_l'$). Then, the Hessian matrix $\nabla^2\mL(\Ub, \Vb)\in\mR^{r(d_1+d_2)\times r(d_1+d_2)}$ is
\begin{multline}
\label{d:1}
\nabla^2\mL(\Ub, \Vb) = \begin{pmatrix}
\big(\frac{\partial^2\mL}{\partial\bu_i\partial\bu_j}\big)_{i,j} & \big(\frac{\partial^2\mL}{\partial\bu_i\partial\bv_j}\big)_{i,j} \\
\big(\frac{\partial^2\mL}{\partial\bv_i\partial\bu_j}\big)_{i,j} & \big(\frac{\partial^2\mL}{\partial\bv_i\partial\bv_j}\big)_{i,j}
\end{pmatrix} \\
= \frac{1}{m^2}\sum_{k, l=1}^m A_{kl}\cdot\begin{pmatrix}
\bd_k - \bd_l'\\
\bp_k - \bp_l'
\end{pmatrix}\begin{pmatrix}
\bd_k - \bd_l'\\
\bp_k - \bp_l'
\end{pmatrix}^T  - \frac{1}{m^2}\sum_{k, l=1}^m B_{kl}\cdot\begin{pmatrix}
\bQ_k - \bQ_l' & \bS_k - \bS_l'\\
\bS_k^T - \bS_l'^T & \bR_k - \bR_l'
\end{pmatrix}.
\end{multline}

\subsection{Identifiability}

In general, the weight matrices in loss function \eqref{equ:loss} are not identifiable as the function is bilinear in $\Ub$, $\Vb$. For example, when both activation functions are identity, $\mL(\Ub \Qb, \Vb(\Qb^{T})^{-1})$ and $\mL(\Ub, \Vb)$ have the same value for any invertible matrix $\Qb\in\mR^{r\times r}$, which makes the Hessian at $(\tUb, \tVb)$ indefinite. Similarly, for ReLU activation, this phenomenon reappears by letting $\Qb$ be any diagonal matrix with positive entries. To resolve this issue, one can use a penalty function $\|\Ub^T\Ub - \Vb^T\Vb\|_F^2$ to balance two components $\Ub$ and $\Vb$ \citep{Yi2016Fast, Park2018Finding, Na2019Estimating}. Fortunately, in our problem, the identifiability issue disappears when a smooth nonlinear activation is used, such as sigmoid or tanh, although their nonconvexity brings other challenges.

We stress that, different from over-parameterized problems in neural networks \citep{Sagun2017Empirical, Li2018Learning, Allen-Zhu2018Learning}, the identifiability issue comes from the redundancy of parameters, which is also observed in inductive matrix completion problem \citep{Zhong2019Provable}. \citet{Zhong2019Provable} showed that by fixing the first row of $\tUb$, both components are recoverable from the square loss even with ReLU activation. In our problem, when either one of activation functions is ReLU, we use a similar restriction on $\tUb$ and show that the loss in \eqref{equ:loss} has positive definite Hessian at $(\tUb, \tVb)$, without adding any penalties.

\section{Theoretical analysis}\label{sec:4}

In this section, we will show that the ground truth $(\tUb, \tVb)$ is a stationary point of the loss \eqref{equ:loss} and then show that the loss is strongly convex in its neighborhood. Using these two observations, we further establish the local linear convergence rate for iterates in \eqref{equ:iter}. Since the radius of the neighborhood is fixed in terms of $(\tUb,\tVb)$, a wart-start initialization can be obtained by a third-order tensor method \citep[see, e.g.,][]{Zhong2017Recovery, Zhong2019Provable}. In our simulations, due to high computational cost of a tensor method, we recommend a random initialization \citep{Du2017Gradient, Cao2019Tight}.

\subsection{Assumptions}

We require two assumptions to establish our main results. The first assumption fixes the scale of weight matrices.

\begin{assumption}\label{ass:1}

The weight matrices $\tUb$, $\tVb$ have rank $r$ and satisfy $\sigma_r(\tUb)= \sigma_r(\tVb)= 1$.

\end{assumption}

The second assumption imposes a mild regularity condition.

\begin{assumption}\label{ass:2}
Let $\D = \{(y_{ij}, \xb_i, \zb_j)\}_{i\in[n_1], j\in[n_2]}$ and $\D' = \{(y_{ij}', \xb_i', \zb_j')\}_{i\in[n_1], j\in[n_2]}$ be two complete subgraphs (the edges between $\xb_i$ and $\zb_j'$, and $\zb_j$ and $\xb_i'$ are ignored). We assume

\begin{enumerate}[label=(\alph*),topsep=-5pt]		\setlength\itemsep{0em}
\item (boundedness): There exist $\alpha, \beta>0$ such that,
for any sample $(y, \xb, \zb)\in \D\cup \D'$, we have $|\tThe| = |\LD \phi_1(\tUb^T\xb), \phi_2(\tVb^T\zb)\RD| \leq \alpha$ and $|y|\leq \beta$;
\item (regularity condition): Suppose $(y, \xb, \zb)\in\D$ and $(y', \xb', \zb')\in\D'$, we let
\begin{align*}
M_{\alpha}(\tThe, \tTheb') = \mE\sbr{(y-y')^2\cdot\psi(2\alpha|y-y'|) \mid {\xb, \zb, \xb', \zb'}},
\end{align*}
where $\psi(x) = \exp(x)/(1+\exp(x))^2$, and assume $M_{\alpha}(\tThe, \tTheb')$ is a continuous, positive two-dimensional function.
\end{enumerate}
\end{assumption}

Assumption \ref{ass:2} is widely assumed in the analysis of logistic loss function \citep{Chen2018Contrastive}. In particular, Assumption \ref{ass:2}(a) restricts the parametric component $\tThe$ into a compact set, which controls the range of proximity between two connected nodes. Intuitively, larger $\alpha$ implies a harder estimation problem. We also add boundedness condition on the response $y$ for simplicity. It can be replaced by assuming $y$ to be subexponential \citep{Ning2017likelihood}. Boundedness holds deterministically for some distribution in exponential family, such as Bernoulli and Beta, and holds with high probability for a wide range of exponential family distributions, though $\beta$ may depend on the sample size $n_1$ and $n_2$. Assumption \ref{ass:2}(b) is the regularity condition, which plays the key role when showing the strong convexity of the population loss at the ground truth. It can be shown to hold for all exponential family distributions with bounded support, and for some unbounded distributions, such as Gaussian and Poisson.

\subsection{Properties of the Population Loss}

With the above assumptions, our first result shows that the gradient
of the population loss at $(\tUb, \tVb)$ is zero. For all quantities
defined in Section \ref{sec:3.1}, we add superscript $(\cdot)^\star$
to denote the underlying true quantities, which are obtained by
replacing $\Ub, \Vb$ with true weight matrices $\tUb, \tVb$. For
example, we have
$A_{kl}^\star, B_{kl}^\star, \bd_{ki}^\star, \bp_{ki}^\star,
\bQ_{ki}^\star, \bR_{ki}^\star, \bS_{ki}^\star$.

The following lemma shows that the conditional expectation of
$B_{kl}^\star$ given all covariates associated to two end vertices is
zero.

\begin{lemma}\label{lem:B:1}
  For any $k, l\in[m]$, we have that the conditional expectation given
  all covariates
  $\mE\sbr{B_{kl}^\star \mid \xb_{k_1}, \zb_{k_2}, \xb_{l_1}',
    \zb_{l_2}'} = 0$.
\end{lemma}

Since $B_{kl}^\star$ is a common factor of the gradients
$\nabla_{\Ub}\mL(\tUb, \tVb)$ and $\nabla_{\Vb}\mL(\tUb, \tVb)$, as
shown in \eqref{eq:gradient}, and vectors
$\bd_{ki}^\star, \bd_{li}^{\prime \star}, \bp_{ki}^\star,
\bp_{li}^{\prime \star}$ only depend on covariates, one can first take
conditional expectation given covariates and show the following
result.

\begin{theorem}\label{thm:1}
The loss \eqref{equ:loss} satisfies $\mE\sbr{\nabla \mL(\tUb, \tVb)} = \0$.
\end{theorem}

\begin{proof}
We take $\mE[\nabla_{\Ub}\mL(\tUb, \tVb)]$ as an example, while $\mE[\nabla_{\Vb}\mL(\tUb, \tVb)]$ can be shown similarly. By the formula in \eqref{eq:gradient}, $\forall i\in[r]$, we have
\begin{align*}
\mE\sbr{\frac{\partial\mL(\tUb, \tVb)}{\partial \bu_i}}
= & - \mE\sbr{\frac{1}{m^2}\sum_{k, l=1}^m B_{kl}^\star \rbr{\bd_{ki}^\star - \bd_{li}^{\star \prime}}}\\
= & - \mE\bigg[\frac{1}{m^2}\sum_{k, l=1}^m \mE[B_{kl}^\star \mid \xb_{k_1}, \zb_{k_2}, \xb_{l_1}', \zb_{l_2}']\cdot \rbr{\bd_{ki}^\star - \bd_{li}^{\star \prime}}\bigg] = \0,
\end{align*}
where, for the second term from the end, the outer expectation is taken over randomness in sampling and all covariate, and the last equality is due to Lemma \ref{lem:B:1}. Doing same derivation for each column and we obtain $\mE\sbr{\nabla_{\Ub} \mL(\tUb, \tVb)} = \0$. Similarly $\mE\sbr{\nabla_{\Vb} \mL(\tUb, \tVb)} = \0$.
\end{proof}

We then study the local curvature of the population loss at
$(\tUb, \tVb)$, which is obtained in the next two steps. We simplify
the notation further by dropping the subscripts of sample index. We
let $A$, $B$, $\bd, \bq$, $\bd'$, $\bq'$, $\ldots$, and their
corresponding $(\cdot)^\star$ version, denote general references of
corresponding quantities, which may be computed by using any samples
in $\D$ and $\D'$. We stress that all samples in $\D$ and $\D'$ have
the same distribution, so that $\bd_1, \ldots, \bd_m \sim \bd$,
$\bp_1, \ldots, \bp_m \sim \bp$, with $\bd$ and $\bd'$, and $\bp$ and
$\bp'$ independent from each other.

\begin{proposition}\label{prop:lower_bound_hessian}
Suppose Assumptions \ref{ass:1} and \ref{ass:2} hold. Define
\begin{align*}
\gamma_\alpha = \inf_{\bThe_1, \bThe_2\in [-\alpha, \alpha]}M_{\alpha}(\bThe_1, \bThe_2),
\end{align*}
then we have $\gamma_\alpha>0$ and
\begin{align*}
\mE\sbr{\nabla^2\mL(\tUb, \tVb)} \succeq \gamma_\alpha \cdot \mE\sbr{\begin{pmatrix}
	\td - \tdp\\
	\tp - \tpp
	\end{pmatrix}\begin{pmatrix}
	\td - \tdp\\
	\tp - \tpp
	\end{pmatrix}^T}.
\end{align*}
\end{proposition}

\begin{proof}
Recall the formula for the Hessian matrix in \eqref{d:1}. The second term has zero expectation at $(\tUb, \tVb)$ by Lemma \ref{lem:B:1}. Therefore,
\begin{align}\label{d:2}
\mE\sbr{\nabla^2\mL(\tUb, \tVb)} = \mE\sbr{A^\star\cdot\begin{pmatrix}
\td - \tdp\\
\tp - \tpp
\end{pmatrix}\begin{pmatrix}
\td - \tdp\\
\tp - \tpp
\end{pmatrix}^T}.
\end{align}
In our notations, $A^\star$ is written as
\begin{align*}
A^\star = \frac{(y - y')^2 \cdot \exp\rbr{(y - y')(\tThe - \tThep)}}{\rbr{1 + \exp\rbr{(y - y')(\tThe - \tThep)}}^2},
\end{align*}
where $\tThe = \tThe(\xb, \zb)$, $\tThep = \tThe(\xb', \zb')$ (cf. Section \ref{sec:2:1} for definition of $\tThe(\xb, \zb)$), and $(y, \xb, \zb)$ and $(y', \xb', \zb')$ are two independent samples from $\D$ and $\D'$, respectively. By Assumption \ref{ass:2},  $|\tThe| \vee |\tThep|\leq \alpha$. Thus, $|(y - y')(\tThe - \tThep)|\leq 2\alpha|y - y'|$. Using the symmetry and monotonicity of $\psi(x)$, defined in Assumption \ref{ass:2},
\begin{align*}
\frac{\exp\rbr{(y - y')(\tThe - \tThep)}}{\rbr{1 + \exp\rbr{(y - y')(\tThe - \tThep)}}^2}
=  \psi\rbr{\abr{(y - y')(\tThe - \tThep)}} \geq  \psi\rbr{2\alpha\abr{y - y'}}.
\end{align*}
Therefore,
\begin{align*}
A^\star \geq (y-y')^2 \cdot \psi\rbr{2\alpha\abr{y - y'}}.
\end{align*}
Taking conditional expectation in \eqref{d:2} and using the definition of $\gamma_\alpha$,
\begin{align*}
\mE\sbr{\nabla^2\mL(\tUb, \tVb)}
& \succeq \mE\sbr{(y-y')^2 \cdot \psi\rbr{2\alpha\abr{y - y'}} \cdot \begin{pmatrix}
	\td - \tdp\\
	\tp - \tpp
	\end{pmatrix}\begin{pmatrix}
	\td - \tdp\\
	\tp - \tpp
	\end{pmatrix}^T}\\
& = \mE\sbr{ \mE\sbr{(y-y')^2 \cdot \psi\rbr{2\alpha\abr{y - y'}} \mid \xb, \zb, \xb', \zb'}\begin{pmatrix}
	\td - \tdp\\
	\tp - \tpp
	\end{pmatrix}\begin{pmatrix}
	\td - \tdp\\
	\tp - \tpp
	\end{pmatrix}^T}\\
& = \mE\sbr{M_{\alpha}(\tThe, \tThep)\cdot\begin{pmatrix}
	\td - \tdp\\
	\tp - \tpp
	\end{pmatrix}\begin{pmatrix}
	\td - \tdp\\
	\tp - \tpp
	\end{pmatrix}^T}\\
& \succeq \gamma_\alpha \cdot\mE\sbr{\begin{pmatrix}
	\td - \tdp\\
	\tp - \tpp
	\end{pmatrix}\begin{pmatrix}
	\td - \tdp\\
	\tp - \tpp
	\end{pmatrix}^T}.
\end{align*}
Note that $|\tThe|\vee |\tThep|\leq \alpha$ and
$M_\alpha(\cdot, \cdot)$ is strictly positive on $[-\alpha, \alpha]\times[-\alpha, \alpha]$, by continuity, $M_\alpha(\cdot, \cdot)$ attains its minimum value in the compact support, hence, $\gamma_\alpha > 0$. This completes the proof.
\end{proof}

Note that $\gamma_\alpha$ in Proposition~\ref{prop:lower_bound_hessian} depends on $\alpha$ reciprocally. The next result lower bounds the minimum eigenvalue of $\mE\sbr{\begin{pmatrix}
\td - \tdp\\
\tp - \tpp
\end{pmatrix}\begin{pmatrix}
\td - \tdp\\
\tp - \tpp
\end{pmatrix}^T}$.
We mention that \cite{Zhong2019Provable} established a similar result
when $\tdp, \tpp$ are not present and
$\phi_1 = \phi_2$. However, our result is based on pairwise measurements
which allows for adaptivity to nonparametric (nuisance) parameter in the model
and, further, also allows for mismatch in activation functions.
These two differences make the proof more involved.
We separate results into two cases:
(1) $\phi_1$, $\phi_2\in\{\text{sigmoid}, \text{tanh}\}$; (2) either $\phi_1$ or $\phi_2$ is ReLU.

\begin{lemma}\label{lem:B:2}

  Suppose Assumptions \ref{ass:1} and \ref{ass:2} hold.
  We let $\barkap(\tUb) = \prod_{p=1}^r\frac{\sigma_p(\tUb)}{\sigma_{r}(\tUb)}$
  and similarly for $\tVb$.
  Then there exists a constant $C>0$, independent of $\tUb$ and $\tVb$, such that:
\begin{enumerate}[label=(Case \arabic*),itemindent=20pt,topsep=-5pt]		\setlength\itemsep{0em}

\item if $\phi_1$, $\phi_2 \in \{\text{sigmoid}, \text{tanh}\}$, then
\begin{align*}
\lambda_{\min}\rbr{
\mE\sbr{\begin{pmatrix}
\td - \tdp\\
\tp - \tpp
\end{pmatrix}\begin{pmatrix}
\td - \tdp\\
\tp - \tpp
\end{pmatrix}^T}}  \geq \frac{C}{\barkap(\tUb)\barkap(\tVb)\max(\|\tUb\|_2^2, \|\tVb\|_2^2)};
\end{align*}

\item if either $\phi_1$ or $\phi_2$ is ReLU, then by fixing the first row of $\tUb$ (i.e. treating it as known),
\begin{multline*}
\lambda_{\min}\rbr{
\mE\sbr{\begin{pmatrix}
\td - \tdp\\
\tp - \tpp
\end{pmatrix}\begin{pmatrix}
\td - \tdp\\
\tp - \tpp
\end{pmatrix}^T}}\\
\geq \frac{C \cdot \|\eb_1^T\tUb\|_{\min}^2}{\barkap(\tUb)\barkap(\tVb)\max(\|\tUb\|_2^2, \|\tVb\|_2^2)(1 + \|\eb_1^T\tUb\|_2)^2},
\end{multline*}
where $\eb_1 = (1, 0, \ldots, 0) \in \mR^{d_1}$.
\end{enumerate}

\end{lemma}

Combining the results of Proposition~\ref{prop:lower_bound_hessian} and Lemma \ref{lem:B:2}, we immediately get the following result regarding the local curvature of the population loss at the ground truth.

\begin{theorem}[Local curvature]\label{thm:2}

Suppose Assumptions \ref{ass:1} and \ref{ass:2} hold. There exists a constant $C>0$, independent of $\tUb$ and $\tVb$, such that:
\begin{enumerate}[label=(Case \arabic*),itemindent=20pt,topsep=-5pt]		\setlength\itemsep{0em}
\item if $\phi_1$, $\phi_2\in\{\text{sigmoid}, \text{tanh}\}$, then
\begin{align*}
\lambda_{\min}\rbr{\mE[\nabla^2\mL(\tUb, \tVb)]}
\geq \frac{C\cdot\gamma_\alpha}{\barkap(\tUb)\barkap(\tVb)\max(\|\tUb\|_2^2, \|\tVb\|_2^2)};
\end{align*}

\item if either $\phi_1$ or $\phi_2$ is ReLU, then by fixing the first row of $\tUb$,
\begin{align*}
\lambda_{\min}\rbr{\mE[\nabla^2\mL(\tUb, \tVb)]}
\geq \frac{C\cdot\gamma_\alpha\|\eb_1^T\tUb\|_{\min}^2}{\barkap(\tUb)\barkap(\tVb)\max(\|\tUb\|_2^2, \|\tVb\|_2^2)(1 + \|\eb_1^T\tUb\|_2)^2},
\end{align*}
where $\eb_1 = (1,0,\ldots,0)\in\mR^{d_1}$ and $\barkap(\tUb), \barkap(\tVb)$ are defined in Lemma \ref{lem:B:2}.
\end{enumerate}
\end{theorem}

By symmetry one can alternatively fix the first row of $\tVb$ in the
second case. We realize that the lower bound of population Hessian in
Case 2 is smaller than the bound in Case 1. This is due to
nonsmoothness and unboundedness of ReLU activation function. In later
analysis we will see the sample complexity when using ReLU for either
networks will have larger logarithmic factor, while is linear in
$d_1 \vee d_2$ in both cases.

Combining Theorem~\ref{thm:1} and \ref{thm:2}, we obtain that
$(\tUb, \tVb)$ is a local minimizer of the population loss. In order
to characterize how the empirical loss behaves near the ground truth,
we study its local geometry via the concentration of the Hessian
matrix.

\subsection{Concentration of the Hessian Matrix}\label{sec:4.3}

In this section, we characterize the concentration of the Hessian matrix. We show that $(m \wedge n_1\wedge n_2)\gtrsim (d_1\vee d_2)\text{poly}(\log(d_1+d_2))$ is sufficient to guarantee that the empirical loss also has positive curvature locally.

Let
\begin{align*}
\Hb_{1,k,l} = A_{kl}\cdot\begin{pmatrix}
\bd_k - \bd_l'\\
\bp_k - \bp_l'
\end{pmatrix}\begin{pmatrix}
\bd_k - \bd_l'\\
\bp_k - \bp_l'
\end{pmatrix}^T
\quad\text{and}\quad
\Hb_{2,k,l} =B_{kl}\cdot\begin{pmatrix}
\bQ_k - \bQ_l' & \bS_k - \bS_l'\\
\bS_k^T - \bS_l'^T & \bR_k - \bR_l'
\end{pmatrix}
\end{align*}
and define
\begin{align*}
\nabla^2\mL_1(\Ub, \Vb) = \frac{1}{m^2}\sum_{k, l=1}^m\Hb_{1,k,l}
\quad\text{and}\quad
\nabla^2\mL_2(\Ub, \Vb) = \frac{1}{m^2}\sum_{k, l=1}^m\Hb_{2,k,l}.
\end{align*}
By the formula in \eqref{d:1}, we have that
\begin{align*}
\nabla^2\mL(\Ub, \Vb) = \nabla^2\mL_1(\Ub, \Vb) - \nabla^2\mL_2(\Ub, \Vb).
\end{align*}
The concentration of each term will be built separately in next two lemmas. We let $q = 1$ if either $\phi_i$ is ReLU and $q=0$ otherwise, and $q' = 1$ if both $\phi_i$ are ReLU and $q'=0$ otherwise.

\begin{lemma}\label{lem:B:3}
Suppose Assumptions \ref{ass:1} and \ref{ass:2} hold. For any $s\geq 1$, if
\begin{align}\label{samp:com:1}
m \wedge n_1\wedge n_2 \gtrsim s(d_1+d_2)\cbr{\log\rbr{r(d_1 + d_2)}}^{1 + 2q},
\end{align}
then with probability at least $1 - 1/(d_1 + d_2)^s$,
\begin{align*}
\|\nabla^2\mL_1(\Ub,\Vb) - &\mE\sbr{\nabla^2\mL_1(\tUb, \tVb)}\|_2  \\
\lesssim & \beta^3r^{\frac{3(1-q)}{2}}\rbr{\|\tVb\|_F^{3q} + \|\tUb\|_F^{3q}} \cdot \\
& \qquad\rbr{\sqrt{\frac{s(d_1 + d_2)\log\rbr{r(d_1 + d_2)}}{m\wedge n_1\wedge n_2}} + \rbr{\|\Ub - \tUb\|_F^2 + \|\Vb - \tVb\|_F^2}^{\frac{2-q}{4}}}.
\end{align*}
\end{lemma}

\begin{lemma}\label{lem:B:4}

Suppose Assumptions \ref{ass:1} and \ref{ass:2} hold. For any $s\geq 1$, if
\begin{align}\label{samp:com:2}
m \wedge n_1\wedge n_2\gtrsim s(d_1+d_2)\cbr{\log\rbr{r(d_1 + d_2)}}^{1 + q - q'},
\end{align}
then with probability at least $1 - 1/(d_1 + d_2)^s$,
\begin{align*}
\|\nabla^2\mL_2(\Ub, \Vb) - &\mE\sbr{\nabla^2\mL_2(\tUb, \tVb)}\|_2  \\
\lesssim & \beta^2r^{\frac{1-q}{2}}\rbr{\|\tVb\|_F^{2q} + \|\tUb\|_F^{2q}} \cdot\\
& \qquad \rbr{\sqrt{\frac{s(d_1 + d_2)\log\rbr{r(d_1 + d_2)}}{m\wedge n_1\wedge n_2}} +  \rbr{\|\Ub - \tUb\|_F^2 + \|\Vb - \tVb\|_F^2}^{\frac{2-q}{4}}}.
\end{align*}
\end{lemma}

Comparing the sample complexity in Lemma \ref{lem:B:3} and \ref{lem:B:4}, we see that \eqref{samp:com:2} is dominated by \eqref{samp:com:1}. Technically, this is because $\bQ_k$ and $\bR_k$ are not present if $\phi_1 = \phi_2 = \text{ReLU}$. Combining the above two lemmas and using the inequality that
\begin{multline*}
\|\nabla^2\mL(\Ub, \Vb) - \mE\sbr{\nabla^2\mL(\tUb, \tVb)}\|_2\\
\leq
\|\nabla^2\mL_1(\Ub, \Vb) - \mE\sbr{\nabla^2\mL_1(\tUb, \tVb)}\|_2
+ \|\nabla^2\mL_2(\Ub, \Vb) - \mE\sbr{\nabla^2\mL_2(\tUb, \tVb)}\|_2,
\end{multline*}
we immediately obtain the following concentration on the Hessian matrix.

\begin{theorem}[Concentration of the Hessian matrix]\label{thm:3}

Suppose Assumptions \ref{ass:1} and \ref{ass:2} hold.
For any $s\geq 1$, if
\begin{align}\label{samp:com}
m \wedge n_1\wedge n_2 \gtrsim s(d_1+d_2)\cbr{\log\rbr{r(d_1 + d_2)}}^{1 + 2q},
\end{align}
where $q = 0$ for Case 1 and $q = 1$ for Case~2, then with probability at least $1 - 1/(d_1 + d_2)^s$,
\begin{align*}
\|\nabla^2\mL(\Ub, \Vb) & - \mE\sbr{\nabla^2\mL(\tUb, \tVb)}\|_2 \\
& \lesssim \beta^3r^{\frac{3(1-q)}{2}}\rbr{\|\tVb\|_F^{3q} + \|\tUb\|_F^{3q}} \\
& \qquad \cdot\Biggl(\sqrt{\frac{s\log(d_1 + d_2)\log\rbr{r(d_1 + d_2)}}{m\wedge n_1\wedge n_2}}	+
\rbr{\|\Ub - \tUb\|_F^2 + \|\Vb - \tVb\|_F^2}^{\frac{2-q}{4}}\Biggr).
\end{align*}

\end{theorem}

Replacing $(\Ub, \Vb)$ with $(\tUb, \tVb)$ in the above inequality, one can show that $\nabla^2\mL(\tUb, \tVb)$ is lower bounded away from zero when $m\wedge n_1 \wedge n_2$ is sufficiently large. It turns out this observation is a fundamental condition for establishing local linear convergence rate for gradient descent.

Comparing the above sample complexity with the one established for inductive matrix completion problem \citep{Zhong2019Provable}, our rate improves from $d(\log d)^3$ to $d\log d$, when $\phi_1, \phi_2$ are sigmoid or tanh. Moreover, we allow a semiparametric model with two different activation functions, which results in a more involved analysis.

\subsection{Local Linear Convergence}

The local geometry established for the loss function \eqref{equ:loss} in previous two subsections allows us to prove the local result: the gradient descent with constant step size converges to the ground truth linearly. For ease of notation, let
\begin{align*}
\lambda_{\min}^\star\coloneqq \lambda_{\min}\rbr{\mE[\nabla^2\mL(\tUb, \tVb)]}
\quad\text{and}\quad
\lambda_{\max}^\star\coloneqq \lambda_{\max}\rbr{\mE[\nabla^2\mL(\tUb, \tVb)]}
\end{align*}
be the minimum and maximum eigenvalue of the population Hessian. The explicit lower bound of $\lambda_{\min}^\star$ is provided in Theorem \ref{thm:2}, while the upper bound of $\lambda_{\max}^\star$ is provided in the following Lemma \ref{lem:B:5}. We define the local neighborhood of $(\tUb, \tVb)$ as 
\begin{align*}
\B_R(\tUb, \tVb) =
\cbr{(\Ub, \Vb): \|\Ub - \tUb\|_F^2 + \|\Vb - \tVb\|_F^2 \leq R}
\end{align*}
with radius satisfying
\begin{align*}
R \leq \cbc \rbr{\frac{\lambda_{\min}^\star}{\beta^3r^{3(1-q)/2}\rbr{\|\tUb\|_F^{3q} + \|\tVb\|_F^{3q}} }}^{\frac{4}{2-q}}
\end{align*}
for a sufficiently small constant $\cbc$. The above radius is
determined by the concentration bound of Hessian in Theorem
\ref{thm:3}, based on which one can show that
$\nabla^2\mL(\Ub, \Vb)\succeq \lambda_{\min}^\star/2\cdot I$ for any
$(\Ub, \Vb)\in\B_R(\tUb, \tVb)$. We also note that the above radius
only depends on true weight matrices and is independent from sample
sizes and dimensions. Thus, it will not vanish as dimension increases,
provided $(\tUb, \tVb)$ scale properly.

In preparation for the convergence analysis,
next lemma
characterizes
the difference
$\nabla^2\mL(\Ub_1, \Vb_1) - \nabla^2\mL(\Ub_2, \Vb_2)$ for any
$(\Ub_1, \Vb_1)$, $(\Ub_2, \Vb_2) \in \B_R(\tUb, \tVb)$.

\begin{lemma}\label{lem:hessian_neighborhood}
	
Suppose the conditions of Theorem \ref{thm:3} hold. For any $s\geq 1$ and any $(\Ub_1, \Vb_1)$, $(\Ub_2, \Vb_2) \in \B_R(\tUb, \tVb)$,
\begin{align*}
\|\nabla^2\mL(\Ub_1, &\Vb_1) - \nabla^2\mL(\Ub_2, \Vb_2)\|_2\\
&\lesssim
\beta^3 r^{\frac{3(1-q)}{2}}\rbr{\|\tUb\|_F^{3q}+\|\tVb\|_F^{3q}}\cdot\\
&\qquad \rbr{\sqrt{\frac{s(d_1 + d_2)\log\rbr{r(d_1 + d_2)}}{m\wedge n_1\wedge n_2}} + \rbr{\|\Ub_1 - \Ub_2\|_F^2 + \|\Vb_1 - \Vb_2\|_F^2}^{\frac{2-q}{4}}},
\end{align*}
with probability at least $1 -1/(d_1 + d_2)^s$.

\end{lemma}

The next result provides an upper bound on $\lambda_{\max}^\star$ and we then establish the local linear convergence rate.

\begin{lemma}\label{lem:B:5}

Under Assumption \ref{ass:2},
\begin{align*}
\lambda_{\max}^\star = \|\mE[\nabla^2\mL(\tUb, \tVb)]\|_2\lesssim \beta^2r^{1-q}\rbr{\|\tVb\|_F^2+\|\tUb\|_F^2}^q.
\end{align*}

\end{lemma}

\begin{theorem}[Local linear convergence rate]\label{thm:4}

Suppose Assumptions \ref{ass:1} and \ref{ass:2} hold and the initial point $(\Ub^0, \Vb^0) \in \B_R(\tUb, \tVb)$. For any $s\geq 1$, if the sample sizes satisfies \eqref{samp:com}, then with probability at least $1 - T/(d_1 + d_2)^s$, the iterates in \eqref{equ:iter} with $\eta= 1/\lambda_{\max}^\star$ satisfy
\begin{align*}
\|\Ub^T - \tUb\|_F^2 + \|\Vb^T - \tVb\|_F^2
\leq \rho^T\rbr{\|\Ub^0 - \tUb\|_F^2 + \|\Vb^0 - \tVb\|_F^2},
\end{align*}
where the contraction rate $\rho = 1 - \lambda_{\min}^\star/(7\lambda_{\max}^\star)$.

\end{theorem}

Based on previous preparation work, the proof of Theorem \ref{thm:4}
is standard for gradient descent. For completeness, we present the
proof in Section \ref{sec:7}.

In next section, we demonstrate the superiority and generality of the
proposed representation learning model via extension simulations and
real experiments.

\section{Experiments}\label{sec:5}

We show experimental results on synthetic and real-world data.  In the following, we call our model nonlinear semiparametric matrix completion (NSMC). We compare NSMC with the baseline nonlinear inductive matrix completion (NIMC) proposed by \citet{Zhong2019Provable}, where they assumed the generative model to be Gaussian and minimized the squared loss. The models obtained by removing non-linear activation functions in NSMC and NIMC are called SMC and IMC, respectively.

\subsection{Local Linear Convergence}\label{sec::local_conv}

We verify the local linear convergence of GD on synthetic data sets sampled with ReLU activation functions. We fix $d = d_1 = d_2 = 10$ and $r=3$. The features $\{\xb_i, \xb_i'\}_{i \in [n_1]}$, $\{\zb_j, \zb'_j\}_{j \in [n_2]}$, are independently sampled from a Gaussian distribution. We fix $n_1 = n_2 = 400$ and the number of observations $m = 2000$. We randomly initialize $(\mathbf{U}^{0},\mathbf{V}^{0})$ near the ground truth $(\mathbf{U}^{\star},\mathbf{V}^{\star})$ with fixed error in Frobenius norm. In particular, we fix $||\mathbf{U}^0-\mathbf{U}^\star||_F^2+||\mathbf{V}^0-\mathbf{V}^\star||_F^2 = 1$. For the Gaussian model, $y \sim \mathcal{N}(\mathbf{\Theta}\cdot\sigma^2,\sigma^2)$. For the binomial model, $y \sim B\rbr{(N_B,\frac{\exp (\mathbf{\Theta})}{1+\exp(\mathbf{\Theta})}}$. For Poisson model, $y \sim {\rm Pois}(\exp (\mathbf{\Theta}))$. To introduce some variations, as well as to verify that our model allows for two separate neural networks, we let $\phi_1 = $ ReLU and $\phi_2 \in$ \{ReLU, sigmoid, tanh\}. The estimation error during training process is shown in Figure~\ref{fig:conver}, which verifies the linear convergence rate of GD before reaching the local minima.

\begin{figure*}[t] \centering
	\begin{subfigure}[t]{0.33\textwidth}
		\label{fig:aaa}
		\includegraphics[scale=0.27]{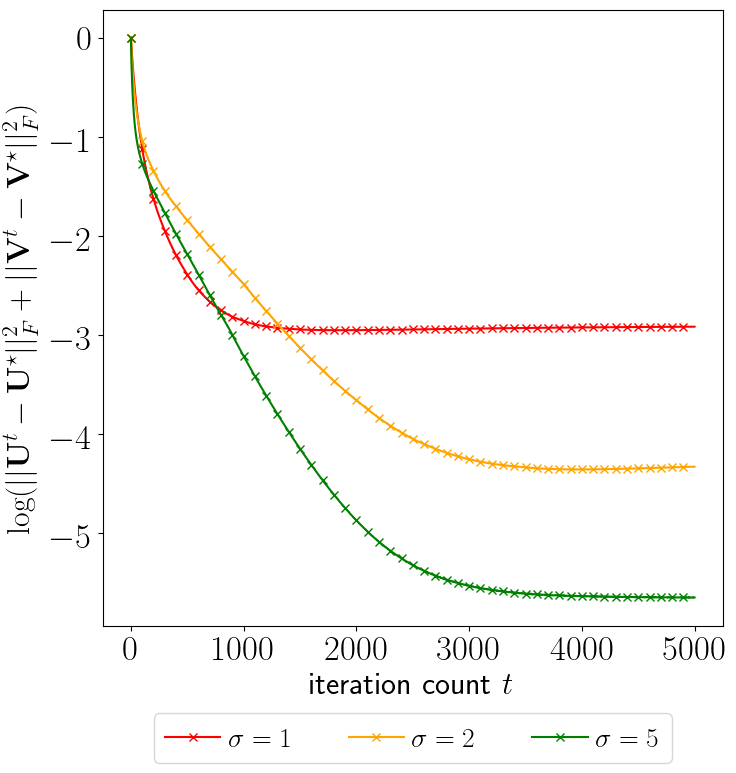}
		\caption{Gaussian Model}
	\end{subfigure}\hfill
	\begin{subfigure}[t]{0.33\textwidth}
		\label{fig:bbb}
		\includegraphics[scale=0.27]{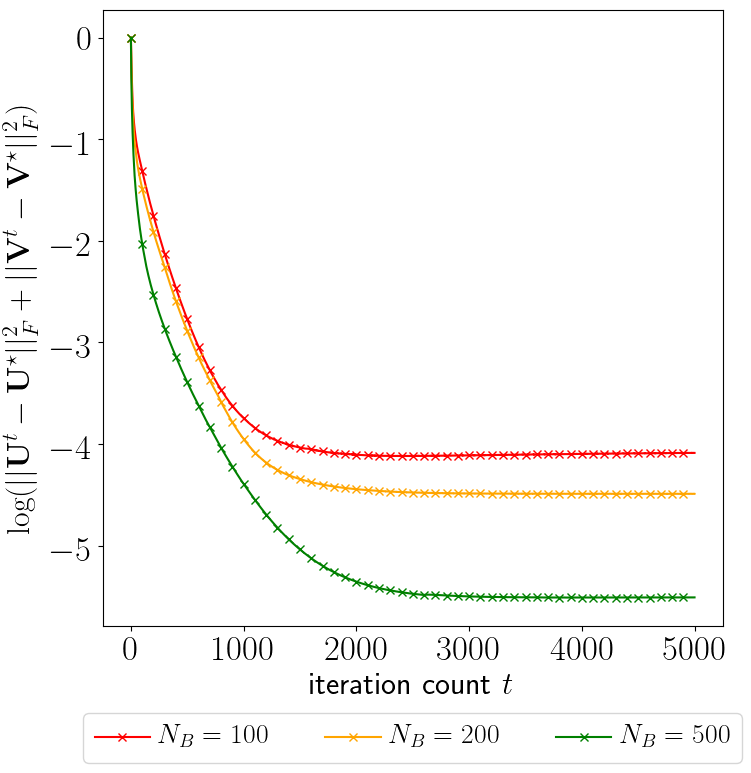}
		\caption{Binomial Model}
	\end{subfigure}\hfill
	\begin{subfigure}[t]{0.33\textwidth}
		\label{fig:ccc}
		\includegraphics[scale=0.27]{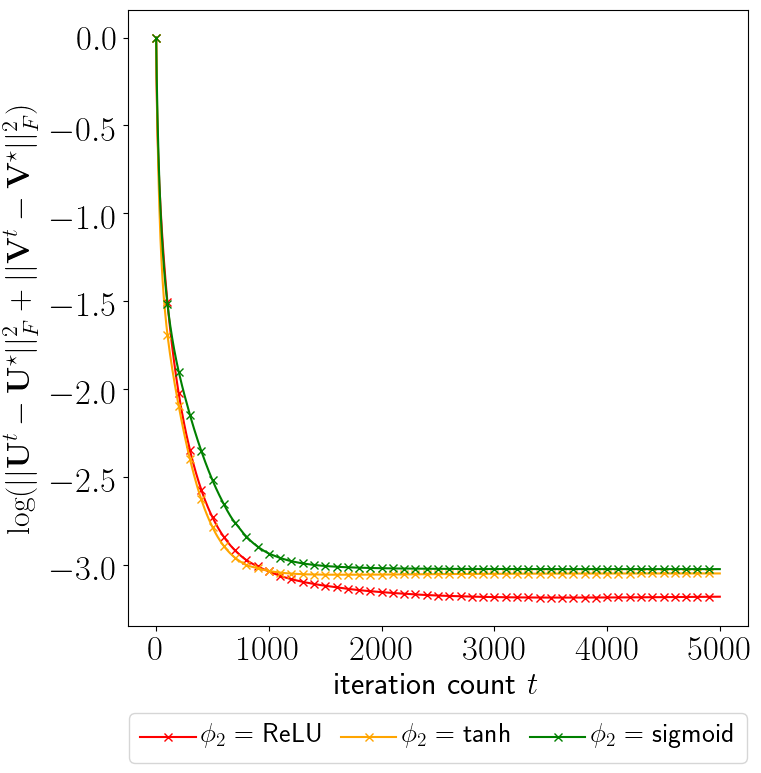}
		\caption{Poisson Model}
	\end{subfigure}
	\caption{Local linear convergence of gradient descent on synthetic data sets.} \label{fig:conver}
\end{figure*}

\subsection{Robustness to Model Misspecification}
\label{sec::misspe}

\begin{figure*}[t] \centering
	\begin{subfigure}[t]{0.33\textwidth}
		\vspace{0pt}
		\includegraphics[scale=0.27]{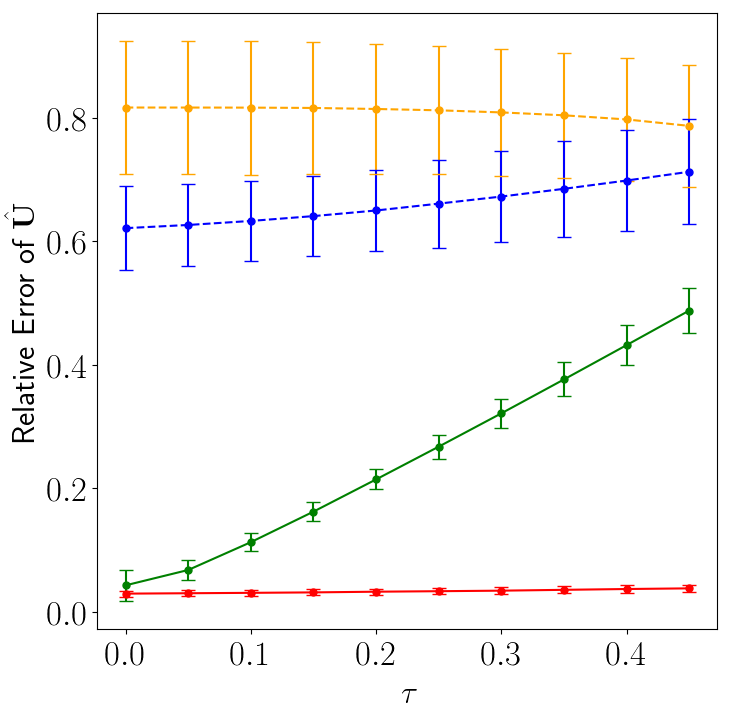}
		\caption{}
	\end{subfigure}\hfill
	\begin{subfigure}[t]{0.33\textwidth}
		\vspace{0pt}
		\includegraphics[scale=0.27]{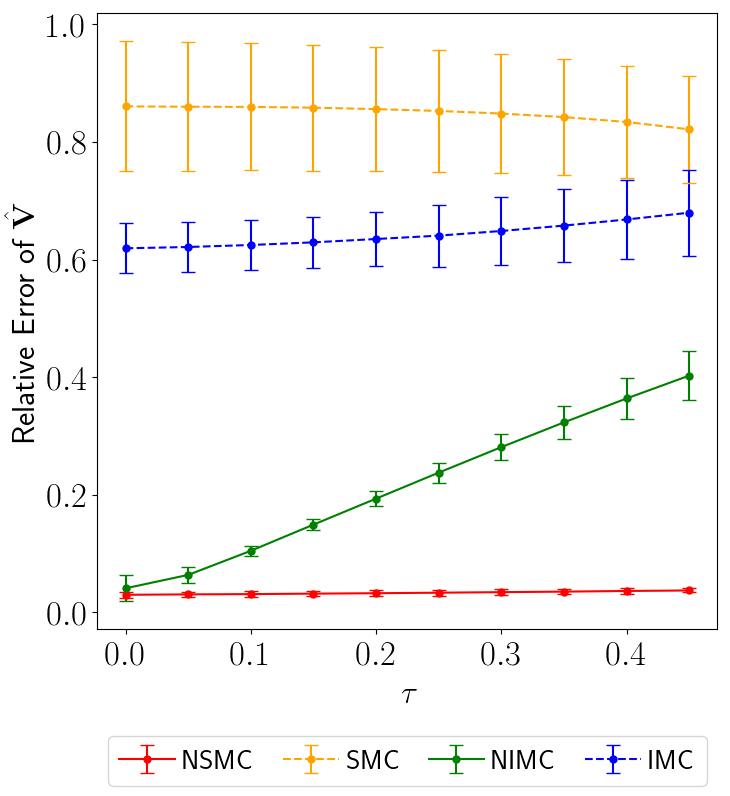}
		\caption{}
	\end{subfigure}\hfill
	\begin{subfigure}[t]{0.33\textwidth}
		\vspace{0pt}
		\includegraphics[scale=0.27]{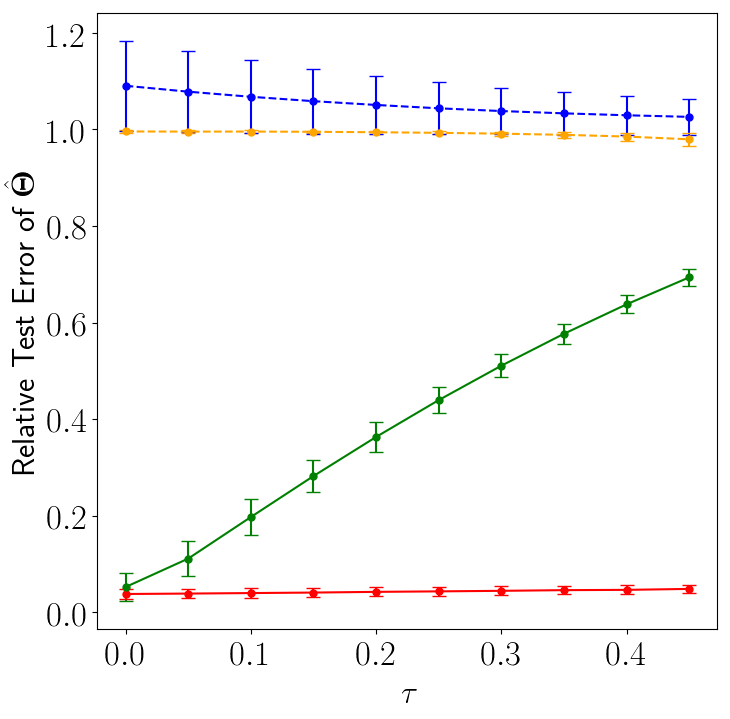}
		\caption{}
	\end{subfigure}
	\caption{Relative Error of NSMC, SMC, NIMC and IMC.	The plot shows how relative error of estimations given by each method varies with parameter $\tau$, which introduces model misspecification in	the Gaussian model. We see that NSMC gives accurate and robust estimation, while NIMC suffers from model misspecification. SMC, IMC fail to learn the non-linear embeddings and give unsatisfactory estimations for all $\tau$.}  \label{fig:gaussian}
\end{figure*}

We generate synthetic data with model misspecification and compare the performance of estimators given by NSMC, SMC, NIMC and IMC. We fix $d = d_1 = d_2 = 50$, $r=3$, $n_1=n_2=400$, and use ReLU as the activation function for NSMC and NIMC. For NSMC and SMC, we randomly generate two independent sample sets with $m=1000$ observations, which are denoted as $\Omega$ and $\Omega^\prime$. The observed sample set for NIMC and IMC are set to be the union $\Omega\cup\Omega^\prime$. For NSMC and SMC, we minimize the proposed pseudo-likelihood objective. For NIMC and IMC, we minimize the square loss as suggested by \citet{Zhong2019Provable}. We apply gradient descent starting from a random initialization near the ground truth $(\mathbf{U}^{\star},\mathbf{V}^{\star})$, in order to guarantee convergence of all methods.  We evaluate the estimated matrix $\hUb$ using the relative approximation error $\mathcal{E}_{\hUb} = ||\hUb -
\mathbf{U}^\star||_F/||\mathbf{U}^\star||_F$, with $\mathcal{E}_{\hVb}$ defined similarly. We also evaluate the performance of a solution $(\hUb,\hVb)$ on recovering the parametric component $\bm{\hat\Theta}$ using relative test error
$\mathcal{E}_{\bm{\hat\Theta}} = \sqrt{\sum_{(\mathbf{x},\mathbf{z})\in\Omega_{t}}(\bm{\hat\Theta} -\bm{\Theta}^\star)^2/ \sum_{(\mathbf{x},\mathbf{z})\in\Omega_{t}}
	{\bm{\Theta}^\star}^2}$, where $\bm{\hat\Theta} = \langle\phi(\hUb^T \mathbf{x}), \phi(\hVb^{T}
\mathbf{z})\rangle$,
$\bm{\Theta}^\star = \langle\phi(\mathbf{U}^{\star T} \mathbf{x}),
\phi(\mathbf{V}^{\star T} \mathbf{z})\rangle$, and $\Omega_t$ is a newly sampled test data set.  For each setting below, we report
results averaged over 10 runs.

\begin{table}[t]
	\resizebox{\columnwidth}{!}{%
		\begin{tabular}{|c|c|c|c|c|c|c|c|c|c|}
			\hline
			\rule{0pt}{2.5ex}
			\multirow{2}{*}{Method}&
			\multicolumn{3}{c|}{$\tau = 0$}&\multicolumn{3}{c|}{ $\tau = 0.2$}&\multicolumn{3}{c|}{ $\tau = 0.4$}\cr\cline{2-10}
			\rule{0pt}{2.5ex}
			&$\mathcal{E}_{\hUb}$&$\mathcal{E}_{\hVb}$&$\mathcal{E}_{\bm{\hat\Theta}}$&$\mathcal{E}_{\hUb}$&$\mathcal{E}_{\hVb}$&$\mathcal{E}_{\bm{\hat\Theta}}$&$\mathcal{E}_{\hUb}$&$\mathcal{E}_{\hVb}$&$\mathcal{E}_{\bm{\hat\Theta}}$\cr
			\hline
			\rule{0pt}{2.5ex}
			NSMC&$\bm{0.0291}$&$\bm{0.0299}$&$\bm{0.0383}$&$\bm{0.0322}$&$\bm{0.0326}$&$\bm{0.0427}$&$\bm{0.0365}$&$\bm{0.0365}$&$\bm{0.0468}$\cr
			SMC&0.8163&0.8603&0.9954&0.8139&0.8557&0.9938&0.7970&0.8338&0.9849\cr
			NIMC&0.0425&0.0410&0.0527&0.2140&0.1935&0.3633&0.4315&0.3638&0.6377\cr
			IMC&0.6209&0.6191&1.0899&0.6495&0.6349&1.0503&0.6981&0.6681&1.0289\cr\hline
			
		\end{tabular}
	}
	\caption{Relative error in the Gaussian model.} \label{table:gaussian}
\end{table}

\begin{table}[t]
	\resizebox{\columnwidth}{!}{%
		\begin{tabular}{|c|c|c|c|c|c|c|c|c|c|}
			\hline
			\rule{0pt}{2.5ex}
			\multirow{2}{*}{Method}&
			\multicolumn{3}{c|}{$N_B = 100$}&\multicolumn{3}{c|}{ $N_B = 200$}&\multicolumn{3}{c|}{ $N_B = 500$}\cr\cline{2-10}
			\rule{0pt}{2.5ex}
			&$\mathcal{E}_{\hUb}$&$\mathcal{E}_{\hVb}$&$\mathcal{E}_{\bm{\hat\Theta}}$&$\mathcal{E}_{\hUb}$&$\mathcal{E}_{\hVb}$&$\mathcal{E}_{\bm{\hat\Theta}}$&$\mathcal{E}_{\hUb}$&$\mathcal{E}_{\hVb}$&$\mathcal{E}_{\bm{\hat\Theta}}$\cr
			\hline
			\rule{0pt}{2.5ex}
			NSMC&$\bm{0.0354}$&$\bm{0.0352}$&$\bm{0.0464}$&$\bm{0.0329}$&$\bm{0.0327}$&$\bm{0.0441}$&$\bm{0.0301}$&$\bm{0.0297}$&$\bm{0.0381}$\cr
			SMC&0.8629&0.8896&0.9956&0.9402&0.9493&0.9988&0.9843&0.9873&0.9998\cr
			NIMC&0.8221&0.6151&0.9364&0.8212&0.6201&0.9248&0.8236&0.6138&0.9259\cr
			IMC&0.8137&0.7934&1.0044&0.8302&0.7781&1.0038&0.8205&0.7891&1.0078\cr\hline
			
		\end{tabular}
	}
	\caption{Relative error in the Binomial model.}\label{table:binom}
\end{table}

\begin{table}[t]
	\resizebox{\columnwidth}{!}{%
		\begin{tabular}{|c|c|c|c|c|c|c|c|c|c|}
			\hline
			\rule{0pt}{2.5ex}
			\multirow{2}{*}{Method}&
			\multicolumn{3}{c|}{ReLU+ReLU}&\multicolumn{3}{c|}{ ReLU+sigmoid}&\multicolumn{3}{c|}{ ReLU+tanh}\cr\cline{2-10}
			\rule{0pt}{2.5ex}
			&$\mathcal{E}_{\hUb}$&$\mathcal{E}_{\hVb}$&$\mathcal{E}_{\bm{\hat\Theta}}$&$\mathcal{E}_{\hUb}$&$\mathcal{E}_{\hVb}$&$\mathcal{E}_{\bm{\hat\Theta}}$&$\mathcal{E}_{\hUb}$&$\mathcal{E}_{\hVb}$&$\mathcal{E}_{\bm{\hat\Theta}}$\cr
			\hline
			\rule{0pt}{2.5ex}
			NSMC&$\bm{0.0691}$&$\bm{0.0718}$&$\bm{0.0975}$&$\bm{0.0661}$&$\bm{0.0617}$&$\bm{0.0631}$&$\bm{0.0442}$&$\bm{0.0457}$&$\bm{0.0727}$\cr
			SMC&0.3696&0.3852&0.5559&0.7855&0.8229&0.9757&0.2812&0.3019&0.4500\cr
			NIMC&2.2479&2.3282&10.7018&1.3024&0.4078&1.4877&0.5203&0.2522&0.5595\cr
			IMC&1.5717&1.5889&5.7169&0.5745&0.6368&1.0922&0.3604&0.3847&1.1643\cr\hline
			
		\end{tabular}
	}
	\caption{Relative error in the Poison model.}\label{table:poisson}
\end{table}

{\bf Gaussian model.}  We introduce model misspecification by sampling $y$ from $y \sim\mathcal{N}\left((1-\tau)^2\cdot\mathbf{\Theta},(1-\tau)^2\right)$. Parameter $\tau$ is introduced to modify the impact of model misspecification. We summarize the relative errors in Table~\ref{table:gaussian} and Figure~\ref{fig:gaussian}. When $\tau = 0$, there is no model misspecification and NSMC and NIMC achieve comparable relative approximation errors. As $\tau$ increases, the relative approximation errors of NIMC grow rapidly due to the increase of model misspecification.  However, NSMC gives robust estimations. SMC and IMC serve as bilinear modeling baselines that fail to learn in the nonlinear embedding setting.

{\bf Binomial model.}
We sample $y \sim B(N_B,\frac{\exp (\mathbf{\Theta})}{1+\exp (\mathbf{\Theta})})$ and apply NSMC and SMC with original attributes $y$. For NIMC and IMC, we first do variance-stabilizing transformation $\tilde y = \arcsin{(\frac{y}{N_B})}$ as the data preprocessing step, inspired by what people might do for non-Gaussian data in practical applications. From Table~\ref{table:binom}, NSMC achieves the best estimating result in each setting, while other methods fail to learn the embeddings with a binomial model.

{\bf Poisson model.}
We generate $y \sim Pois(\exp (\mathbf{\Theta}))$,
where the activation functions are $\phi_1 =$ ReLU and $\phi_2 \in$ \{ReLU, sigmoid, tanh\}. Due to model misspecification, we apply transformation $\tilde y = \sqrt{y}$ for NIMC and IMC. The activation function of NIMC is set to be the same as $\phi_2$. We see from Table~\ref{table:poisson}
that NSMC achieves the best estimating result, while other methods fail to recover the parameters.

\subsection{Clustering of Embeddings}
\label{sec::cluster}

We generate synthetic data with clustered embeddings and compare the performance of NSMC and NIMC on learning the true embedding clustering. We fix $d = d_1 = d_2 = 30$, $r=2$, $n_1=n_2=400$, and choose $\text{tanh}$ as the activation function. We generate features $\xb$ and $\zb$ independently from a Gaussian mixture model with four components, resulting in the ground-truth embedding clustering with four components. We sample $y$ from a binomial model with $N_B =
20$. We fix observed sample size $m = 1000$ and apply NSMC and NIMC to get the estimated $\hUb$ and $\hVb$, respectively. We plot the top 2 left singular vectors $(\hat\iota_1,\hat\iota_2)$ of $\phi_1(\hUb^T \xb)$ for NSMC and NIMC, respectively, where the points are colored according to the ground-truth clustering. We also plot the top 2 left singular vectors $(\iota_1^\star,\iota_2^\star)$ of the ground-truth embeddings $\phi_1(\tUb^T \xb)$. Similar plots for feature $\zb$ are shown as well. We see from Figure~\ref{fig:cluster} that NIMC fails to find the ground-truth embeddings due to model misspecification, while NSMC gives robust estimation and recovers the ground-truth embeddings.

To quantitatively evaluate the performance, we apply the k-means clustering to the left singular vectors. We define the clustering error following \citet{Zhong2019Provable} as
\begin{equation}
\label{eq:error}
\frac{2}{n(n-1)}\left(\sum_{(i, j):
	\aleph_{i}^{\star}=\aleph_{j}^{\star}} 1_{\aleph_{i} \neq
	\aleph_{j}}+\sum_{(i, j): \aleph_{i}^{\star} \neq
	\aleph_{j}^{\star}} 1_{\aleph_{i}=\aleph_{j}}\right),
\end{equation}
where $\aleph^\star$ is the ground-truth clustering and $\aleph$ is the predicted clustering. As a result, NIMC attains clustering error $0.0596$ and $0.1725$ for $\xb$ and $\zb$ respectively. NSMC achieves a better performance with clustering error $0.0196$ and $0.0147$ for $\xb$ and $\zb$ respectively.

\begin{figure*}[t] \centering
	\begin{subfigure}[t]{0.32\textwidth}
		\vspace{0pt}
		\includegraphics[scale=0.26]{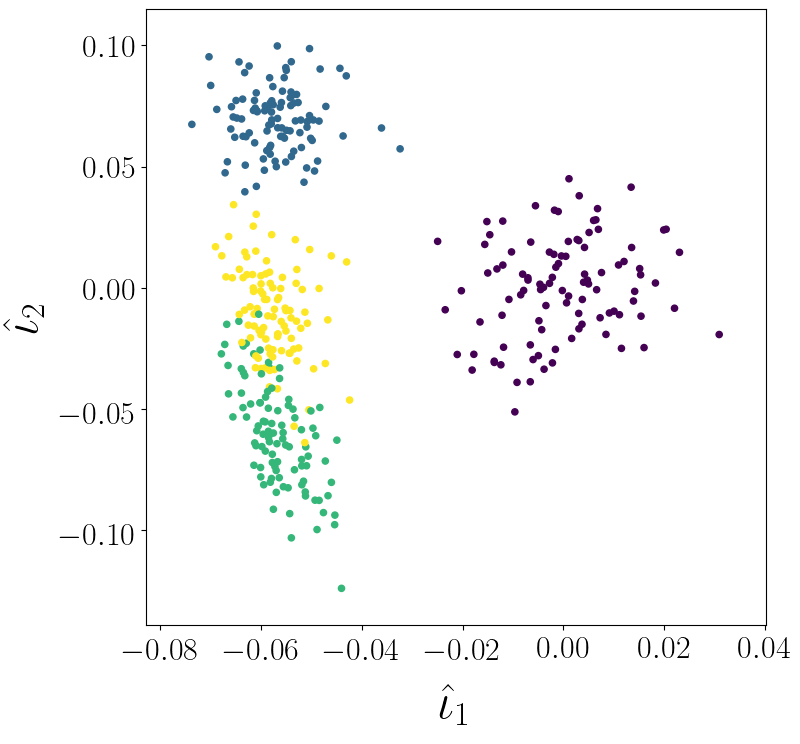}
		\caption{NIMC}
	\end{subfigure}\hfill
	\begin{subfigure}[t]{0.32\textwidth}
		\vspace{0pt}
		\includegraphics[scale=0.26]{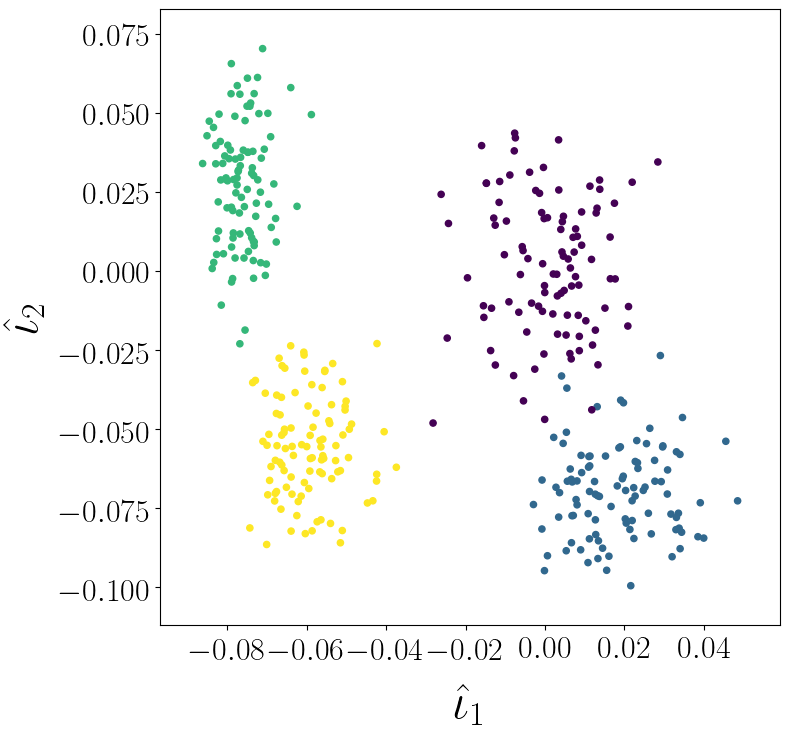}
		\caption{NSMC}
	\end{subfigure}\hfill
	\begin{subfigure}[t]{0.32\textwidth}
		\vspace{0pt}
		\includegraphics[scale=0.26]{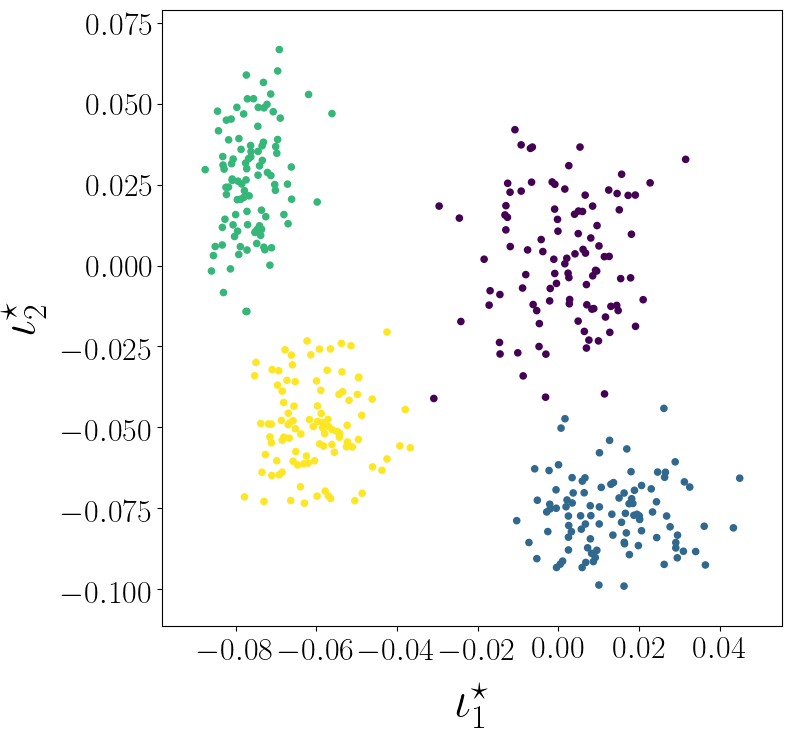}
		\caption{Ground Truth}
	\end{subfigure}
	
	\begin{subfigure}[t]{0.32\textwidth}
		\vspace{0pt}
		\includegraphics[scale=0.26]{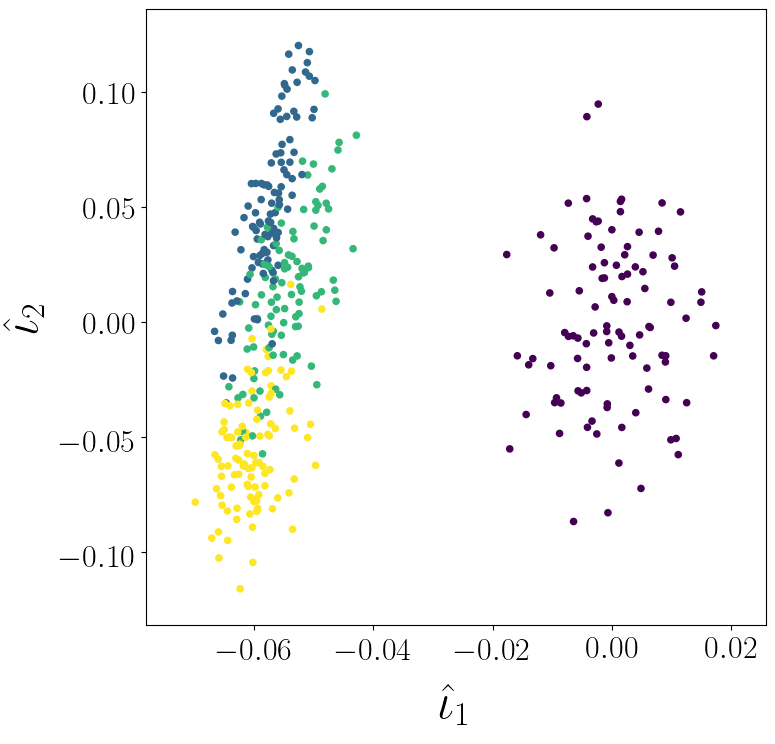}
		\caption{NIMC}
	\end{subfigure}\hfill
	\begin{subfigure}[t]{0.32\textwidth}
		\vspace{0pt}
		\includegraphics[scale=0.26]{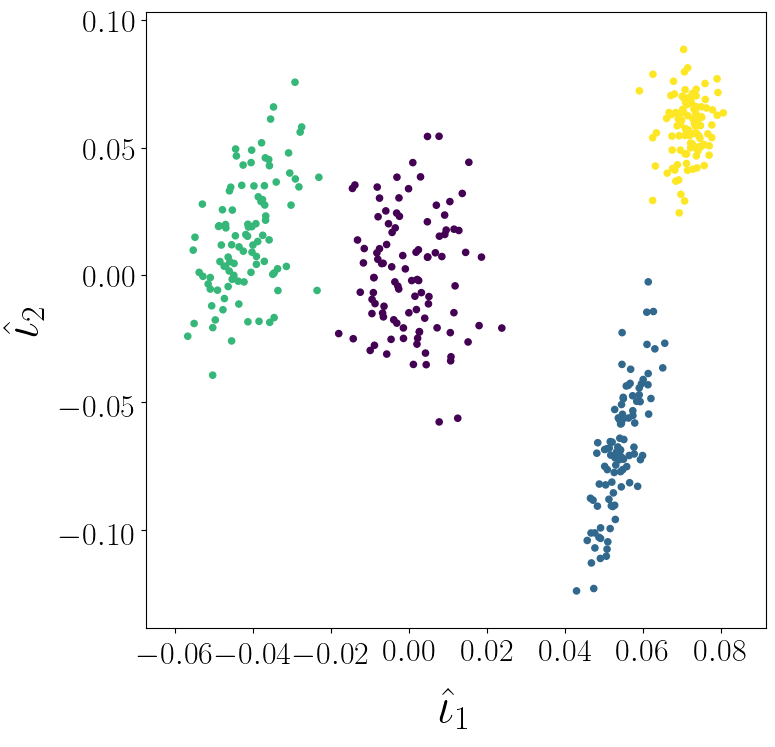}
		\caption{NSMC}
	\end{subfigure}\hfill
	\begin{subfigure}[t]{0.32\textwidth}
		\vspace{0pt}
		\includegraphics[scale=0.26]{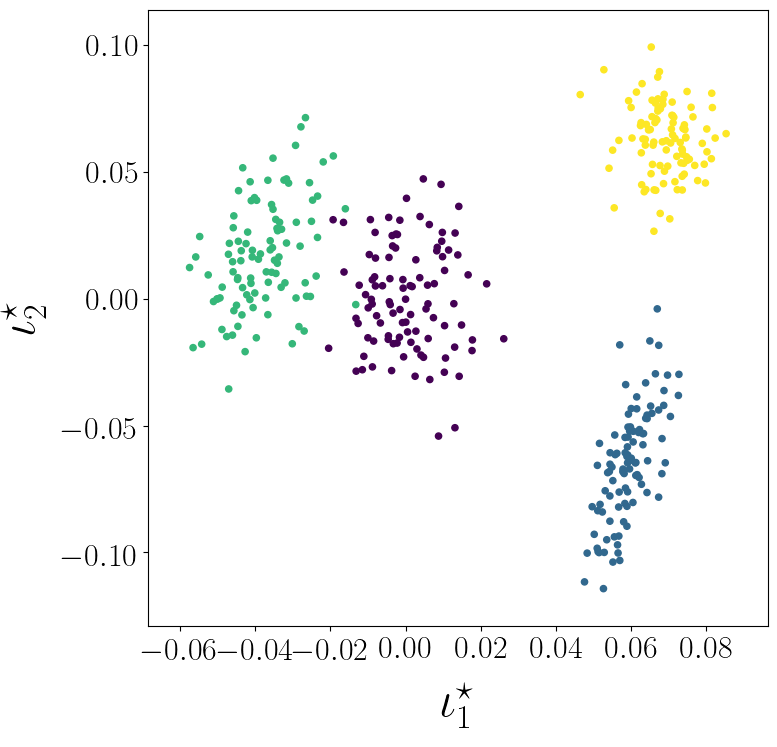}
		\caption{Ground Truth}
	\end{subfigure}
	\caption{The comparison of learned embeddings based on NIMC and NSMC, with the ground-truth embeddings. The first row shows embeddings of $\xb$, while the second row shows embeddings of $\zb$. The points are colored according to the ground-truth clustering.} \label{fig:cluster}
\end{figure*}

\subsection{Semi-supervised Clustering}
\label{sec::semi}

We further illustrate the superior performance of NSMC over NIMC with real-world data. Following the experimental setting in \citet{Zhong2019Provable}, we apply NSMC and NIMC to a semi-supervised clustering problem, where we only have one kind of features, $\xb \in \mR^{d_1}$, on a set of items. The edge attribute $y_{ij}=1$, if the $i$-th item and $j$-th item are similar, and $y_{ij}=0$, if they are dissimilar. To apply NSMC and NIMC, we set $\xb = \zb$, $\phi_1=\phi_2=\phi$, and assume $\tUb = \tVb$. We initialize $\Ub^0 = \Vb^0$ as the same random Gaussian matrix and apply gradient descent to ensure $\Ub^t = \Vb^t$ during training. After training, we apply k-means clustering to the top $r$ left singular vectors of $\phi(\hUb^T \xb)$. We follow \citet{Zhong2019Provable} and again use the clustering error defined by \eqref{eq:error}. We set the activation function $\phi$ to be tanh for all data sets. For NSMC, we first uniformly sample two independent sets of items with $n_1=n_2=1000$. Then we generate independent observation sets $\Omega$ and $\Omega^\prime$ with size $m=5000$. For NIMC, the observed dataset is set to be the union $\Omega\cup\Omega^\prime$. We consider three datasets: \emph{Mushroom}, \emph{Segment} and \emph{Covtype} \citep{Dua:2019}, and regard items with the same label as similar ($y_{ij}=1$). \emph{Covtype} dataset is subsampled first to balance the size of each cluster. As shown in Table~\ref{table:semi}, for linear separable dataset \emph{Mushroom}, both NSMC and NIMC achieve perfect clustering. For the other two datasets, NSMC achieves better clustering results than NIMC.

\begin{table}[t]
	\centering
	\begin{tabular}{|c|c|c||c|c|}
		\hline
		Dataset&d&r&NIMC&NSMC\cr\hline
		Mushroom&112&2&$\bm{0}$&$\bm{0}$\cr\hline
		Segment&19&7&0.0971&$\bm{0.0427}$\cr\hline
		Covtype&54&7&0.1931&$\bm{0.1373}$\cr\hline
	\end{tabular}
	\caption{Clustering error on real-world data.}\label{table:semi}
\end{table}

\section{Conclusion}\label{sec:6}

We studied the nonlinear bipartite graph representation learning problem. We formalized the representation learning problem as a statistical parameter estimation problem in a semiparametric model. In particular, the edge attributes, given node features, are assumed to follow an exponential family distribution with unknown base measure. The parametric component of the model is assumed to be the proximity of outputs of one-layer neural network, whose inputs are node representations. In this setting, learning embedding vectors is equivalent to estimating two low-rank weight matrices $(\tUb,\tVb)$. Using the rank-order decomposition technique, we proposed a quasi-likelihood function, and proved that GD with constant step size achieves local linear convergence rate. The sample complexity is linear in dimensions up to a logarithmic factor, which matches existing results in matrix completion. However, our estimator is robust to model misspecification within exponential family due to the adaptivity to the base measure. We also provided numerical simulations and real experiments to corroborate the main theoretical results, which demonstrated superior performance of our method over existing approaches.

One potential extension is to consider a more general distribution for node representations. For example, when node representations follow a heavy-tailed distribution, it is not clear whether we can still recover $(\tUb, \tVb)$ with the same convergence rate. In addition, using two-layer or even deep neural networks for encoders in our semiparametric model, while still providing theoretical guarantee is another interesting extension.

\section{Technical Proofs}\label{sec:7}

In this section, we provide proofs of lemmas in the main text.
Auxiliary results are presented in the appendix.

\subsection{Proof of Lemma~\ref{lem:B:1}}
For any pair $(y_k, y_l')$, let $R_{kl}$ denote the rank statistics, and $y^{kl}_{(\cdot)}$ denote the order statistics. We have
\begin{align*}
\mE\big[B_{kl}^\star \mid \xb_{k_1}, \zb_{k_2}, \xb_{l_1}', \zb_{l_2}'\big] = \mE\big[\mE[B_{kl}^\star \mid \xb_{k_1}, \zb_{k_2}, \xb_{l_1}', \zb_{l_2}', y^{kl}_{(\cdot)} ]\mid \xb_{k_1}, \zb_{k_2}, \xb_{l_1}', \zb_{l_2}'\big].
\end{align*}
Moreover, as shown in \eqref{equ:rank:order:decom},
\begin{align*}
P(R_{kl} \mid y^{kl}_{(\cdot)}, \xb_{k_1}, \zb_{k_2}, \xb_{l_1}', \zb_{l_2}')
&= \frac{\exp\rbr{y_k\tThe_{k_1k_2} + y'_l\tThep_{l_1l_2}}}{\exp(y_k\tThe_{k_1k_2} + y'_l\tThep_{l_1l_2}) + \exp(y_k\tThep_{l_1l_2} + y'_l\tThe_{k_1k_2}) }\\
&= \frac{1}{1 + \exp\big(-(y_k-y'_l)(\tThe_{k_1k_2} - \tThep_{l_1l_2})\big) }.
\end{align*}
Thus,
\begin{align*}
\mE\big[B_{kl}^\star & \mid \xb_{k_1}, \zb_{k_2}, \xb_{l_1}', \zb_{l_2}', y^{kl}_{(\cdot)}\big]\\
& = \frac{y_k - y_l'}{1 + \exp\big((y_k - y_l')(\tThe_{k_1k_2} - \tThep_{l_1l_2})\big) }\cdot P(R_{kl} \mid y^{kl}_{(\cdot)}, \xb_{k_1}, \zb_{k_2}, \xb_{l_1}', \zb_{l_2}')\\
& \quad  + \frac{y_l' - y_k}{1 + \exp\big((y_l' - y_k)(\tThe_{k_1k_2} - \tThep_{l_1l_2})\big) }\cdot\big(1 - P(R_{kl} \mid y^{kl}_{(\cdot)}, \xb_{k_1}, \zb_{k_2}, \xb_{l_1}', \zb_{l_2}')\big)\\
& = \frac{y_k - y_l'}{\rbr{1 + \exp\big((y_k - y_l')(\tThe_{k_1k_2} - \tThep_{l_1l_2})\big) }\rbr{1 + \exp\big(-(y_k-y'_l)(\tThe_{k_1k_2} - \tThep_{l_1l_2})\big) }}\\
& \quad + \frac{y_l' - y_k}{\rbr{1 + \exp\big(-(y_k-y'_l)(\tThe_{k_1k_2} - \tThep_{l_1l_2})\big) }\rbr{1 + \exp\big((y_k - y_l')(\tThe_{k_1k_2} - \tThep_{l_1l_2})\big) }}\\
& = 0,
\end{align*}
which completes the proof.

\subsection{Proof of Lemma~\ref{lem:B:2}}\label{sec:7.2}

Let us first introduce additional notations. Suppose QR decomposition of $\tUb, \tVb$ is $\tUb = \Qb_1\Rb_1$ and $\tVb = \Qb_2\Rb_2$, respectively, with $\Qb_i\in\mR^{d_i\times r}$ and $\Rb_i\in\mR^{r\times r}$ for $i=1,2$. Let $\Qb_i^{\perp}\in\mR^{d_i\times (d_i-r)}$ be the orthogonal complement of $\Qb_i$. For any vectors $\ba = (\ba_1; \ldots; \ba_r)$ and $\bb = (\bb_1; \ldots; \bb_r)$ such that $\ba_p \in\mR^{d_1}$, $\bb_p\in\mR^{d_2}$ for $p\in[r]$ and $\|\ba\|_2^2 + \|\bb\|_2^2 = 1$, we express each component by $\ba_p = \Qb_1 \br_{1p} + \Qb^{\perp}_1\bs_{1p}$ and $\bb_p = \Qb_2 \br_{2p} + \Qb^{\perp}_2\bs_{2p}$, and let $\br_i = (\br_{i1}, \ldots, \br_{ir})\in\mR^{r\times r}$ and $\bs_i = (\bs_{i1}, \ldots, \bs_{ir})\in\mR^{(d_i-r)\times r}$. Further, we let $\bt_i = (\bt_{i1}, \ldots, \bt_{ir})\in\mR^{r\times r}$ with $\bt_{ip} = \Rb^{-1}_i\br_{ip}$, and also let $\bbt_i\in\mR^{r\times r}$ denote the matrix that replaces the diagonal entries of $\bt_i$ by $0$. Lastly, for $i = 1, 2$ and variable $x\sim \mN(0, 1)$, we define following quantities
\begin{align*}
\tau_{i, j, k} = \mE[(\phi_i(x))^jx^k], \quad \quad \tau_{i, j, k}' = \mE[(\phi_i'(x))^jx^k], \quad\quad \tau_i'' = \mE[\phi_i(x)\phi_i'(x)x].
\end{align*}
Using the above notations,
\begin{align}\label{c:2}
\frac{1}{2}&
\begin{pmatrix}
\ba ^T & \bb^T
\end{pmatrix}\mE\bigg[\begin{pmatrix}
	\td - \tdp\\
	\tp - \tpp
	\end{pmatrix}\begin{pmatrix}
	\td - \tdp\\
	\tp - \tpp
	\end{pmatrix}^T\bigg]\begin{pmatrix}
\ba\\
\bb
\end{pmatrix} \nonumber\\
&= \frac{1}{2} \mE\bigg[\bigg(\sum_{p=1}^{r}\big(\phi_1'(\tuT_p\xb)\phi_2(\tvT_p\zb)\ba_p^T\xb  + \phi_1(\tuT_p\xb)\phi_2'(\tvT_p\zb)\bb_p^T\zb \big) \nonumber\\
& \quad\quad - \sum_{p=1}^r \big(\phi_1'(\tuT_p\xb')\phi_2(\tvT_p\zb')\ba_p^T\xb'+ \phi_1(\tuT_p\xb')\phi_2'(\tvT_p\zb')\bb_p^T\zb'\big)\bigg)^2\bigg] \nonumber\\
&= \Var\bigg(\sum_{p=1}^{r}\big(\phi_1'(\tuT_p\xb)\phi_2(\tvT_p\zb)\ba_p^T\xb  + \phi_1(\tuT_p\xb)\phi_2'(\tvT_p\zb)\bb_p^T\zb \big)\bigg) \nonumber\\
&= \Var\bigg(\sum_{p=1}^r\big(\phi_1'(\tuT_p\xb)\phi_2(\tvT_p\zb)\xb^T\Qb_1\br_{1p}  + \phi_1(\tuT_p\xb)\phi_2'(\tvT_p\zb)\zb^T\Qb_2\br_{2p} \big) \nonumber\\
&\quad\quad  + \sum_{p=1}^r\phi_1'(\tuT_p\xb)\phi_2(\tvT_p\zb)\xb^T\Qb_1^{\perp}\bs_{1p} + \sum_{p=1}^r\phi_1(\tuT_p\xb)\phi_2'(\tvT_p\zb)\zb^T\Qb_2^{\perp}\bs_{2p}\bigg) \nonumber\\
&= \Var\bigg(\sum_{p=1}^r\big(\phi_1'(\tuT_p\xb)\phi_2(\tvT_p\zb)\xb^T\Qb_1\br_{1p}  + \phi_1(\tuT_p\xb)\phi_2'(\tvT_p\zb)\zb^T\Qb_2\br_{2p} \big)\bigg) \nonumber\\
& \quad\quad + \Var\bigg(\sum_{p=1}^r\phi_1'(\tuT_p\xb)\phi_2(\tvT_p\zb)\xb^T\Qb_1^{\perp}\bs_{1p}\bigg) + \Var\bigg(\sum_{p=1}^r\phi_1(\tuT_p\xb)\phi_2'(\tvT_p\zb)\zb^T\Qb_2^{\perp}\bs_{2p}\bigg) \nonumber\\
& \eqqcolon \I_1 + \I_2 + \I_3,
\end{align}
where we have used the independence among $\xb^T\Qb_1\br_{1p}$, $\xb^T\Qb_1^\perp \bs_{1p}$, $\zb^T\Qb_2\br_{2p}$ and $\zb^T\Qb_2^\perp\bs_{2p}$. By Lemma \ref{lem:E:1}, there exists a constant $C_1$ not depending on $(\tUb, \tVb)$ such that
\begin{align}\label{c:3}
\I_2 + \I_3 \geq \frac{C_1}{\barkap(\tUb)\barkap(\tVb)}\rbr{\|\bs_1\|_F^2 + \|\bs_2\|_F^2}.
\end{align}
For term $\I_1$, let us denote the inside variable as
\begin{align}\label{equ:def:g}
g(\tUb^T\xb, \tVb^T\zb) = \sum_{p=1}^r\big(\phi_1'(\tuT_p\xb)\phi_2(\tvT_p\zb)\xb^T\tUb\bt_{1p}  + \phi_1(\tuT_p\xb)\phi_2'(\tvT_p\zb)\zb^T\tVb\bt_{2p}\big).
\end{align}
Using Lemma \ref{lem:G:1}, Assumption \ref{ass:1}, and independence among $\xb, \xb', \zb, \zb'$,
\begin{align}\label{c:4}
\I_1 & =  \Var(g(\tUb^T\xb, \tVb^T\zb)) = \frac{1}{2}\mE\sbr{\big(g(\tUb^T\xb, \tVb^T\zb) - g(\tUb^T\xb', \tVb^T\zb')\big)^2} \nonumber\\
& \geq \frac{1}{2\barkap(\tUb)\barkap(\tVb)}\mE\big[\big(g(\bxb, \bzb) -g(\bxb', \bzb')\big)^2\big] = \frac{1}{\barkap(\tUb)\barkap(\tVb)}\Var(g(\bxb, \bzb)).
\end{align}
Here, $\bxb, \bxb', \bzb, \bzb'\stackrel{i.i.d}{\sim}\mN(0, I_r)$. We separate into the following two cases.

\noindent{\bf Case 1,} $\phi_1, \phi_2 \in\{\text{sigmoid}, \text{tanh}\}$. By Lemma \ref{lem:E:0}, we plug the lower bound of $\Var(g(\bxb, \bzb))$ into \eqref{c:4} and know that there exists a constant $C_2$ not depending on $(\tUb, \tVb)$, such that
\begin{align*}
\I_1
& \geq  \frac{C_2}{\barkap(\tUb)\barkap(\tVb)}\rbr{\|\bt_1\|_F^2 + \|\bt_2\|_F^2}  \\
& = \frac{C_2}{\barkap(\tUb)\barkap(\tVb)}\rbr{\|\Rb_1^{-1}\br_1\|_F^2 + \|\Rb_2^{-1}\br_2\|_F^2}   \\
& \geq  \frac{C_2}{\barkap(\tUb)\barkap(\tVb)\max(\|\tUb\|_2^2, \|\tVb\|_2^2)}\rbr{\|\br_1\|_F^2 + \|\br_2\|_F^2}.
\end{align*}
Combining the above display with \eqref{c:2} and \eqref{c:3}, Minimizing over the set $\{(\ba, \bb): \|\ba\|_F^2 + \|\bb\|_F^2 = 1\}$,
\begin{align*}
\lambda_{\min}\rbr{\mE\sbr{\begin{pmatrix}
	\td - \tdp\\
	\tp - \tpp
	\end{pmatrix}\begin{pmatrix}
	\td - \tdp\\
	\tp - \tpp
	\end{pmatrix}^T}} \geq \frac{2\min(C_1, C_2)}{\barkap(\tUb)\barkap(\tVb)\max(\|\tUb\|_2^2, \|\tVb\|_2^2)}.
\end{align*}
This completes the proof for Case 1.

\noindent{\bf Case 2,} either $\phi_1$ or $\phi_2$ is ReLU. By Lemma \ref{lem:E:0}, we have
\begin{align*}
\I_1 \geq \frac{C_3}{\barkap(\tUb)\barkap(\tVb)}\rbr{\|\bbt_1\|_F^2 + \|\bbt_2\|_F^2 + \|\diag(\bt_1) + \diag(\bt_2)\|_2^2}.
\end{align*}
The above display, together with \eqref{c:2} and \eqref{c:3}, leads to
\begin{multline*}
\begin{pmatrix}
\ba ^T & \bb^T
\end{pmatrix}\mE\sbr{\begin{pmatrix}
	\td - \tdp\\
	\tp - \tpp
	\end{pmatrix}\begin{pmatrix}
	\td - \tdp\\
	\tp - \tpp
	\end{pmatrix}^T}\begin{pmatrix}
\ba\\
\bb
\end{pmatrix} \\
\geq \frac{2\min(C_1, C_2)}{\barkap(\tUb)\barkap(\tVb)}\rbr{\|\bbt_1\|_F^2 + \|\bbt_2\|_F^2 + \|\diag(\bt_1) + \diag(\bt_2)\|_2^2 + \|\bs_1\|_F^2 + \|\bs_2\|_F^2}.
\end{multline*}
Since the first row of $\tUb$ is fixed, we minimize over the set $\{(\ba, \bb): \|\ba\|_F^2 + \|\bb\|_F^2 = 1,  \eb_1^T\ba_p = 0, \forall p\in[r]\}$. Equivalently, the right hand side has the following optimization problem
\begin{align*}
\gamma_\tUb \coloneqq \min_{\bt_1, \bt_2, \bs_1, \bs_2} \text{\ \ } &\|\bbt_1\|_F^2 + \|\bbt_2\|_F^2 + \|\diag(\bt_1) + \diag(\bt_2)\|_2^2 + \|\bs_1\|_F^2 + \|\bs_2\|_F^2\\
\text{s.t.\ \ \ \ }& \Rb_1\bt_1 = \br_1,\quad  \Rb_2\bt_2 = \br_2,\\
& \|\br_1\|_F^2 + \|\br_2\|_F^2 + \|\bs_1\|_F^2 + \|\bs_2\|_F^2 = 1,\\
& \eb_1^T\Qb_1\br_{1} + \eb_1^T\Qb_1^{\perp}\bs_1 = \0.
\end{align*}
By Theorem D.6. in \citet{Zhong2018Nonlinear},
\begin{align*}
\gamma_\tUb \geq \frac{\|\eb_1^T\tUb\|_{\min}^2}{36\max(\|\tUb\|_2^2, \|\tVb\|_2^2)(1 + \|\eb_1^T\tUb\|_2)^2}.
\end{align*}
Thus,
\begin{multline*}
\lambda_{\min}\rbr{\mE\sbr{\begin{pmatrix}
	\td - \tdp\\
	\tp - \tpp
	\end{pmatrix}\begin{pmatrix}
	\td - \tdp\\
	\tp - \tpp
	\end{pmatrix}^T}}\\
\geq \frac{\min(C_1, C_3)\|\eb_1^T\tUb\|_{\min}^2}{18\barkap(\tUb)\barkap(\tVb)\max(\|\tUb\|_2^2, \|\tVb\|_2^2)(1 + \|\eb_1^T\tUb\|_2)^2}.
\end{multline*}
This completes the proof.

\subsection{Proof of Lemma~\ref{lem:B:3}}

The concentration is shown by taking expectation hierarchically. In particular, we let $\nabla^2\bar\mL_1(\Ub, \Vb) = \mE\big[\nabla^2\mL_1(\Ub, \Vb) \mid \D, \D'\big]$, where the expectation is over the random sampling of the entries from $\D$ and $\D'$. Then, we know $\mE[\nabla^2\bar\mL_1(\Ub, \Vb)] = \mE[\nabla^2\mL_1(\Ub, \Vb)]$. Moreover,
\begin{align*}
\|\nabla^2\mL_1(\Ub, \Vb) - \mE[\nabla^2&\mL_1(\tUb, \tVb)]\|_2 \\
\leq& \|\nabla^2\mL_1(\Ub, \Vb) - \nabla^2\bar\mL_1(\Ub, \Vb)\|_2 + \|\nabla^2\bar\mL_1(\Ub, \Vb) - \mE[\nabla^2\bar\mL_1(\Ub, \Vb)]\|_2 \\
&\quad + \|\mE[\nabla^2\mL_1(\Ub, \Vb)] - \mE[\nabla^2\mL_1(\tUb, \tVb)] \|_2\\
\eqqcolon & \J_1 + \J_2 + \J_3.
\end{align*}
Using Lemma~\ref{lem:E:5}, for any $s\geq 1$,
\begin{align*}
P\rbr{\J_1 + \J_2 \gtrsim \beta^2r^{1-q}\sqrt{\frac{s(d_1 + d_2)\log\rbr{r(d_1 + d_2)}}{m\wedge n_1\wedge n_2}}\rbr{\|\Vb\|_F^{2q} + \|\Ub\|_F^{2q}} }\lesssim \frac{1}{(d_1 + d_2)^s}.
\end{align*}
Using Lemma \ref{lem:E:9},
\begin{align*}
\J_3\lesssim \beta^3 r^{\frac{3(1-q)}{2}}\rbr{\|\tVb\|_F^{3q} + \|\tUb\|_F^{3q}}\rbr{\|\Ub - \tUb\|_F^{1-q/2} + \|\Vb - \tVb\|_F^{1-q/2}}.
\end{align*}
Combining the above two displays, using the fact that $\|\Vb\|_F^{2q} + \|\Ub\|_F^{2q}
\lesssim \|\Vb- \tVb\|_F^{2q} + \|\Ub - \tUb\|_F^{2q} + \|\tVb\|_F^{2q} + \|\tUb\|_F^{2q}$, and dropping high order terms, we know that, with probability at least $1 - 1/(d_1+d_2)^s$,
\begin{align*}
  \|\nabla^2\mL_1(\Ub, &\Vb) - \mE\sbr{\nabla^2\mL_1(\tUb, \tVb)}\|_2  \\
  &\lesssim \beta^3r^{\frac{3(1-q)}{2}}\rbr{\|\tVb\|_F^{3q} + \|\tUb\|_F^{3q}} \cdot \\
  &\qquad \rbr{\sqrt{\frac{s(d_1 + d_2)\log\rbr{r(d_1 + d_2)}}{m\wedge n_1\wedge n_2}} + \|\Ub - \tUb\|_F^{1-q/2} + \|\Vb - \tVb\|_F^{1-q/2}}.
\end{align*}
Noting that $ \|\Ub - \tUb\|_F^{1-q/2} + \|\Vb - \tVb\|_F^{1-q/2} \lesssim \rbr{ \|\Ub - \tUb\|_F^{2} + \|\Vb - \tVb\|_F^{2}}^{\frac{2-q}{4}}$
completes the proof.

\subsection{Proof of Lemma~\ref{lem:B:4}}
The proof is similar to that of Lemma \ref{lem:B:3}. Let $\nabla^2\bar\mL_2(\Ub, \Vb) = \mE\big[\nabla^2\mL_2(\Ub, \Vb) \mid \D, \D'\big]$. Then
\begin{align*}
\|\nabla^2\mL_2(\Ub, \Vb) - \mE[\nabla^2&\mL_2(\tUb, \tVb)] \|_2\\
\leq& \|\nabla^2\mL_2(\Ub, \Vb) - \nabla^2\bar\mL_2(\Ub, \Vb)\|_2 + \|\nabla^2\bar\mL_2(\Ub, \Vb) - \mE[\nabla^2\bar\mL_2(\Ub, \Vb)]\|_2 \\
& \quad + \|\mE[\nabla^2\mL_2(\Ub, \Vb)] - \mE[\nabla^2\mL_2(\tUb, \tVb)]\|_2\\
\eqqcolon& \T_1 + \T_2 + \T_3.
\end{align*}
Using Lemma \ref{lem:E:7} and noting that $\|\Vb\|_2^{q_2(1-q_1)} + \|\Ub\|_2^{q_1(1-q_2)}\leq \|\Vb\|_2^q + \|\Ub\|_2^q$, for all $s\geq 1$,
\begin{align*}
P\rbr{\T_1+\T_2 \gtrsim \beta\sqrt{\frac{s(d_1 + d_2)\log\rbr{r(d_1 + d_2)}}{m \wedge n_1\wedge n_2}}\rbr{\|\Vb\|_2^{q} + \|\Ub\|_2^{q}}}\lesssim \frac{1}{(d_1 + d_2)^s}.
\end{align*}
Using Lemma \ref{lem:E:10},
\begin{align*}
\T_3 \lesssim \beta^2r^{\frac{1-q}{2}}\rbr{\|\tVb\|_F^{2q} + \|\tUb\|_F^{2q}}\rbr{\|\Ub - \tUb\|_F^{1-q/2} + \|\Vb - \tVb\|_F^{1-q/2}}.
\end{align*}
Combining the last two displays, we complete the proof.

\subsection{Proof of Lemma~\ref{lem:hessian_neighborhood}}

Note that
\begin{align*}
\|\nabla^2\mL(\Ub_1, \Vb_1) - &\nabla^2\mL(\Ub_2, \Vb_2)\|_2\\
\leq & \|\nabla^2\mL(\Ub_1, \Vb_1) - \mE[\nabla^2\mL(\Ub_1, \Vb_1)]\|_2 +\|\nabla^2\mL(\Ub_2, \Vb_2) - \mE[\nabla^2\mL(\Ub_2, \Vb_2)]\|_2\\
 &\quad + \|\mE[\nabla^2\mL(\Ub_1, \Vb_1)] - \mE[\nabla^2\mL(\Ub_2, \Vb_2)]\|_2.
\end{align*}
Following the proof of Lemma \ref{lem:E:5}, \ref{lem:E:7}, \ref{lem:E:9} and \ref{lem:E:10}, we can show that, for any $s\geq 1$, with probability $1 - 1/(d_1 +d_2)^s$,
\begin{multline*}
\|\nabla^2\mL(\Ub_1, \Vb_1) - \mE[\nabla^2\mL(\Ub_1, \Vb_1)]\|_2\\
\lesssim \beta^2r^{1-q}\sqrt{\frac{s(d_1 + d_2)\log\rbr{r(d_1 + d_2)}}{m\wedge n_1\wedge n_2}}\rbr{\|\Ub_1\|_F^{2q} + \|\Vb_1\|_F^{2q}},
\end{multline*}
\begin{multline*}
\|\nabla^2\mL(\Ub_2, \Vb_2) - \mE[\nabla^2\mL(\Ub_2, \Vb_2)]\|_2\\
\lesssim \beta^2r^{1-q}\sqrt{\frac{s(d_1 + d_2)\log\rbr{r(d_1 + d_2)}}{m\wedge n_1\wedge n_2}}\rbr{\|\Ub_2\|_F^{2q} + \|\Vb_2\|_F^{2q}},
\end{multline*}
and
\begin{multline*}
\|\mE[\nabla^2\mL(\Ub_1, \Vb_1)] - \mE[\nabla^2\mL(\Ub_2, \Vb_2)]\|_2\\
\lesssim  \beta^3r^{\frac{3(1-q)}{2}}\rbr{\|\Ub_2\|_F^{3q} + \|\Vb_2\|_F^{3q}}\rbr{\|\Ub_1 - \Ub_2\|_F^2 + \|\Vb_1 - \Vb_2\|_F^2}^{\frac{2-q}{4}}.
\end{multline*}
The proof follows by noting that $\|\Ub\|_F^2 + \|\Vb\|_F^2\lesssim \|\tUb\|_F^2 + \|\tVb\|_F^2$ for $(\Ub, \Vb)\in\B_R(\tUb, \tVb)$.

\subsection{Proof of Lemma~\ref{lem:B:5}}
By definition in Section \ref{sec:4.3}, we have the following decomposition
\begin{align*}
 \mE[\nabla^2\mL(\tUb, \tVb)] = \mE[\nabla^2\mL_1(\tUb, \tVb)] + \mE[\nabla^2\mL_2(\tUb, \tVb)] \stackrel{\eqref{d:2}}{=}\mE[\nabla^2\mL_1(\tUb, \tVb)].
\end{align*}
By \eqref{f:13},
\begin{align*}
\|\mE[\nabla^2\mL_1(\tUb, \tVb)]\|_2\lesssim \beta^2\rbr{\|\tVb\|_F^{2q_2}r^{1-q_2} + \|\tUb\|_F^{2q_1}r^{1-q_1}}\lesssim \beta^2r^{1-q}\rbr{\|\tVb\|_F^2+\|\tUb\|_F^2}^q.
\end{align*}
This completes the proof.

\subsection{Proof of Theorem \ref{thm:4}}

Define $\Upsilon^\star = \Cbc \beta^3 r^{\frac{3(1-q)}{2}}\rbr{\|\tUb\|_F^{3q}+\|\tVb\|_F^{3q}}$ for sufficiently large constant $\Cbc$. For any two points $(\Ub_1, \Vb_1)$, $(\Ub_2, \Vb_2) \in \B_R(\tUb, \tVb)$, if their distance satisfies
\begin{align*}
\|\Ub_1 - \Ub_2\|_F^2 + \|\Vb_1 - \Vb_2\|_F^2 \leq \rbr{\frac{\lambda_{\min}^\star}{20\Upsilon^\star}}^{\frac{4}{2-q}},
\end{align*}
and the sample sizes $m, n_1, n_2$ satisfy (which is implied by the condition in Theorem \ref{thm:3})
\begin{align*}
m\wedge n_1\wedge n_2 \geq \rbr{\frac{20\Upsilon^\star}{\lambda_{\min}^\star}}^2 s(d_1 + d_2)\log(r(d_1 + d_2)),
\end{align*}
by Lemma~\ref{lem:hessian_neighborhood} we know
\begin{align}\label{d:4}
\|\nabla^2\mL(\Ub_1, \Vb_1) - \nabla^2\mL(\Ub_2, \Vb_2)\|_2\leq \frac{\lambda_{\min}^\star}{10}.
\end{align}
Using this result and letting $\cbc\leq 1/(4\Cbc)^{\frac{4}{2-q}}$, then for any $(\Ub, \Vb)\in\B_R(\tUb, \tVb)$ we have
\begin{align*}
\|\Ub - \tUb\|_F^2 + \|\Vb - \tVb\|_F^2\leq  (\frac{\lambda_{\min}^\star}{4\Upsilon^\star})^\frac{4}{2-q}.
\end{align*}
Moreover, for any $(\Ub, \Vb)$ in this neighborhood, by Theorem \ref{thm:3}, we have
\begin{align*}
\|\nabla^2\mL(\Ub, \Vb) - \mE[\nabla^2\mL(\tUb, \tVb)]\|_2\leq \frac{\lambda_{\min}^\star}{20} + \frac{\lambda_{\min}^\star}{4}\leq \frac{\lambda_{\min}^\star}{2}.
\end{align*}
with high probability. Thus, by Weyl's theorem \citep{Weyl1912Das}, we have
\begin{multline*}
\lambda_{\min}(\nabla^2\mL(\Ub, \Vb))\geq \lambda_{\min}(\mE[\nabla^2\mL(\tUb, \tVb)]) - \|\nabla^2\mL(\Ub, \Vb) - \mE[\nabla^2\mL(\tUb, \tVb)]\|_2\\
\geq \lambda_{\min}^\star  - \lambda_{\min}^\star/2\geq \lambda_{\min}^\star/2,
\end{multline*}
and similarly $\lambda_{\max}(\nabla^2\mL(\Ub, \Vb))\leq 3\lambda_{\max}^\star/2$. Let us consider doing one-step GD at $(\Ub, \Vb)$. Let
\begin{align*}
\Ub' = \Ub - \eta \nabla_{\Ub}\mL(\Ub, \Vb)
\quad\text{and}\quad
\Vb' = \Vb - \eta \nabla_{\Vb}\mL(\Ub, \Vb).
\end{align*}
Suppose the continuous line from $(\Ub, \Vb)$ to $(\tUb, \tVb)$ is parameterized by $\xi \in[0, 1]$ with $\Ub_\xi = \tUb + \xi(\Ub - \tUb)$ and $\Vb_\xi = \tVb + \xi(\Vb - \tVb)$. Let $\Xi = \{\xi_1,\ldots, \xi_{|\Xi|}\}$ be a $(\frac{1}{5})^{\frac{4}{2-q}}$-net of interval $[0,1]$ with $|\Xi| = 5^{\frac{4}{2-q}}\leq 5^4$, and accordingly, we define $(\Ub_i, \Vb_i) = (\Ub_{\xi_i}, \Vb_{\xi_i})$ for $i\in[|\Xi|]$ and have set $\mS = \{(\Ub_1, \Vb_1), \ldots, (\Ub_{|\Xi|}, \Vb_{|\Xi|})\}$. Taking the union bound over $\mS$,
\begin{multline}\label{d:5}
P\rbr{\exists (\Ub, \Vb)\in \mS,  \lambda_{\min}(\nabla^2\mL(\Ub, \Vb))\leq \frac{\lambda_{\min}^\star}{2} \text{\ or\ } \lambda_{\max}(\nabla^2\mL(\Ub, \Vb))\geq \frac{3\lambda_{\max}^\star}{2} }\\ \lesssim \frac{1}{(d_1 + d_2)^s}.
\end{multline}
Furthermore, since $\Xi$ is a net of $[0, 1]$, for any $\xi\in[0, 1]$ there exists $\xi'\in[|\Xi|]$ such that
\begin{align*}
\|\Ub_\xi - \Ub_{\xi'}\|_F^2 + \| \Vb_\xi - \Vb_{\xi'}\|_F^2\leq  \rbr{\frac{\lambda_{\min}^\star}{20\Upsilon^\star}}^{\frac{4}{2-q}}.
\end{align*}	
Thus, by \eqref{d:4}, \eqref{d:5}, and Weyl's theorem, we obtain
\begin{align*}
\lambda_{\min}(\nabla^2\mL(\Ub_\xi, \Vb_\xi)) &\geq  \frac{\lambda_{\min}^\star}{2} - \frac{\lambda_{\min}^\star}{10} = \frac{2\lambda_{\min}^\star}{5},\\
\lambda_{\max}(\nabla^2\mL(\Ub_\xi, \Vb_\xi)) &\leq  \frac{3\lambda_{\max}^\star}{2} + \frac{\lambda_{\min}^\star}{10} \leq \frac{8\lambda_{\max}^\star}{5}.
\end{align*}	
With this,	
\begin{align*}
\|\Ub' - \tUb\|_F^2 &+ \|\Vb' - \tVb\|_F^2\\
&= \|\Ub - \tUb\|_F^2 + \|\Vb - \tVb\|_F^2 + \eta^2\|\nabla\mL(\Ub, \Vb)\|_F^2 \\
&\quad\quad - 2\eta \underbrace{\VEC\begin{pmatrix}
	\Ub - \tUb\\
	\Vb- \tVb
	\end{pmatrix}^T\bigg(\int_{0}^{1}\nabla^2\mL(\Ub_{\xi}, \Vb_{\xi} ) d\xi \bigg)\VEC\begin{pmatrix}
	\Ub - \tUb\\
	\Vb- \tVb
	\end{pmatrix}}_{\Hb(\Ub, \Vb)}\\
& \leq  \|\Ub - \tUb\|_F^2 + \|\Vb - \tVb\|_F^2 + \rbr{\frac{8\eta^2\lambda_{\max}^\star}{5} - 2\eta}\Hb(\Ub, \Vb).
\end{align*}
The last inequality is from Theorem \ref{thm:1} and the fact that $\|\nabla\mL(\tUb, \tVb) - \mE[\nabla\mL(\tUb, \tVb)]\|_F$ only contributes higher-order terms by concentration. Let $\eta = 1/\lambda_{\max}^\star$, then
\begin{align*}
\|\Ub' - \tUb\|_F^2 + \|\Vb' - \tVb\|_F^2&\leq  \|\Ub - \tUb\|_F^2 + \|\Vb - \tVb\|_F^2 - \frac{2}{5\lambda_{\max}^\star}\Hb(\Ub, \Vb)\\
&\leq  (1 - \frac{\lambda_{\min}^\star}{7\lambda_{\max}^\star})\big(\|\Ub - \tUb\|_F^2 + \|\Vb - \tVb\|_F^2\big),
\end{align*}
which completes the proof.

\acks{This work is partially supported by the William S. Fishman
  Faculty Research Fund at the University of Chicago Booth School of Business. This work was completed in part with resources provided by the University of Chicago Research Computing Center.}

\newpage

\appendix

\section{Complementary Lemmas}\label{appen:5}

In this section, we list intermediate results required for proving lemmas in Section \ref{sec:7}. Notations in each lemma are introduced in the proofs of the corresponding lemmas.

\begin{lemma}\label{lem:E:0}

Under conditions of Lemma \ref{lem:B:2}, there exists a constant $C_1>0$ not depending on $\tUb, \tVb$ such that
\begin{enumerate}
[label=(Case \arabic*),itemindent=20pt,topsep=-5pt]		\setlength\itemsep{0em}
\item if $\phi_1$, $\phi_2 \in \{\text{sigmoid}, \text{tanh}\}$, then
\begin{align*}
\Var(g(\bxb, \bzb)) \geq C_1\rbr{\|\bt_1\|_F^2 + \|\bt_2\|_F^2};
\end{align*}

\item if either $\phi_1$ or $\phi_2$ is ReLU, then
\begin{align*}
\Var(g(\bxb, \bzb))\geq C_1 \rbr{\|\bbt_1\|_F^2 + \|\bbt_2\|_F^2 + \|\diag(\bt_1) + \diag(\bt_2)\|_2^2}.
\end{align*}

\end{enumerate}

\end{lemma}

\begin{proof}
The notations in this proof are inherited from the proof of Lemma \ref{lem:B:2} in Section \ref{sec:7.2}. By the definition of $g(\cdot, \cdot)$ in \eqref{equ:def:g},
\begin{align*}
g(\bxb, \bzb) =& \sum_{p=1}^r\rbr{\phi_1'(\bxb_p)\phi_2(\bzb_p)\bxb^T\bt_{1p}  + \phi_1(\bxb_p)\phi_2'(\bzb_p)\bzb^T\bt_{2p}}.
\end{align*}
Therefore,
\begin{align}\label{c:5}
\mE\sbr{g(\bxb, \bzb)} = \tau_{1,1,1}'\tau_{2,1,0}\TR(\bt_1) + \tau_{1,1,0}\tau_{2,1,1}'\TR(\bt_2)
\end{align}
and
\begin{multline}\label{c:6}
\mE\sbr{g^2(\bxb, \bzb)} = \mE\big[\big(\sum_{p=1}^r\phi_1'(\bxb_p)\phi_2(\bzb_p)\bxb^T\bt_{1p} \big)^2\big]+ \mE\big[\big(\sum_{p=1}^r\phi_1(\bxb_p)\phi_2'(\bzb_p)\bzb^T\bt_{2p} \big)^2\big] \\
+ 2\sum_{1\leq p, q \leq r} \mE\sbr{\phi_1'(\bxb_p)\phi_2(\bzb_p)\phi_1(\bxb_q)\phi_2'(\bzb_q)\bxb^T\bt_{1p}\bzb^T\bt_{2q}}
\eqqcolon\I_4 + \I_5 + 2\I_6.
\end{multline}
Plugging the expressions of $\I_4$, $\I_5$, and $\I_6$ in Lemma \ref{lem:E:2} into \eqref{c:6}, combining with \eqref{c:5}, and using the fact that
\begin{align*}
\TR(\bbt_1^2) = \frac{1}{2}\|\bbt_1 + \bbt_1^T\|_F^2 - \|\bbt_1\|_F^2, \quad \quad 2\TR(\bbt_1\bbt_2) = \|\bbt_1 + \bbt_2^T\|_F^2 - \|\bbt_1\|_F^2 - \|\bbt_2\|_F^2,
\end{align*}
we obtain
\begin{align}\label{c:7}
&\Var(g(\bxb, \bzb)) = \mE\sbr{g^2(\bxb, \bzb)} - \rbr{\mE\sbr{g(\bxb, \bzb)}}^2 \nonumber\\
& =  \tau_{1,1,1}\tau_{2,1,1}\tau_{1,1,0}'\tau_{2,1,0}'\|\bbt_1 + \bbt_2^T\|_F^2 + \frac{1}{2}\tau_{2,1,0}^2(\tau_{1,1,1}')^2\|\bbt_1+ \bbt_1^T\|_F^2 + \frac{1}{2}\tau_{1,1,0}^2(\tau_{2,1,1}')^2\|\bbt_2+\bbt_2^T\|_F^2 \nonumber\\
&  +  \rbr{\tau_{2,2,0}\tau_{1,2,0}' - \tau_{2,1,0}^2(\tau_{1,1,0}')^2 - \tau_{1,1,1}\tau_{2,1,1}\tau_{1,1,0}'\tau_{2,1,0}' - \tau_{2,1,0}^2(\tau_{1,1,1}')^2}\|\bbt_1\|_F^2 \nonumber\\
& + \rbr{\tau_{1,2,0}\tau_{2,2,0}' - \tau_{1,1,0}^2(\tau_{2,1,0}')^2 - \tau_{1,1,1}\tau_{2,1,1}\tau_{1,1,0}'\tau_{2,1,0}' - \tau_{1,1,0}^2(\tau_{2,1,1}')^2}\|\bbt_2\|_F^2 \nonumber\\
& +  \|\tau_{2,1,0}\tau_{1,1,0}'\bbt_1\b1 + \tau_{2,1,0}\tau_{1,1,2}'\diag(\bt_1) + \tau_{1,1,1}\tau_{2,1,1}'\diag(\bt_2)\|_2^2 \nonumber\\
& +  \|\tau_{1,1,0}\tau_{2,1,0}'\bbt_2\b1 + \tau_{1,1,0}\tau_{2,1,2}'\diag(\bt_2) + \tau_{2,1,1}\tau_{1,1,1}'\diag(\bt_1)\|_2^2 \nonumber\\
& + \rbr{\tau_{2,2,0}\tau_{1,2,2}' - \tau_{2,1,0}^2(\tau_{1,1,1}')^2 - \tau_{2,1,0}^2(\tau_{1,1,2}')^2 - \tau_{2,1,1}^2(\tau_{1,1,1}')^2}\|\diag(\bt_1)\|_2^2 \nonumber\\
& + \rbr{\tau_{1,2,0}\tau_{2,2,2}' - \tau_{1,1,0}^2(\tau_{2,1,1}')^2 - \tau_{1,1,0}^2(\tau_{2,1,2}')^2 - \tau_{1,1,1}^2(\tau_{2,1,1}')^2}\|\diag(\bt_2)\|_2^2  + 2\big(\tau_1''\tau_2''  \nonumber\\
& - \tau_{1,1,0}\tau_{2,1,0}\tau_{1,1,1}'\tau_{2,1,1}' - \tau_{2,1,0}\tau_{1,1,1}\tau_{1,1,2}'\tau_{2,1,1}' - \tau_{1,1,0}\tau_{2,1,1}\tau_{2,1,2}'\tau_{1,1,1}' \big)\diag(\bt_1)^T\diag(\bt_2).
\end{align}
Based on the above expression, we further provide the lower bound for $\Var(g(\bxb, \bzb))$. We separate into two cases.

\noindent{\bf Case 1,} $\phi_1, \phi_2 \in\{\text{sigmoid}, \text{tanh}\}$. By symmetry of activation functions, $\tau_{i,1,1}' = 0$. Thus, plugging into \eqref{c:7},
\begin{align*}
&\Var(g(\bxb, \bzb)) \\
&=  \tau_{1,1,1}\tau_{2,1,1}\tau_{1,1,0}'\tau_{2,1,0}'\|\bbt_1 + \bbt_2^T\|_F^2 + \rbr{\tau_{2,2,0}\tau_{1,2,0}' - \tau_{2,1,0}^2(\tau_{1,1,0}')^2 - \tau_{1,1,1}\tau_{2,1,1}\tau_{1,1,0}'\tau_{2,1,0}'}\|\bbt_1\|_F^2\\
& \quad + \rbr{\tau_{1,2,0}\tau_{2,2,0}' - \tau_{1,1,0}^2(\tau_{2,1,0}')^2 - \tau_{1,1,1}\tau_{2,1,1}\tau_{1,1,0}'\tau_{2,1,0}'}\|\bbt_2\|_F^2+ 2\tau_1''\tau_2''\diag(\bt_1)^T\diag(\bt_2) \\
& \quad + \|\tau_{2,1,0}\tau_{1,1,0}'\bbt_1\b1 + \tau_{2,1,0}\tau_{1,1,2}'\diag(\bt_1)\|_2^2 + \|\tau_{1,1,0}\tau_{2,1,0}'\bbt_2\b1 + \tau_{1,1,0}\tau_{2,1,2}'\diag(\bt_2)\|_2^2\\
& \quad +  \rbr{\tau_{2,2,0}\tau_{1,2,2}'  - \tau_{2,1,0}^2(\tau_{1,1,2}')^2 }\|\diag(\bt_1)\|_2^2  + \rbr{\tau_{1,2,0}\tau_{2,2,2}' - \tau_{1,1,0}^2(\tau_{2,1,2}')^2}\|\diag(\bt_2)\|_2^2\\
&\geq  \tau_{1,1,1}\tau_{2,1,1}\tau_{1,1,0}'\tau_{2,1,0}'\|\bbt_1 + \bbt_2^T\|_F^2 + \rho_1\rbr{\|\bbt_1\|_F^2 + \|\bbt_2\|_F^2} + \tau_1''\tau_2''\|\diag(\bt_1) + \diag(\bt_2)\|_2^2\\
& \quad + \rbr{\tau_{2,2,0}\tau_{1,2,2}'  - \tau_{2,1,0}^2(\tau_{1,1,2}')^2  - \tau_1''\tau_2''}\|\diag(\bt_1)\|_2^2 \\
&\quad+ \rbr{\tau_{1,2,0}\tau_{2,2,2}' - \tau_{1,1,0}^2(\tau_{2,1,2}')^2 - \tau_1''\tau_2''}\|\diag(\bt_2)\|_2^2\\
&\geq  \tau_{1,1,1}\tau_{2,1,1}\tau_{1,1,0}'\tau_{2,1,0}'\|\bbt_1 + \bbt_2^T\|_F^2 + \rho_1\rbr{\|\bbt_1\|_F^2 + \|\bbt_2\|_F^2} + \tau_1''\tau_2''\|\diag(\bt_1) + \diag(\bt_2)\|_2^2\\
& \quad + \rho_2\rbr{\|\diag(\bt_1)\|_2^2 + \|\diag(\bt_2)\|_2^2},
\end{align*}
where, for $j = 1, 2$, $i = 1, 2$ and $\bari = 3-i$, $\rho_j = \rho_{j1} \wedge \rho_{j2}$ with
\begin{align*}
\rho_{1i} &=  \tau_{\bari,2,0}\tau_{i,2,0}' - \tau_{\bari,1,0}^2(\tau_{i,1,0}')^2 - \tau_{1,1,1}\tau_{2,1,1}\tau_{1,1,0}'\tau_{2,1,0}', \\
\rho_{2i} &=  \tau_{\bari,2,0}\tau_{i,2,2}'  - \tau_{\bari,1,0}^2(\tau_{i,1,2}')^2 - \tau_1''\tau_2''.
\end{align*}
Further, by Stein's identity \citep{Stein1972bound}, $\tau_{i,1,1} = \tau_{i,1,0}'$. We can also numerically check that $\tau_1'', \tau_2'', \rho_1, \rho_2> 0$. Therefore, the above display leads to
\begin{align*}
\Var(g(\bxb, \bzb)) \geq \min(\rho_1, \rho_2)\rbr{\|\bt_1\|_F^2 + \|\bt_2\|_F^2}.
\end{align*}

\noindent{\bf Case 2,} either $\phi_1$ or $\phi_2$ is ReLU. Without loss of generality, we assume $\phi_1$ is ReLU. Then, $\tau_{1,1,1} = \tau_{1,2,0} = \tau_{1,1,0}' = \tau_{1,2,0}' = \tau_{1,1,2}' = \tau_{1,2,2}'= \tau_1'' = 1/2$ and $\tau_{1,1,0} = \tau_{1,1,1}' = 1/\sqrt{2\pi}$. Thus, plugging into \eqref{c:7},
\begin{align*}
&\Var(g(\bxb, \bzb)) \\
& =  \frac{(\tau_{2,1,0}')^2}{4}\|\bbt_1 + \bbt_2^T\|_F^2 + \frac{\tau_{2,1,0}^2}{4\pi}\|\bbt_1+ \bbt_1^T\|_F^2 + \frac{(\tau_{2,1,1}')^2}{4\pi}\|\bbt_2+\bbt_2^T\|_F^2\\
& \quad + \frac{1}{2}\big(\tau_{2,2,0} - \frac{\pi + 2}{2\pi}\tau_{2,1,0}^2 - \frac{1}{2}(\tau_{2,1,0}')^2\big)\|\bbt_1\|_F^2 + \frac{1}{2}\big(\tau_{2,2,0}' - \frac{\pi + 2}{2\pi}(\tau_{2,1,0}')^2  - \frac{1}{\pi}(\tau_{2,1,1}')^2\big)\|\bbt_2\|_F^2\\
& \quad + \frac{1}{4}\|\tau_{2,1,0}\bbt_1\b1 + \tau_{2,1,0}\diag(\bt_1) + \tau_{2,1,1}'\diag(\bt_2)\|_2^2 \\
&\quad + \frac{1}{2\pi}\|\tau_{2,1,0}'\bbt_2\b1 + \tau_{2,1,2}'\diag(\bt_2) + \tau_{2,1,1}\diag(\bt_1)\|_2^2\\
& \quad + \frac{1}{2}\bigg\{\big(\tau_{2,2,0} - \frac{\pi +2}{2\pi}\tau_{2,1,0}^2 - \frac{1}{\pi}(\tau_{2,1,0}')^2\big)\|\diag(\bt_1)\|_2^2 \\&\quad + \big(\tau_{2,2,2}' - \frac{\pi + 2}{2\pi}(\tau_{2,1,1}')^2 - \frac{1}{\pi}(\tau_{2,1,2}')^2\big)\|\diag(\bt_2)\|_2^2\bigg\}  \\
& \quad + \big(\tau_2'' - \frac{\pi + 2}{2\pi}\tau_{2,1,0}\tau_{2,1,1}' - \frac{1}{\pi}\tau_{2,1,0}'\tau_{2,1,2}'\big)\diag(\bt_1)^T\diag(\bt_2)\\
&\geq  \frac{1}{2}\bigg\{\big(\tau_{2,2,0} - \frac{\pi + 2}{2\pi}\tau_{2,1,0}^2 - \frac{1}{2}(\tau_{2,1,0}')^2\big)\|\bbt_1\|_F^2 + \big(\tau_{2,2,0}' - \frac{\pi + 2}{2\pi}(\tau_{2,1,0}')^2  - \frac{1}{\pi}(\tau_{2,1,1}')^2\big)\|\bbt_2\|_F^2\\
& \quad + \big(\tau_{2,2,0} - \frac{\pi +2}{2\pi}\tau_{2,1,0}^2 - \frac{1}{\pi}(\tau_{2,1,0}')^2\big)\|\diag(\bt_1)\|_2^2 \\&\quad+ \big(\tau_{2,2,2}' - \frac{\pi + 2}{2\pi}(\tau_{2,1,1}')^2 - \frac{1}{\pi}(\tau_{2,1,2}')^2\big)\|\diag(\bt_2)\|_2^2\bigg\}\\
& \quad + \big(\tau_2'' - \frac{\pi + 2}{2\pi}\tau_{2,1,0}\tau_{2,1,1}' - \frac{1}{\pi}\tau_{2,1,0}'\tau_{2,1,2}'\big)\diag(\bt_1)^T\diag(\bt_2).
\end{align*}
Define
\begin{align*}
\rho_3 &= \big(\tau_{2,2,0} - \frac{\pi + 2}{2\pi}\tau_{2,1,0}^2 - \frac{1}{2}(\tau_{2,1,0}')^2\big) \wedge \big(\tau_{2,2,0}' - \frac{\pi + 2}{2\pi}(\tau_{2,1,0}')^2  - \frac{1}{\pi}(\tau_{2,1,1}')^2\big),\\
\rho_4 &=  \big(\tau_{2,2,0} - \frac{\pi +2}{2\pi}\tau_{2,1,0}^2 - \frac{1}{\pi}(\tau_{2,1,0}')^2\big) \wedge \big(\tau_{2,2,2}' - \frac{\pi + 2}{2\pi}(\tau_{2,1,1}')^2 - \frac{1}{\pi}(\tau_{2,1,2}')^2\big) \wedge \big(\tau_2'' \\
&\quad - \frac{\pi + 2}{2\pi}\tau_{2,1,0}\tau_{2,1,1}'  - \frac{1}{\pi}\tau_{2,1,0}'\tau_{2,1,2}'\big).
\end{align*}
Then, we can numerically check $\rho_3, \rho_4 > 0$ when $\phi_2 \in \{\text{sigmoid}, \text{tanh}, \text{ReLU}\}$ and hence
\begin{align*}
\Var(g(\xb, \zb))\geq \frac{\min(\rho_3, \rho_4)}{2}\rbr{\|\bbt_1\|_F^2 + \|\bbt_2\|_F^2 + \|\diag(\bt_1) + \diag(\bt_2)\|_2^2}.
\end{align*}
This completes the proof.
\end{proof}

\begin{lemma}\label{lem:E:1}

Under conditions of Lemma \ref{lem:B:2}, there exists a constant $C_2>0$ not depending on $\tUb$, $\tVb$ such that
\begin{align*}
\I_2\geq  \frac{C_2}{\barkap(\tUb)\barkap(\tVb)}\|\bs_1\|_F^2 \quad\text{and}\quad
\I_3\geq \frac{C_2}{\barkap(\tUb)\barkap(\tVb)}\|\bs_2\|_F^2.
\end{align*}

\end{lemma}

\begin{proof}
By symmetry, we only show the proof for $\I_2$. By the definition of $\I_2$ in \eqref{c:2} and noting that the inner variable has mean zero,
\begin{align}\label{f:1}
\I_2
& = \mE\big[\big(\sum_{p=1}^r\phi_1'(\tuT_p\xb)\phi_2(\tvT_p\zb)\xb^T\Qb_1^{\perp}\bs_{1p}\big)^2\big] \nonumber\\
&= \sum_{p=1}^r\mE\big[(\phi_1'(\tuT_p\xb))^2(\phi_2(\tvT_p\zb))^2\bs_{1p}^T(\Qb_1^{\perp})^T\xb\xb^T\Qb_1^{\perp}\bs_{1p}\big] \nonumber\\
& \quad + \sum_{1\leq p\neq q \leq r}\mE\big[\phi_1'(\tuT_p\xb)\phi_1'(\tuT_q\xb)\phi_2(\tvT_p\zb)\phi_2(\tvT_q\zb)\bs_{1q}^T(\Qb_1^{\perp})^T\xb\xb^T\Qb_1^{\perp}\bs_{1p}\big]\nonumber\\
&= \sum_{p=1}^r\mE\big[(\phi_1'(\tuT_p\xb))^2(\phi_2(\tvT_p\zb))^2\bs_{1p}^T\bs_{1p}\big] \nonumber\\
&\quad + \sum_{1\leq p\neq q \leq r}\mE\big[\phi_1'(\tuT_p\xb)\phi_1'(\tuT_q\xb)\phi_2(\tvT_p\zb)\phi_2(\tvT_q\zb)\bs_{1q}^T\bs_{1p}\big] \nonumber\\
&= \mE\big[\big\|\sum_{p=1}^r\phi_1'(\tuT_p\xb)\phi_2(\tvT_p\zb)\bs_{1p}\big\|^2\big] \nonumber\\
& \geq \frac{1}{\barkap(\tUb)\barkap(\tVb)}\mE\big[\big\|\sum_{p=1}^r\phi_1'(\bxb_p)\phi_2(\bzb_p)\bs_{1p}\big\|_2^2\big],
\end{align}
where the third equality is due to the independence among $\tuT_p\xb$, $\xb^T\Qb_1^\perp$ and $\zb$; the last inequality is due to Lemma \ref{lem:G:1} and Assumption \ref{ass:1}. Here, $\bxb,\bzb\stackrel{i.i.d}{\sim}\mN(0, I_r)$ and $\bxb_p, \bzb_p$ denote the $p$-th component of $
\bxb, \bzb$, respectively. Moreover,
\begin{align*}
\mE\big[\big\|\sum_{p=1}^r\phi_1'(\bxb_p)&\phi_2(\bzb_p)\bs_{1p}\big\|_2^2\big] \\
&=\sum_{p, q=1}^r\mE\sbr{\phi_1'(\bxb_p)\phi_2(\bzb_p)\phi_1'(\bxb_q)\phi_2(\bzb_q)\bs_{1p}^T\bs_{1q}}\\
&= \tau_{1,2,0}'\tau_{2,2,0}\sum_{p=1}^r\|\bs_{1p}\|^2 + (\tau'_{1,1,0})^2(\tau_{2,1,0})^2\sum_{1\leq p\neq q \leq r} \bs_{1p}^T\bs_{1q}\\
&= \tau_{1,2,0}'\tau_{2,2,0}\|\bs_1\|_F^2 + (\tau'_{1,1,0})^2(\tau_{2,1,0})^2\rbr{\|\bs_1 \b1\|_2^2 - \|\bs_1\|_F^2}\\
&\geq \rbr{\tau_{1,2,0}'\tau_{2,2,0} - (\tau'_{1,1,0})^2(\tau_{2,1,0})^2}\|\bs_1\|_F^2.
\end{align*}
Combining with \eqref{f:1},
\begin{align*}
\I_2\geq  \frac{\tau_{1,2,0}'\tau_{2,2,0} - (\tau'_{1,1,0})^2(\tau_{2,1,0})^2}{\barkap(\tUb)\barkap(\tVb)}\|\bs_1\|_F^2.
\end{align*}
Since $\tau_{1,2,0}' > (\tau'_{1,1,0})^2$ and $\tau_{2,2,0}>(\tau_{2,1,0})^2$, we complete the proof.
\end{proof}

\begin{lemma}\label{lem:E:2}

Under the setup of Lemma \ref{lem:E:0}, we have
\begin{align*}
\I_4 = & \rbr{\tau_{2,2,0}\tau_{1,2,0}' - \tau_{2,1,0}^2(\tau_{1,1,0}')^2}\|\bbt_1\|_F^2 + \tau_{2,1,0}^2(\tau_{1,1,1}')^2\TR(\bbt_1^2) + \tau_{2,1,0}^2(\tau_{1,1,0}')^2\|\bbt_1\b1\|_2^2 \nonumber\\
&\qquad + 2\tau_{2,1,0}^2\tau_{1,1,2}'\tau_{1,1,0}'\b1^T\bbt_1^T\diag(\bt_1) + \tau_{2,1,0}^2(\tau_{1,1,1}')^2(\b1^T\diag(\bt_1))^2 \\
&\qquad  +\big(\tau_{2,2,0}\tau_{1,2,2}'  - \tau_{2,1,0}^2(\tau_{1,1,1}')^2\big)\|\diag(\bt_1)\|_2^2,\\
\I_5 = &\rbr{\tau_{1,2,0}\tau_{2,2,0}' - \tau_{1,1,0}^2(\tau_{2,1,0}')^2}\|\bbt_2\|_F^2 + \tau_{1,1,0}^2(\tau_{2,1,1}')^2\TR(\bbt_2^2) + \tau_{1,1,0}^2(\tau_{2,1,0}')^2\|\bbt_2\b1\|_2^2 \nonumber\\
&\qquad + 2\tau_{1,1,0}^2\tau_{2,1,2}'\tau_{2,1,0}'\b1^T\bbt_2^T\diag(\bt_2) + \tau_{1,1,0}^2(\tau_{2,1,1}')^2(\b1^T\diag(\bt_2))^2  \nonumber\\
&\qquad + \big(\tau_{1,2,0}\tau_{2,2,2}'  - \tau_{1,1,0}^2(\tau_{2,1,1}')^2\big)\|\diag(\bt_2)\|_2^2,\\
  \intertext{and}
\I_6 = & \rbr{\tau_1''\tau_2'' - \tau_{1,1,0}\tau_{2,1,0}\tau_{1,1,1}'\tau_{2,1,1}'}\diag(\bt_1)^T\diag(\bt_2) + \tau_{1,1,1}\tau_{2,1,1}\tau_{1,1,0}'\tau_{2,1,0}'\TR(\bbt_1\bbt_2)  \nonumber\\
& \qquad + \tau_{1,1,0}\tau_{2,1,0}\tau_{1,1,1}'\tau_{2,1,1}'\b1^T\diag(\bt_1)\diag(\bt_2)^T\b1 + \tau_{1,1,0}\tau_{2,1,1}\tau_{1,1,1}'\tau_{2,1,0}'\b1^T\bbt_2^T\diag(\bt_1) \nonumber\\
& \qquad + \tau_{1,1,1}\tau_{2,1,0}\tau_{1,1,0}'\tau_{2,1,1}'\b1^T\bbt_1^T\diag(\bt_2).
\end{align*}

\end{lemma}

\begin{proof}
By symmetry, we only show the proof for $\I_4$ and $\I_5$ can be proved analogously. By the definition of $\I_4$ in \eqref{c:6},
\begin{align}\label{f:2}
\I_4 & =  \mE\big[\big(\sum_{p=1}^r\phi_1'(\bxb_p)\phi_2(\bzb_p)\bxb^T\bt_{1p} \big)^2\big] \nonumber\\
&=\sum_{p=1}^r\mE\sbr{(\phi_1'(\bxb_p))^2(\phi_2(\bzb_p))^2\bt_{1p}^T\bxb\bxb^T\bt_{1p}}+ \sum_{1\leq p\neq q \leq r}\mE\sbr{\phi_1'(\bxb_p)\phi_2(\bzb_p)\phi_1'(\bxb_q)\phi_2(\bzb_q)\bt_{1p}^T\bxb\bxb^T\bt_{1q}} \nonumber\\
&= 
\tau_{2,2,0}\sum_{p=1}^r\mE\sbr{(\phi_1'(\bxb_p))^2\bt_{1p}^T\bxb\bxb^T\bt_{1p}} + \tau_{2,1,0}^2\sum_{1\leq p\neq q \leq r}\mE\sbr{\phi_1'(\bxb_p)\phi_1'(\bxb_q)\bt_{1p}^T\bxb\bxb^T\bt_{1q}} \nonumber\\
&\eqqcolon  \tau_{2,2,0}\I_{41} + \tau_{2,1,0}^2\I_{42}.
\end{align}
By simple derivations, we let $\bt_{1pp} = [\bt_{1p}]_p$ be the $p$-th entry of $\bt_{1p}$, and have
\begin{align}\label{f:3}
\I_{41} &= (\tau_{1,2,2}' - \tau_{1,2,0}')\sum_{p=1}^r\bt_{1pp}^2 + \tau_{1,2,0}'\sum_{p=1}^r\|\bt_{1p}\|_2^2 =  (\tau_{1,2,2}' - \tau_{1,2,0}')\|\diag(\bt_1)\|_2^2 + \tau_{1,2,0}'\|\bt_1\|_F^2,
\end{align}
and
\begin{align*}
\I_{42} & =  \sum_{1\leq p\neq q \leq r}\Bigg((\tau_{1,1,1}')^2(\bt_{1pp}\bt_{1qq} + \bt_{1pq}\bt_{1qp}) \nonumber\\
&\qquad\qquad + \tau_{1,1,2}'\tau_{1,1,0}'(\bt_{1pp}\bt_{1qp} + \bt_{1pq}\bt_{1qq}) +  (\tau_{1,1,0}')^2\sum_{\substack{k=1\\ k\neq p,q}}^r\bt_{1pk}\bt_{1qk}\Bigg) \nonumber\\
& =  \sum_{1\leq p\neq q \leq r}\bigg((\tau_{1,1,1}')^2(\bt_{1pp}\bt_{1qq} + \bt_{1pq}\bt_{1qp})\nonumber\\
&\qquad\qquad  + (\tau_{1,1,2}'\tau_{1,1,0}' - (\tau_{1,1,0}')^2)(\bt_{1pp}\bt_{1qp} + \bt_{1pq}\bt_{1qq}) + (\tau_{1,1,0}')^2\bt_{1p}^T\bt_{1q}\bigg). 
\end{align*}
Moreover, for each component of $\I_{42}$ we have
\begin{align*}
& \sum_{1\leq p\neq q \leq r} \bt_{1pp}\bt_{1qq} + \bt_{1pq}\bt_{1qp} =  (\b1^T\diag(\bt_1))^2  + \TR(\bt_1^2) - 2\|\diag(\bt_1)\|_2^2, \\
& \sum_{1\leq p\neq q \leq r} \bt_{1pp}\bt_{1qp} + \bt_{1pq}\bt_{1qq} =  2\sum_{1\leq p\neq q \leq r} \bt_{1pp}\bt_{1qp} = 2\big(\b1^T\bt_1^T\diag(\bt_1) - \|\diag(\bt_1)\|_2^2\big), \\
& \sum_{1\leq p\neq q \leq r}\bt_{1p}^T\bt_{1q} = \|\bt_1\b1\|_2^2 - \|\bt_1\|_F^2.
\end{align*}
Plugging into the formula of $\I_{42}$,
\begin{align}\label{f:4}
\I_{42} =& (\tau_{1,1,1}')^2\rbr{(\b1^T\diag(\bt_1))^2  + \TR(\bt_1^2) - 2\|\diag(\bt_1)\|_2^2} + (\tau_{1,1,0}')^2\rbr{\|\bt_1\b1\|_2^2 - \|\bt_1\|_F^2}  \nonumber\\
& \quad + 2(\tau_{1,1,2}'\tau_{1,1,0}' - (\tau_{1,1,0}')^2)\rbr{\b1^T\bt_1^T\diag(\bt_1) - \|\diag(\bt_1)\|_2^2}.
\end{align}
Combining \eqref{f:2}, \eqref{f:3}, \eqref{f:4} together,
\begin{align*}
\I_4 =&  \tau_{2,2,0}\rbr{(\tau_{1,2,2}' - \tau_{1,2,0}')\|\diag(\bt_1)\|_2^2 + \tau_{1,2,0}'\|\bt_1\|_F^2} + \tau_{2,1,0}^2
\bigg((\tau_{1,1,0}')^2(\|\bt_1\b1\|_2^2 - \|\bt_1\|_F^2)\\
& + 2\rbr{\tau_{1,1,2}'\tau_{1,1,0}' - (\tau_{1,1,0}')^2}\rbr{\b1^T\bt_1^T\diag(\bt_1) - \|\diag(\bt_1)\|_2^2} + (\tau_{1,1,1}')^2\big((\b1^T\diag(\bt_1))^2 \\
& + \TR(\bt_1^2) - 2\|\diag(\bt_1)\|_2^2\big)\bigg) \\
= & \rbr{\tau_{2,2,0}\tau_{1,2,0}' - \tau_{2,1,0}^2(\tau_{1,1,0}')^2}\|\bt_1\|_F^2 + \tau_{2,1,0}^2(\tau_{1,1,0}')^2\|\bt_1\b1\|_2^2 + \tau_{2,1,0}^2(\tau_{1,1,1}')^2\rbr{\b1^T\diag(\bt_1)}^2\\
& + \tau_{2,1,0}^2(\tau_{1,1,1}')^2\TR(\bt_1^2) + 2\rbr{\tau_{2,1,0}^2\tau_{1,1,2}'\tau_{1,1,0}' - \tau_{2,1,0}^2(\tau_{1,1,0}')^2}\b1^T\bt_1^T\diag(\bt_1) + \big(\tau_{2,2,0}\tau_{1,2,2}'\\
&  - \tau_{2,2,0}\tau_{1,2,0}' - 2\tau_{2,1,0}^2\tau_{1,1,2}'\tau_{1,1,0}' + 2\tau_{2,1,0}^2(\tau_{1,1,0}')^2 - 2\tau_{2,1,0}^2(\tau_{1,1,1}')^2\big)\|\diag(\bt_1)\|_2^2.
\end{align*}
Recall from Section \ref{sec:7.2} that $\bbt_i\in\mR^{r\times r}$, $i = 1, 2$, denotes the matrix that replaces the diagonal entries of $\bt_i$ by $0$. Therefore, the above equation can be further simplified as
\begin{align*}
\I_4 = & \rbr{\tau_{2,2,0}\tau_{1,2,0}' - \tau_{2,1,0}^2(\tau_{1,1,0}')^2}\|\bbt_1\|_F^2 + \tau_{2,1,0}^2(\tau_{1,1,1}')^2\TR(\bbt_1^2) + \tau_{2,1,0}^2(\tau_{1,1,0}')^2\|\bbt_1\b1\|_2^2 \nonumber\\
&\qquad + 2\tau_{2,1,0}^2\tau_{1,1,2}'\tau_{1,1,0}'\b1^T\bbt_1^T\diag(\bt_1) + \tau_{2,1,0}^2(\tau_{1,1,1}')^2(\b1^T\diag(\bt_1))^2 \\
&\qquad + \big(\tau_{2,2,0}\tau_{1,2,2}'  - \tau_{2,1,0}^2(\tau_{1,1,1}')^2\big)\|\diag(\bt_1)\|_2^2.
\end{align*}
This completes the proof for $\I_4$. $\I_5$ can be obtained analogously by changing the role of $\phi_1$ and $\phi_2$. By the definition of $\I_6$ in \eqref{c:6},
\begin{align*}
\I_6 &=  \sum_{p=1}^r\mE\sbr{\phi_1'(\bxb_p)\phi_1(\bxb_p)\bxb^T\bt_{1p}}\mE\sbr{\phi_2'(\bzb_p)\phi_2(\bzb_p)\bzb^T\bt_{2p}}  \nonumber\\
& \quad + \sum_{1\leq p\neq q \leq r}\mE\sbr{\phi_1'(\bxb_p)\phi_1(\bxb_q)\bxb^T\bt_{1p}}\mE\sbr{\phi_2'(\bzb_q)\phi_2(\bzb_p)\bzb^T\bt_{2q}} \nonumber\\
& =  \tau_1''\tau_2''\sum_{p=1}^r\bt_{1pp}\bt_{2pp} \\
&\quad + \sum_{1\leq p\neq q \leq r}\rbr{\tau_{1,1,0}\tau_{1,1,1}'\bt_{1pp} + \tau_{1,1,0}'\tau_{1,1,1}\bt_{1pq}}\rbr{\tau_{2,1,1}\tau_{2,1,0}'\bt_{2qp} + \tau_{2,1,1}'\tau_{2,1,0}\bt_{2qq}} \nonumber\\
&=  \tau_1''\tau_2''\diag(\bt_1)^T\diag(\bt_2) + \tau_{1,1,0}\tau_{1,1,1}'\tau_{2,1,1}\tau_{2,1,0}'\big(\b1^T\bt_2^T\diag(\bt_1) - \diag(\bt_1)^T\diag(\bt_2)\big)  \nonumber\\
&\quad + \tau_{1,1,0}'\tau_{1,1,1}\tau_{2,1,1}'\tau_{2,1,0}\big(\b1^T\bt_1^T\diag(\bt_2) - \diag(\bt_1)^T\diag(\bt_2)\big) \nonumber\\
&\quad  + \tau_{1,1,0}\tau_{1,1,1}'\tau_{2,1,1}'\tau_{2,1,0}\big(\b1^T\diag(\bt_1)\diag(\bt_2)^T\b1 - \diag(\bt_1)^T\diag(\bt_2)\big) \nonumber\\
& \quad + \tau_{1,1,0}'\tau_{1,1,1}\tau_{2,1,1}\tau_{2,1,0}'\big(\TR(\bt_1\bt_2) - \diag(\bt_1)^T\diag(\bt_2)\big) \nonumber\\
&=  \rbr{\tau_1''\tau_2'' - \tau_{1,1,0}\tau_{2,1,0}\tau_{1,1,1}'\tau_{2,1,1}'}\diag(\bt_1)^T\diag(\bt_2) + \tau_{1,1,1}\tau_{2,1,1}\tau_{1,1,0}'\tau_{2,1,0}'\TR(\bbt_1\bbt_2)  \nonumber\\
& \quad + \tau_{1,1,0}\tau_{2,1,0}\tau_{1,1,1}'\tau_{2,1,1}'\b1^T\diag(\bt_1)\diag(\bt_2)^T\b1 + \tau_{1,1,0}\tau_{2,1,1}\tau_{1,1,1}'\tau_{2,1,0}'\b1^T\bbt_2^T\diag(\bt_1) \nonumber\\
& \quad + \tau_{1,1,1}\tau_{2,1,0}\tau_{1,1,0}'\tau_{2,1,1}'\b1^T\bbt_1^T\diag(\bt_2).
\end{align*}
This completes the proof. 
\end{proof}

\begin{lemma}\label{lem:E:5}

Under conditions of Lemma \ref{lem:B:3}, we let $q_i = 1$ if $\phi_i$ is ReLU and $q_i = 0$ if $\phi_i \in\{\text{sigmoid}, \text{tanh}\}$. Then for any $s\geq 1$,
\begin{align*}
P\rbr{\J_1 \gtrsim \beta^2\sqrt{\frac{s(d_1 + d_2)\log\rbr{r(d_1 + d_2)}}{m}} \rbr{\|\Vb\|_F^{2q_2}r^{1-q_2} + \|\Ub\|_F^{2q_1}r^{1-q_1}}} &\lesssim \frac{1}{(d_1 + d_2)^s}, \\
P\rbr{\J_2 \gtrsim \beta^2\sqrt{\frac{s(d_1 + d_2)\log\rbr{r(d_1 + d_2)}}{n_1\wedge n_2}} \rbr{\|\Vb\|_F^{2q_2}r^{1-q_2} + \|\Ub\|_F^{2q_1}r^{1-q_1}}} & \lesssim \frac{1}{(d_1 + d_2)^s}.
\end{align*}
\end{lemma}

\begin{proof}
\textit{Proof of $\J_1$.} 
For any two samples $(y, \xb, \zb)\in\D$ and $(y', \xb', \zb')\in \D'$, let us define
\begin{align*}
\Hb_1\rbr{(\xb, \zb), (\xb', \zb')} = \frac{(y - y')^2\exp\big((y - y')(\bTheta - \bTheta')\big)}{\big(1 + \exp\big((y - y')(\bTheta - \bTheta')\big)\big)^2}\cdot \begin{pmatrix}
\bd - \bd'\\
\bp - \bp'
\end{pmatrix}\begin{pmatrix}
\bd - \bd'\\
\bp - \bp'
\end{pmatrix}^T,
\end{align*}
where $\bTheta = \LD \phi_1(\Ub^T\xb), \phi_2(\Vb^T\zb)\RD$. To ease notations, we suppress the evaluation sample of $\Hb_1$. We apply Lemma \ref{lem:G:4} to bound $\J_1$. We first check all conditions of Lemma \ref{lem:G:4}. By Assumption \ref{ass:2} and symmetry of $(\bd, \bp)$ and $(\bd', \bp')$,
\begin{align*}
\|\Hb_1\|_2 &\leq 4\beta^2\bigg\|\begin{pmatrix}
\bd - \bd'\\
\bp - \bp'
\end{pmatrix}\begin{pmatrix}
\bd - \bd'\\
\bp - \bp'
\end{pmatrix}^T \bigg\|_2\\
& \leq 16\beta^2 \rbr{\bd^T\bd + \bp^T\bp}\\
&=  16\beta^2 \rbr{\sum_{p=1}^r\rbr{\phi_1'(\bu_p^T\xb)}^2\rbr{\phi_2(\bv_p^T\zb)}^2\xb^T\xb + \rbr{\phi_1(\bu_p^T\xb)}^2\rbr{\phi_2'(\bv_p^T\zb)}^2\zb^T\zb}\\
&\leq  16\beta^2 \rbr{\sum_{p=1}^r\rbr{\phi_2(\bv_p^T\zb)}^2\xb^T\xb + \rbr{\phi_1(\bu_p^T\xb)}^2\zb^T\zb}.
\end{align*}
For activation functions in $\{\text{sigmoid}, \text{tanh}, \text{ReLU}\}$,
the last inequality is due to the fact that $|\phi_i'|\leq 1$. Note that
\begin{align}\label{b:1}
|\phi_i(x)|\leq |x|^{q_i}, \quad \forall i = 1,2,
\end{align}
thus we further obtain
\begin{align}\label{f:5}
\|\Hb_1\|_2 &\leq 16\beta^2 \rbr{\sum_{p=1}^r(\zb^T\bv_p\bv_p^T\zb)^{q_2}\cdot\xb^T\xb + (\xb^T\bu_p\bu_p^T\xb)^{q_1}\cdot\zb^T\zb} \nonumber\\
&= 16\beta^2 \rbr{\rbr{\zb^T\Vb\Vb^T\zb}^{q_2}r^{1-q_2}\cdot\xb^T\xb + \rbr{\xb^T\Ub\Ub^T\xb}^{q_1}r^{1-q_1}\cdot\zb^T\zb}.
\end{align}
By Lemma \ref{lem:G:2}, $\forall s\geq 1$
\begin{multline*}
P\bigg(\max_{(\xb, \zb)\in \D\cup \D'} (\zb^T\Vb\Vb^T\zb)^{q_2}r^{1-q_2}\cdot\xb^T\xb\\ \gtrsim  (\|\Vb\|_F  + \sqrt{s\log n_2}\|\Vb\|_2)^{2q_2}r^{1-q_2}\cdot
(\sqrt{d_1} + \sqrt{s\log n_1})^2\bigg)
\lesssim \frac{1}{(n_1\wedge n_2)^s}.
\end{multline*}
We bound the second term in \eqref{f:5} similarly and have
\begin{align}\label{f:6}
P\bigg(\max_{\D\cup \D'}&\|\Hb_1\|_2 \gtrsim \beta^2\big((\|\Vb\|_F + \sqrt{s\log n_2}\|\Vb\|_2)^{2q_2}r^{1-q_2}\cdot (\sqrt{d_1} + \sqrt{s\log n_1})^2 \nonumber\\
& \underbrace{\quad +(\|\Ub\|_F + \sqrt{s\log n_1}\|\Ub\|_2)^{2q_1}r^{1-q_1}\cdot (\sqrt{d_2} + \sqrt{s\log n_2})^2 \big)}_{\nu_1(\J_1)}\bigg)\lesssim  \frac{1}{(n_1\wedge n_2)^s}.
\end{align}
We next verify the second condition in Lemma \ref{lem:G:4}. By the symmetry of $\Hb_1$, we only need bound the following quantity
\begin{align}\label{f:7}
&\frac{1}{n_1^2n_2^2}\sum_{(\xb, \zb)\in \D}\sum_{(\xb', \zb')\in \D'} \Hb_1\big((\xb, \zb), (\xb', \zb')\big)\Hb_1\big((\xb, \zb), (\xb', \zb')\big)^T \nonumber\\
&=  \frac{1}{n_1^2n_2^2}\sum_{(\xb, \zb)\in \D}\sum_{(\xb', \zb')\in \D'} \frac{(y - y')^4\exp(2(y-y')(\bTheta - \bThe'))}{(1+\exp((y-y')(\bThe - \bThe')))^4}\bigg\|\begin{pmatrix}
\bd - \bd'\\
\bp - \bp'
\end{pmatrix}\bigg\|_2^2\begin{pmatrix}
\bd - \bd'\\
\bp - \bp'
\end{pmatrix}\begin{pmatrix}
\bd - \bd'\\
\bp - \bp'
\end{pmatrix}^T \nonumber\\
&\preceq \frac{64\beta^4}{n_1^2n_2^2}\sum_{(\xb, \zb)\in \D}\sum_{(\xb', \zb')\in \D'}\big((\bd^T\bd + \bp^T\bp) + (\bd'^T\bd' + \bp'^T\bp')\big)\cdot\bigg(\begin{pmatrix}
\bd\\
\bp
\end{pmatrix}\begin{pmatrix}
\bd\\
\bp
\end{pmatrix}^T + \begin{pmatrix}
\bd'\\
\bp'
\end{pmatrix}\begin{pmatrix}
\bd'\\
\bp'
\end{pmatrix}^T\bigg) \nonumber\\
&=  \frac{128\beta^4}{n_1n_2}\sum_{(\xb, \zb)\in \D}\big(\bd^T\bd + \bp^T\bp\big)\cdot\begin{pmatrix}
\bd\\
\bp
\end{pmatrix}\begin{pmatrix}
\bd\\
\bp
\end{pmatrix}^T\nonumber \\&\quad+ \frac{128\beta^4}{n_1n_2}\sum_{(\xb, \zb)\in \D}\big(\bd^T\bd + \bp^T\bp\big)\cdot\frac{1}{n_1n_2}\sum_{(\xb', \zb')\in \D'}\begin{pmatrix}
\bd'\\
\bp'
\end{pmatrix}\begin{pmatrix}
\bd'\\
\bp'
\end{pmatrix}^T \nonumber\\
&\eqqcolon  128\beta^4 \J_{11} + 128\beta^4\J_{12}.
\end{align}
We only bound $\J_{11}$ as an example. $\J_{12}$ can be bounded in the same sketch.

\noindent{\bf Step 1.} Bound $\|\mE[\J_{11}]\|_2$. For any vectors $\ba = (\ba_1; \ldots; \ba_r)$ and $\bb = (\bb_1; \ldots; \bb_r)$ such that $\ba_p \in\mR^{d_1}$, $\bb_p\in\mR^{d_2}$ for $p\in[r]$ and $\|\ba\|_2^2 + \|\bb\|_2^2 = 1$,
\begin{align*}
\bigg|\begin{pmatrix}
\ba & \bb
\end{pmatrix}\mE&[\J_{11}]\begin{pmatrix}
\ba\\
\bb
\end{pmatrix}\bigg| \\ 
&=  \mE\bigg[\bigg(\sum_{p=1}^r(\phi_1'(\bu_p^T\xb))^2(\phi_2(\bv_p^T\zb))^2\xb^T\xb + (\phi_1(\bu_p^T\xb))^2(\phi_2'(\bv_p^T\zb))^2\zb^T\zb\bigg) \nonumber\\
&\qquad \cdot\bigg(\sum_{i=1}^r\phi_1'(\bu_i^T\xb)\phi_2(\bv_i^T\zb)\ba_i^T\xb + \sum_{j=1}^r\phi_1(\bu_j^T\xb)\phi_2'(\bv_j^T\zb)\bb_j^T\zb\bigg)^2\bigg]\nonumber\\
&\leq  \mE\bigg[\big((\zb^T\Vb\Vb^T\zb)^{q_2}r^{1-q_2}\cdot\xb^T\xb + (\xb^T\Ub\Ub^T\xb)^{q_1}r^{1-q_1}\cdot\zb^T\zb\big) \nonumber\\
&\qquad \cdot\bigg(\sum_{i,j=1}^r|\zb^T\bv_i\bv_j^T\zb|^{q_2}|\xb^T\ba_i\ba_j^T\xb| + 2\sum_{i,j=1}^r|\xb^T\ba_i|\cdot|\bu_j^T\xb|^{q_1}\cdot|\zb^T\bb_j|\cdot|\bv_i^T\zb|^{q_2} \nonumber\\
&\qquad + \sum_{i,j=1}^r|\xb^T\bu_i\bu_j^T\xb|^{q_1}|\zb^T\bb_i\bb_j^T\zb|\bigg)\bigg].
\end{align*}
By Lemma \ref{lem:G:3} and we have
\begin{multline*}
\bigg|\begin{pmatrix}
\ba & \bb
\end{pmatrix}\mE[\J_{11}]\begin{pmatrix}
\ba\\
\bb
\end{pmatrix}\bigg| \\
\lesssim  \big(d_1r^{1-q_2}\|\Vb\|_F^{2q_2} + d_2r^{1-q_1}\|\Ub\|_F^{2q_1}\big)\rbr{\sum_{i=1}^{r}\|\ba_i\|_2\|\bv_i\|_2^{q_2}+ \|\bb_i\|_2\|\bu_i\|_2^{q_1}}^2.
\end{multline*}
Maximizing over set $\{(\ba, \bb): \|\ba\|_2^2 + \|\bb\|_2^2 = 1\}$ on both sides and we get
\begin{align}\label{f:8}
\|\mE[\J_{11}]\|_2\lesssim \rbr{d_1\|\Vb\|_F^{2q_2}r^{1-q_2} + d_2\|\Ub\|_F^{2q_1}	r^{1-q_1}}\rbr{\|\Vb\|_F^{2q_2}r^{1-q_2} + \|\Ub\|_F^{2q_1}r^{1-q_1}}.
\end{align}

\noindent{\bf Step 2.} Bound $\|\J_{11} - \mE[\J_{11}]\|_2$. We apply Lemma \ref{lem:G:6}. Let us first define the random matrix
\begin{align*}
\Jb_{11}(\xb, \zb) \coloneqq\big(\bd^T\bd + \bp^T\bp\big)\cdot\begin{pmatrix}
\bd\\
\bp
\end{pmatrix}\begin{pmatrix}
\bd\\
\bp
\end{pmatrix}^T.
\end{align*}
For the condition (a) in Lemma \ref{lem:G:6}, we note that
\begin{align*}
\|\Jb_{11}(\xb, \zb)\|_2 & =  (\bd^T\bd + \bp^T\bp)^2 \leq \big(\sum_{p=1}^r(\zb^T\bv_p\bv_p^T\zb)^{q_2}\xb^T\xb + \sum_{p=1}^r(\xb^T\bu_p\bu_p^T\xb)^{q_1}\zb^T\zb\big)^2\\
& =  \rbr{r^{1-q_2}(\zb^T\Vb\Vb^T\zb)^{q_2}\xb^T\xb + r^{1-q_1}(\xb^T\Ub\Ub^T\xb)^{q_1}\zb^T\zb}^2.
\end{align*}
By Lemma \ref{lem:G:2}, for any constants $K_1^{(1 ,1)}\wedge K_2^{(1, 1)}\wedge K_3^{(1, 1)}\geq 1$ (in what follows we may keep using such notation, where the first superscript indexes the function $\{\mL_i\}_{i=1,2}$ we are dealing with; the second superscript indexes the times we have used for this notation),
\begin{multline}\label{f:9}
P\rbr{\|\Jb_{11}(\xb, \zb)\|_2\gtrsim (K_3^{(1,1)})^2\rbr{d_1(K_2^{(1,1)})^{q_2}\|\Vb\|_F^{2q_2}r^{1-q_2} + d_2(K_1^{(1,1)})^{q_1}\|\Ub\|_F^{2q_1}r^{1-q_1}}^2}\\ \leq 2\exp\rbr{-(d_1\wedge d_2)K_3^{(1,1)}} + q_2\exp\rbr{-\frac{\|\Vb\|_F^2K_2^{(1,1)}}{\|\Vb\|_2^2}} + q_1\exp\rbr{-\frac{\|\Ub\|_F^2K_1^{(1,1)}}{\|\Ub\|_2^2}}.
\end{multline}
For the condition (b) in Lemma \ref{lem:G:6}, we apply the inequalities in Lemma \ref{lem:G:3} and have
\begin{align}\label{f:10}
\|\mE&[\Jb_{11}(\xb, \zb)\Jb_{11}(\xb, \zb)^T]\|_2 =  \max_{\|\ba\|_F^2+\|\bb\|_F^2=1} \mE\big[\big(\bd^T\bd + \bp^T\bp\big)^3\big(\ba^T\bd + \bb^T\bp\big)^2\big] \nonumber\\
\leq &  \max_{\|\ba\|_F^2+\|\bb\|_F^2=1} \mE\big[\big(r^{1-q_2}(\zb^T\Vb\Vb^T\zb)^{q_2}\xb^T\xb + r^{1-q_1}(\xb^T\Ub\Ub^T\xb)^{q_1}\zb^T\zb\big)^3 \nonumber\\
&\qquad \cdot\big(\sum_{i=1}^r|\bv_i^T\zb|^{q_2}|\ba_i^T\xb| + \sum_{j=1}^r|\bu_j^T\xb|^{q_1}|\bb_j^T\zb|\big)^2\big] \nonumber\\
\lesssim & \rbr{d_1\|\Vb\|_F^{2q_2}r^{1-q_2} + d_2\|\Ub\|_F^{2q_1}r^{1-q_1}}^3\max_{\|\ba\|_F^2+\|\bb\|_F^2=1}\big(\sum_{i=1}^r\|\ba_i\|_2\|\bv_i\|_2^{q_2} + \|\bb_i\|_2\|\bu_i\|_2^{q_1}\big)^2 \nonumber\\
\lesssim & \rbr{d_1\|\Vb\|_F^{2q_2}r^{1-q_2} + d_2\|\Ub\|_F^{2q_1}r^{1-q_1}}^3\rbr{\|\Vb\|_F^{2q_2}r^{1-q_2} + \|\Ub\|_F^{2q_1}r^{1-q_1}}.
\end{align}
For the condition (c) in Lemma \ref{lem:G:6}, we consider the following quantity for any unit vector $(\ba; \bb)$:
\begin{align}\label{f:11}
\mE\big[\big(\bd^T\bd + &\bp^T\bp\big)^2\big(\ba^T\bd + \bb^T\bp\big)^4\big] \nonumber\\
&\leq \mE\big[\big(r^{1-q_2}(\zb^T\Vb\Vb^T\zb)^{q_2}\xb^T\xb\nonumber\\
&\quad+r^{1-q_1}(\xb^T\Ub\Ub^T\xb)^{q_1}\zb^T\zb\big)^2 \big(\sum_{i=1}^r|\bv_i^T\zb|^{q_2}|\ba_i^T\xb| + \sum_{j=1}^r|\bu_j^T\xb|^{q_1}|\bb_j^T\zb|\big)^4\big] \nonumber\\
&\lesssim  \rbr{d_1\|\Vb\|_F^{2q_2}r^{1-q_2} + d_2\|\Ub\|_F^{2q_1}r^{1-q_1}}^2 \rbr{\|\Vb\|_F^{2q_2}r^{1-q_2} + \|\Ub\|_F^{2q_1}r^{1-q_1}}^2.
\end{align}
Combining \eqref{f:8}, \eqref{f:9}, \eqref{f:10}, \eqref{f:11} together and defining
\begin{align}\label{f:12}
\Upsilon_1 \coloneqq d_1\|\Vb\|_F^{2q_2}r^{1-q_2} + d_2\|\Ub\|_F^{2q_1}r^{1-q_1}, \quad \Upsilon_2 \coloneqq \|\Vb\|_F^{2q_2}r^{1-q_2} + \|\Ub\|_F^{2q_1}r^{1-q_1},
\end{align}
we know conditions in Lemma \ref{lem:G:6} hold for $\J_{11}$ with parameters (up to constants)
\begin{align*}
\mu_1(\J_{11}) &\coloneqq  (K_3^{(1,1)})^2\rbr{d_1(K_2^{(1,1)})^{q_2}\|\Vb\|_F^{2q_2}r^{1-q_2} + d_2(K_1^{(1,1)})^{q_1}\|\Ub\|_F^{2q_1}r^{1-q_1}}^2,\\
\nu_1(\J_{11}) &\coloneqq \exp\rbr{-(d_1\wedge d_2)K_3^{(1,1)}} + q_2\exp\rbr{-\frac{\|\Vb\|_F^2K_2^{(1,1)}}{\|\Vb\|_2^2}} + q_1\exp\rbr{-\frac{\|\Ub\|_F^2K_1^{(1,1)}}{\|\Ub\|_2^2}},\\
\nu_2(\J_{11}) &\coloneqq \Upsilon_1^3\Upsilon_2, \quad\quad \nu_3(\J_{11}) \coloneqq  \Upsilon_1\Upsilon_2, \quad\quad \|\mE[\J_{11}]\| \lesssim \Upsilon_1\Upsilon_2.
\end{align*}
Thus, $\forall t>0$
\begin{align*}
P\big(\big\|\J_{11} - &\mE[\J_{11}]\big\|_2 > t + \Upsilon_1\Upsilon_2\sqrt{\nu_1(\J_{11})}\big) \\
&\leq  n_1n_2\nu_1(\J_{11}) \\
&\qquad+ 2r(d_1 + d_2)\exp\rbr{-\frac{(n_1\wedge n_2)t^2}{\rbr{2\Upsilon_1^3\Upsilon_2 + 4\Upsilon_1^2\Upsilon_2^2 + 4 \Upsilon_1^2\Upsilon_3^2\nu_1(\J_{11})} + 4\mu_1(\J_{11})t } }\\
&\leq  n_1n_2\nu_1(\J_{11}) + 2r(d_1 + d_2)\exp\rbr{-\frac{(n_1\wedge n_2)t^2}{10\Upsilon_1^3\Upsilon_2 + 4\mu_1(\J_{11})t } }.
\end{align*}
In the above inequality, for any constant $s\geq 1$ we let
\begin{align*}
K_1^{(1,1)} = K_2^{(1,1)} = \log(n_1n_2) + s\log(d_1 + d_2), \quad\quad K_3^{(1,1)} = 1.
\end{align*}
By simple calculation, we can let
\begin{align*}
\epsilon_1 \asymp \sqrt{\frac{s(d_1 + d_2)\log\rbr{r(d_1 + d_2)}}{n_1\wedge n_2}} \vee \frac{s(d_1 + d_2)\cbr{\log\rbr{r(d_1 + d_2)}}^{1 + 2(q_1\vee q_2)}}{n_1\wedge n_2}
\end{align*}
and further have
\begin{align*}
P\rbr{\big\|\J_{11} - \mE[\J_{11}]\big\|_2 > \epsilon_1\Upsilon_1\Upsilon_2} \lesssim \frac{1}{(d_1 + d_2)^s}.
\end{align*}
Under the conditions of Lemma \ref{lem:E:5}, we combine the above inequality with \eqref{f:8} and have $P(\|\J_{11}\|_2\gtrsim \Upsilon_1\Upsilon_2)\lesssim 1/(d_1 + d_2)^s$, $\forall s\geq 1$. Dealing with $\J_{12}$ in \eqref{f:7} similarly, one can show \eqref{f:8} and the above result hold for $\J_{12}$ as well. So $P(\|\J_{12}\|_2\gtrsim \Upsilon_1\Upsilon_2)\lesssim 1/(d_1 + d_2)^s$. Plugging back into \eqref{f:7}, we can define $\nu_2(\J_1) = \beta^4\Upsilon_1\Upsilon_2$ and then conditions of Lemma \ref{lem:G:4} hold for $\J_1$ with parameters $\nu_1(\J_1)$ (defined in \eqref{f:6}) and $\nu_2(\J_1)$. Therefore, we have $\forall t >0$

\begin{align*}
P\rbr{\J_1 >t} \lesssim 2r(d_1 + d_2)\exp\rbr{-\frac{mt^2}{4\nu_2(\J_1) + 4\nu_1(\J_1)t} }.
\end{align*}
For any $s\geq 1$, we let
\begin{align*}
\epsilon_2\asymp \sqrt{\frac{s(d_1 + d_2)\log\rbr{r(d_1 + d_2)}}{m}} \vee \frac{s(d_1 + d_2)\cbr{\log\rbr{r(d_1 + d_2)}}^{1 +q}}{m}
\end{align*}
and have
\begin{align*}
P\rbr{\J_1 > \beta^2\epsilon_2\Upsilon_2} \lesssim \frac{1}{(d_1 + d_2)^s}.
\end{align*}
The result follows by the definition of $\Upsilon_2$ in \eqref{f:12} and noting that the first term in $\epsilon_2$ is the dominant term.

\noindent\textit{Proof of $\J_2$.} We apply Lemma \ref{lem:G:5} to bound $\J_2$. We check all conditions of Lemma \ref{lem:G:5}. Some of steps are similar as above.
By definition of $\Hb_1$,
\begin{align*}
\nabla^2\bar\mL_1(\Ub, \Vb) = \frac{1}{n_1^2n_2^2}\sum_{(\xb, \zb)\in\D}\sum_{(\xb', \zb')\in \D'}\Hb_1\big((\xb, \zb), (\xb',
\zb')\big).
\end{align*}
We first bound $\|\mE[\Hb_1]\|_2$. We have
\begin{align}\label{f:13}
\|\mE[\Hb_1]\|_2 \lesssim \beta^2\max_{\|\ba\|_F^2+\|\bb\|_F^2=1} \mE\sbr{(\ba^T\bd + \bb^T\bp)^2} \lesssim \beta^2\Upsilon_2,
\end{align}
where the last inequality is derived similarly to \eqref{f:10}. For the condition (a) in Lemma \ref{lem:G:5}, we apply \eqref{f:5} and Lemma \ref{lem:G:2} (similar to \eqref{f:9}),
\begin{multline*}
P\rbr{\|\Hb_{1}\|_2\gtrsim \beta^2K_3^{(1,2)}\rbr{d_1(K_2^{(1,2)})^{q_2}\|\Vb\|_F^{2q_2}r^{1-q_2} + d_2(K_1^{(1,2)})^{q_1}\|\Ub\|_F^{2q_1}r^{1-q_1}}}\\ \leq 2\exp\rbr{-(d_1\wedge d_2)K_3^{(1,2)}} + q_2\exp\rbr{-\frac{\|\Vb\|_F^2K_2^{(1,2)}}{\|\Vb\|_2^2}} + q_1\exp\rbr{-\frac{\|\Ub\|_F^2K_1^{(1,2)}}{\|\Ub\|_2^2}}.
\end{multline*}
For the condition (b) in Lemma \ref{lem:G:5},
\begin{align*}
\|\mE[\Hb_1\Hb_1^T]\|_2 \lesssim\beta^4\|\mE[\J_{11}]\|_2\stackrel{\eqref{f:8}}{\lesssim} \beta^4\Upsilon_1\Upsilon_2.
\end{align*}
For the condition (c) in Lemma \ref{lem:G:5},
\begin{align*}
\max_{\|\ba\|_F^2+\|\bb\|_F^2=1}\mE\big[\big(\begin{pmatrix}
\ba^T & \bb^T
\end{pmatrix}\Hb_1\begin{pmatrix}
\ba\\
\bb
\end{pmatrix}\big)^2\big] &\lesssim \beta^4  \max_{\|\ba\|_F^2+\|\bb\|_F^2=1}\mE\big[\big(\ba^T(\bd - \bd') + \bb^T(\bp - \bp')\big)^4\big]\\
&\stackrel{\eqref{f:11}}{\lesssim}  \beta^4\Upsilon_2^2.
\end{align*}
Thus, conditions of Lemma \ref{lem:G:5} hold with parameters (up to constants)
\begin{align*}
\mu_1(\J_2) \coloneqq & \beta^2K_3^{(1,2)}\rbr{d_1(K_2^{(1,2)})^{q_2}\|\Vb\|_F^{2q_2}r^{1-q_2} + d_2(K_1^{(1,2)})^{q_1}\|\Ub\|_F^{2q_1}r^{1-q_1}},\\
\nu_1(\J_2) \coloneqq & \exp\rbr{-(d_1\wedge d_2)K_3^{(1,2)}} + q_2\exp\rbr{-\frac{\|\Vb\|_F^2K_2^{(1,2)}}{\|\Vb\|_2^2}} + q_1\exp\rbr{-\frac{\|\Ub\|_F^2K_1^{(1,2)}}{\|\Ub\|_2^2}},\\
\nu_2(\J_2) \coloneqq & \beta^4\Upsilon_1\Upsilon_2, \quad\quad \nu_3(\J_2) \coloneqq  \beta^2\Upsilon_2, \quad\quad \|\mE[\Hb_1]\| \lesssim \beta^2\Upsilon_2.
\end{align*}
Similar to the proof of $\J_1$, for any $s\geq 1$, we let $K_1^{(1,2)} = K_2^{(1,2)} = 2\log n_1n_2 + s\log(d_1+d_2)$, $K_3^{(1,2)} = 1$, and
\begin{align*}
\epsilon_3\asymp \sqrt{\frac{s(d_1 + d_2)\log\rbr{r(d_1 + d_2)}}{n_1\wedge n_2}} \vee \frac{s(d_1 + d_2)\cbr{\log\rbr{r(d_1 + d_2)}}^{1 + q}}{n_1\wedge n_2},
\end{align*}
and then have
\begin{align*}
P\rbr{\J_2\gtrsim \beta^2\epsilon_3\Upsilon_2}\lesssim \frac{1}{(d_1 +d_2)^s}.
\end{align*}
Noting that the first term in $\epsilon_3$ is the dominant term, we complete the proof.
\end{proof}

\begin{lemma}\label{lem:E:7}

Under conditions of Lemma \ref{lem:B:4} and the definition of $q_1, q_2$ in Lemma \ref{lem:E:5}, we have that for any $s\geq 1$,
\begin{align*}
P\rbr{\T_1 \gtrsim \beta\sqrt{\frac{s(d_1 + d_2)\log\rbr{r(d_1 + d_2)}}{m}}\rbr{\|\Vb\|_2^{q_2(1-q_1)} + \|\Ub\|_2^{q_1(1-q_2)}}} & \lesssim \frac{1}{(d_1 + d_2)^s}, \\
P\rbr{\T_2 \gtrsim \beta\sqrt{\frac{s(d_1 + d_2)\log\rbr{r(d_1 + d_2)}}{n_1\wedge n_2}}\rbr{\|\Vb\|_2^{q_2(1-q_1)} + \|\Ub\|_2^{q_1(1-q_2)}}} & \lesssim \frac{1}{(d_1 + d_2)^s}.
\end{align*}

\end{lemma}

\begin{proof}
\textit{Proof of $\T_1$.} For any two samples $(y, \xb, \zb)\in\D$ and $(y', \xb', \zb')\in \D'$, we define
\begin{align}\label{f:14}
\Hb_2\rbr{(\xb, \zb), (\xb', \zb')} = \frac{y - y'}{1 + \exp\big((y - y')(\bTheta - \bTheta')\big)}\cdot \begin{pmatrix}
\bQ - \bQ' & \bS - \bS'\\
\bS^T - \bS'^T & \bR - \bR'
\end{pmatrix},
\end{align}
where $\bTheta = \LD \phi_1(\Ub^T\xb), \phi_2(\Vb^T\zb)\RD$. We follow the same proof sketch as Lemma \ref{lem:E:5}. We apply Lemma \ref{lem:G:4} to bound $\T_1$. We first check all conditions of Lemma \ref{lem:G:4}. By Assumption \ref{ass:2},
\begin{align}\label{d:3}
\|\Hb_2\|_2 &\leq 4\beta\bigg\|\begin{pmatrix}
\bQ& \bS\\
\bS^T & \bR
\end{pmatrix}\bigg\|_2 \nonumber\\
&\leq  4\beta \max_{\|\ba\|_F^2+\|\bb\|_F^2=1} \bigg|\sum_{p=1}^r\phi_1''(\bu_p^T\xb)\phi_2(\bv_p^T\zb)(\ba_p^T\xb)^2 + 2\sum_{p=1}^r\phi_1'(\bu_p^T\xb)\phi_2'(\bv_p^T\zb)\xb^T\ba_p\bb_p^T\zb \nonumber\\
& \qquad + \sum_{p=1}^r\phi_1(\bu_p^T\xb)\phi_2''(\bv_p^T\zb)(\bb_p^T\zb)^2\bigg| \nonumber\\
& \lesssim \beta \max_{\|\ba\|_F^2+\|\bb\|_F^2=1} \bigg|\sum_{p=1}^r \b1_{q_1=0}\cdot|\bv_p^T\zb|^{q_2}(\ba_p^T\xb)^2 + 2\sum_{p=1}^r |\xb^T\ba_p|\cdot |\bb_p^T\zb| \nonumber\\
& \qquad + \sum_{p=1}^r\b1_{q_2=0}\cdot|\bu_p^T\xb|^{q_1}(\bb_p^T\zb)^2\bigg| \nonumber\\
& \lesssim \beta\rbr{(1-q_1)\xb^T\xb\max_{p\in[r]}|\zb^T\bv_p|^{q_2} + (1-q_2)\zb^T\zb\max_{p\in[r]}|\xb^T\bu_p|^{q_1} + \|\xb\|_2\|\zb\|_2}.
\end{align}
Here, the third inequality is due to the fact that $|\phi_i''|\leq 2$ if $\phi_i\in \{\text{sigmoid}, \text{tanh}\}$ and $\phi_i'' = 0$ if $\phi_i$ is ReLU. Taking union bound over $\D\cup \D'$, noting that $\log(r(n_1 + n_2)(d_1 + d_2))\asymp \log\rbr{r(d_1 + d_2)}$, and applying Lemma \ref{lem:G:2}, for any $s \geq 1$, we define
\begin{align}\label{f:15}
\Upsilon_3 &= (1-q_1)d_1\rbr{\log(r(d_1 + d_2))}^{q_2/2}\|\Vb\|_2^{q_2} + \sqrt{d_1d_2} + (1-q_2)d_2\rbr{\log(r(d_1 + d_2))}^{q_1/2}\|\Ub\|_2^{q_1} \nonumber\\
& \asymp  d_1^{\frac{2-q_1}{2}}d_2^{\frac{q_1}{2}}\rbr{\log\rbr{r(d_1 + d_1)}}^{\frac{{q_2(1-q_1)}}{2}}\|\Vb\|_2^{q_2(1-q_1)}\nonumber\\
&\qquad + d_2^{\frac{2-q_2}{2}}d_1^{\frac{q_2}{2}}\rbr{\log\rbr{r(d_1 + d_1)}}^{\frac{{q_1(1-q_2)}}{2}}\|\Ub\|_2^{q_1(1-q_2)}
\end{align}
and have
\begin{align}\label{f:16}
P\rbr{\max_{\D\cup \D'}\|\Hb_2\|_2\gtrsim \beta\Upsilon_3}\lesssim \frac{1}{\rbr{d_1 + d_2}^s}.
\end{align}
Next, we bound the following quantity
\begin{align}\label{f:17}
&\frac{1}{n_1^2n_2^2}\sum_{(\xb, \zb)\in \D}\sum_{(\xb', \zb')\in \D'}\Hb_2\big((\xb, \zb), (\xb', \zb')\big)\Hb_2\big((\xb, \zb), (\xb', \zb')\big)^T \nonumber\\
&=  \frac{1}{n_1^2n_2^2}\sum_{(\xb, \zb)\in \D}\sum_{(\xb', \zb')\in \D'}\frac{(y - y')^2}{\rbr{1 + \exp\big((y - y')(\bTheta - \bTheta')\big)}^2}\cdot \begin{pmatrix}
\bQ - \bQ' & \bS - \bS'\\
\bS^T - \bS'^T & \bR - \bR'
\end{pmatrix}^2 \nonumber\\
&\preceq \frac{16\beta^2}{n_1n_2}\sum_{(\xb, \zb)\in \D}\begin{pmatrix}
\bQ & \bS\\
\bS^T & \bR
\end{pmatrix}^2 = \frac{16\beta^2}{n_1n_2}\sum_{(\xb, \zb)\in \D}\begin{pmatrix}
\bQ^2 + \bS\bS^T & \bQ\bS + \bS\bR\\
\bS^T\bQ + \bR\bS^T & \bR^2 + \bS^T\bS
\end{pmatrix} \coloneqq 16\beta^2\T_{11}.
\end{align}
Similarly to Lemma \ref{lem:E:5}, we have two steps.

\noindent{\bf Step 1.} Bound $\|\mE[\T_{11}]\|_2$. For any vectors $\ba = (\ba_1; \ldots; \ba_r)$ and $\bb = (\bb_1; \ldots; \bb_r)$ such that $\ba_p \in\mR^{d_1}$, $\bb_p\in\mR^{d_2}$ for $p\in[r]$ and $\|\ba\|_2^2 + \|\bb\|_2^2 = 1$,
\begin{align*}
\bigg| & \begin{pmatrix}
\ba & \bb
\end{pmatrix}\mE[\T_{11}]\begin{pmatrix}
\ba\\
\bb
\end{pmatrix}\bigg| \\
&=  \mE\bigg[\sum_{p=1}^r\bigg(\big(\phi_1''(\bu_p^T\xb)\phi_2(\bv_p^T\zb)\big)^2\xb^T\xb + \big(\phi_1'(\bu_p^T\xb)\phi_2'(\bv_p^T\zb)\big)^2\zb^T\zb\bigg)(\ba_p^T\xb)^2\\
&\quad + 2\sum_{p=1}^r\bigg(\phi_1''(\bu_p^T\xb)\phi_1'(\bu_p^T\xb)\phi_2(\bv_p^T\zb)\phi_2'(\bv_p^T\zb)\xb^T\xb \\ 
&\quad  + \phi_1(\bu_p^T\xb)\phi_1'(\bu_p^T\xb)\phi_2'(\bv_p^T\zb)\phi_2''(\bv_p^T\zb)\zb^T\zb\bigg)\xb^T\ba_p\bb_p^T\zb\\
&\quad + \sum_{p=1}^r\bigg(\big(\phi_1(\bu_p^T\xb)\phi_2''(\bv_p^T\zb)\big)^2\zb^T\zb + \big(\phi_1'(\bu_p^T\xb)\phi_2'(\bv_p^T\zb)\big)^2\xb^T\xb\bigg)(\bb_p^T\zb)^2\bigg]\\
&\lesssim  \mE\bigg[\sum_{p=1}^r\big((1-q_1)(\zb^T\bv_p\bv_p^T\zb)^{q_2}\xb^T\xb + \zb^T\zb\big)\cdot\xb^T\ba_p\ba_p^T\xb + \sum_{p=1}^r\big((1-q_2)(\xb^T\bu_p\bu_p^T\xb)^{q_1}\zb^T\zb\\
&\quad + \xb^T\xb\big)\cdot\zb^T\bb_p\bb_p^T\zb + \sum_{p=1}^r\big((1-q_1)|\bv_p^T\zb|^{q_2}\xb^T\xb + (1-q_2)|\bu_p^T\xb|^{q_1}\zb^T\zb\big)|\xb^T\ba_p\bb_p^T\zb|\bigg]\\
&\lesssim  \mE\big[(1-q_1)\xb^T\xb\sum_{p=1}^r(\zb^T\bv_p\bv_p^T\zb)^{q_2}\xb^T\ba_p\ba_p^T\xb + (1-q_2)\zb^T\zb\sum_{p=1}^r (\xb^T\bu_p\bu_p^T\xb)^{q_1}\zb^T\bb_p\bb_p^T\zb\\
&\quad + (1-q_1)\xb^T\xb\sum_{p=1}^r|\xb^T\ba_p|\cdot|\zb^T\bb_p|\cdot|\zb^T\bv_p|^{q_2} + (1-q_2)\zb^T\zb\sum_{p=1}^r|\zb^T\bb_p|\cdot|\xb^T\ba_p|\cdot|\xb^T\bu_p|^{q_1}\\
&\quad + \zb^T\zb\cdot\xb^T(\sum_{p=1}^{r}\ba_p\ba_p^T)\xb + \xb^T\xb\cdot\zb^T(\sum_{p=1}^{r}\bb_p\bb_p^T)\zb\big].
\end{align*}
Using Lemma \ref{lem:G:3} and maximizing over set $\{(\ba, \bb): \|\ba\|_2^2 + \|\bb\|_2^2 = 1\}$, we get
\begin{multline}\label{f:18}
\|\mE[\T_{11}]\|_2\lesssim (1-q_1)d_1\|\Vb\|_2^{2q_2} + (1-q_2)d_2\|\Ub\|_2^{2q_1} + d_1 + d_2 \lesssim d_1\|\Vb\|_2^{2q_2(1-q_1)} + d_2\|\Ub\|_2^{2q_1(1-q_2)}.
\end{multline}

\noindent{\bf Step 2.} Bound $\|\T_{11} - \mE[\T_{11}]\|_2$. We still apply Lemma \ref{lem:G:6}. Define the following random matrix
\begin{align*}
\Tb_{11}(\xb, \zb)\coloneqq \begin{pmatrix}
\bQ^2 + \bS\bS^T & \bQ\bS + \bS\bR\\
\bS^T\bQ + \bR\bS^T & \bR^2 + \bS^T\bS
\end{pmatrix}.
\end{align*}
For the condition (a) in Lemma \ref{lem:G:6}, we note that
\begin{align*}
\|\Tb_{11}(\xb, \zb)\|_2 &= \max_{p\in[r]}\bigg\|\begin{pmatrix}
\phi_1''(\bu_p^T\xb)\phi_2(\bv_p^T\zb)\cdot\xb\xb^T & \phi_1'(\bu_p^T\xb)\phi_2'(\bv_p^T\zb)\cdot\xb\zb^T\\
\phi_1'(\bu_p^T\xb)\phi_2'(\bv_p^T\zb)\cdot\zb\xb^T & \phi_1(\bu_p^T\xb)\phi_2''(\bv_p^T\zb)\cdot\zb\zb^T
\end{pmatrix}\bigg\|_2^2\\
&= \max_{p\in[r]}\bigg(\max_{\|\ba_p\|_2^2+\|\bb_p\|_2^2=1} \phi_1''(\bu_p^T\xb)\phi_2(\bv_p^T\zb)(\ba_p^T\xb)^2 \\
&\quad+ 2\phi_1'(\bu_p^T\xb)\phi_2'(\bv_p^T\zb)\cdot(\ba_p^T\bx)(\bb_p^T\zb) + \phi_1(\bu_p^T\xb)\phi_2''(\bv_p^T\zb)(\bb_p^T\zb)^2\bigg)^2\\
&\lesssim  \max_{p\in[r]}\bigg(\max_{\|\ba_p\|_2^2+\|\bb_p\|_2^2=1}(1-q_1)|\bv_p^T\zb|^{q_2}\xb^T\ba_p\ba_p^T\xb \\
&\quad+ |\xb^T\ba_p\bb_p^T\zb| + (1-q_2)|\bu_p^T\xb|^{q_1}\zb^T\bb_p\bb_p^T\zb\bigg)^2\\
&\lesssim  \max_{p\in[r]}\bigg((1-q_1)(\zb^T\bv_p\bv_p^T\zb)^{q_2}(\xb^T\xb)^2 \\
&\quad+ (\xb^T\xb)(\zb^T\zb) + (1-q_2)(\xb^T\bu_p\bu_p^T\xb)^{q_1}(\zb^T\zb)^2\bigg).
\end{align*}
By Lemma \ref{lem:G:2}, for any $K_1^{(2,1)}\wedge K_2^{(2,1)}\wedge K_3^{(2,1)}\geq 1$, defining
\begin{align}\label{f:19}
\Upsilon_4 = d_1(K_2^{(2,1)})^{\frac{q_2(1-q_1)}{2}}\|\Vb\|_2^{q_2(1-q_1)} + d_2(K_1^{(2,1)})^{\frac{q_1(1-q_2)}{2}}\|\Ub\|_2^{q_1(1-q_2)}
\end{align}
and we have
\begin{multline}\label{f:20}
P\rbr{\|\Tb_{11}(\xb, \zb)\|_2\gtrsim (K_3^{(2,1)})^2 \Upsilon_4^2}\\
\leq 2\exp\rbr{-(d_1\wedge d_2)K_3^{(2,1)}} + (1-q_1)q_2r\exp(-K_2^{(2,1)}) + (1-q_2)q_1r\exp(-K_1^{(2,1)}).
\end{multline}
For the condition (b) in Lemma \ref{lem:G:6}, let us define
\begin{align*}
\Tb_{11}^{(1)} &\coloneqq  (1-q_1)(\bv_p^T\zb)^{2q_2}\xb^T\xb + \zb^T\zb,\\
\Tb_{11}^{(2)} &\coloneqq  (1-q_1)|\bv_p^T\zb|^{q_2}\xb^T\xb + (1-q_2)|\bu_p^T\xb|^{q_1}\zb^T\zb,\\
\Tb_{11}^{(3)} &\coloneqq (1-q_2)(\bu_p^T\xb)^{2q_1}\zb^T\zb + \xb^T\xb.
\end{align*}
Then,
\begin{align*}
\|\mE[\Tb_{11}(\xb, \zb)\Tb_{11}(\xb, \zb)^T]\|_2 \lesssim & \max_{\|\ba\|_F^2+\|\bb\|_F^2=1} \mE\bigg[\bigg(\sum_{p=1}^r\big((\Tb_{11}^{(1)})^2\xb^T\xb + (\Tb_{11}^{(2)})^2\zb^T\zb\big)(\ba_p^T\xb)^2\bigg) \\
& + 2\bigg(\sum_{p=1}^r\big(\Tb_{11}^{(1)}\Tb_{11}^{(2)}\xb^T\xb + \Tb_{11}^{(3)}\Tb_{11}^{(2)}\zb^T\zb\big)|\xb^T\ba_p\bb_p^T\zb|\bigg) \\
&+ \bigg(\sum_{p=1}^r\big((\Tb_{11}^{(3)})^2\zb^T\zb + (\Tb_{11}^{(2)})^2\xb^T\xb\big)(\bb_p^T\zb)^2\bigg)\bigg].
\end{align*}
By simple calculations based on Lemma \ref{lem:G:3},
\begin{align*}
&\mE\big[(\Tb_{11}^{(1)})^2\xb^T\xb\cdot \xb^T\ba_p\ba_p^T\xb\big]\lesssim  \big((1-q_1)\|\bv_p\|_2^{4q_2}d_1^2 + d_2^2\big) d_1\|\ba_p\|_2^2,\\
&\mE\big[(\Tb_{11}^{(2)})^2\zb^T\zb\cdot \xb^T\ba_p\ba_p^T\xb\big]\lesssim  \big((1-q_1)\|\bv_p\|_2^{2q_2}d_1^2 + (1-q_2)\|\bu_p\|_2^{2q_1}d_2^2\big) d_2\|\ba_p\|_2^2,\\
&\mE\big[(\Tb_{11}^{(3)})^2\zb^T\zb\cdot \zb^T\bb_p\bb_p^T\zb\big]\lesssim \big((1-q_2)\|\bu_p\|_2^{4q_1}d_2^2 + d_1^2\big) d_2\|\bb_p\|_2^2,\\
&\mE\big[(\Tb_{11}^{(2)})^2\xb^T\xb\cdot \zb^T\bb_p\bb_p^T\zb\big]\lesssim \big((1-q_1)\|\bv_p\|_2^{2q_2}d_1^2 + (1-q_2)\|\bu_p\|_2^{2q_1}d_2^2\big) d_1\|\bb_p\|_2^2,\\
&\mE\big[\Tb_{11}^{(1)}\Tb_{11}^{(2)}\xb^T\xb|\xb^T\ba_p\bb_p^T\zb|\big] \lesssim \big((1-q_1)\|\bv_p\|_2^{2q_2}d_1 + d_2\big)\big((1-q_1)\|\bv_p\|_2^{q_2}d_1  \\
&\quad \quad\quad\quad\quad\quad\quad\quad\quad\quad\quad\quad + (1-q_2)\|\bu_p\|_2^{q_1}d_2\big)d_1\|\ba_p\|_2\|\bb_p\|_2,\\
&\mE\big[\Tb_{11}^{(3)}\Tb_{11}^{(2)}\zb^T\zb|\xb^T\ba_p\bb_p^T\zb|\bigg]  \lesssim \big((1-q_1)\|\bu_p\|_2^{2q_1}d_2 + d_1\big)\big((1-q_1)\|\bv_p\|_2^{q_2}d_1 \\
&\quad \quad\quad\quad\quad\quad\quad\quad\quad\quad\quad\quad+ (1-q_2)\|\bu_p\|_2^{q_1}d_2\big)d_2\|\ba_p\|_2\|\bb_p\|_2.
\end{align*}
Combining the above displays together and maximizing over $\{(\ba, \bb): \|\ba\|_2^2 + \|\bb\|_2^2 = 1\}$,
\begin{align}\label{f:21}
\|\mE[\Tb_{11}(\xb, \zb)\Tb_{11}(\xb, \zb)^T]\|_2 \lesssim d_1^{3-q_1}d_2^{q_1}\|\Vb\|_2^{4q_2(1-q_1)} + d_2^{3-q_2}d_1^{q_2}\|\Ub\|_2^{4q_1(1-q_2)}.
\end{align}
For condition (c) in Lemma \ref{lem:G:6},
\begin{multline*}
\mE[\big((\ba; \bb)^T\Tb_{11}\big(\xb, \zb\big) (\ba; \bb)\big)^2]\\
\lesssim  \mE \bigg[\bigg(\sum_{p=1}^r\Tb_{11}^{(1)}\xb^T\ba_p\ba_p^T\xb + 2\sum_{p=1}^r\Tb_{11}^{(2)}|\xb^T\ba_p\bb_p^T\zb| + \sum_{p=1}^r\Tb_{11}^{(3)}\zb^T\bb_p\bb_p^T\zb\bigg)^2\bigg].
\end{multline*}
Applying Lemma \ref{lem:G:3},
\begin{align*}
\mE\big[\big(\sum_{p=1}^r\Tb_{11}^{(1)}\xb^T\ba_p\ba_p^T\xb\big)^2\big]&\lesssim \big((1-q_1)\|\Vb\|_2^{4q_2}d_1^2 + d_2 ^2\big)\bigg(\sum_{p=1}^r\|\ba_p\|_2^2\bigg)^2,\\
\mE\big[\big(\sum_{p=1}^r\Tb_{11}^{(2)}|\xb^T\ba_p\bb_p^T\zb|\big)^2\big]&\lesssim \big((1-q_1)\|\Vb\|_2^{2q_2}d_1^2 \\
&\quad + (1-q_2)\|\Ub\|_2^{2q_1}d_2^2\big)\bigg(\sum_{p=1}^r\|\ba_p\|_2^2\bigg)\bigg(\sum_{p=1}^r\|\bb_p\|_2^2\bigg),\\
\mE\big[\big(\sum_{p=1}^r\Tb_{11}^{(3)}\zb^T\bb_p\bb_p^T\zb\big)^2\big]&\lesssim \big((1-q_2)\|\Ub\|_2^{4q_1}d_2^2 + d_1^2\big)\bigg(\sum_{p=1}^r\|\bb_p\|_2^2\bigg)^2.
\end{align*}
Thus,
\begin{align}\label{f:22}
\max_{\|\ba\|_F^2+\|\bb\|_F^2=1}\rbr{\mE\big[\rbr{(\ba; \bb)^T\Tb_{11}\big(\xb, \zb\big) (\ba; \bb)}^2\big]}^{1/2}\lesssim d_1\|\Vb\|_2^{2q_2(1-q_1)} + d_2\|\Ub\|_2^{2q_1(1-q_2)}.
\end{align}
Combining \eqref{f:18}, \eqref{f:20}, \eqref{f:21}, \eqref{f:22}, and defining
\begin{align}\label{f:23}
\Upsilon_5 & =  d_1^{3-q_1}d_2^{q_1}\|\Vb\|_2^{4q_2(1-q_1)} + d_2^{3-q_2}d_1^{q_2}\|\Ub\|_2^{4q_1(1-q_2)}, \quad \Upsilon_6 = d_1\|\Vb\|_2^{2q_2(1-q_1)} + d_2\|\Ub\|_2^{2q_1(1-q_2)},
\end{align}
then conditions in Lemma \ref{lem:G:6} hold for $\T_{11}$ with parameters
\begin{align*}
\nu_1(\T_{11}) &\coloneqq  \exp\rbr{-(d_1\wedge d_2)K_3^{(2,1)}} + q_2(1-q_1)r\exp\rbr{-K_2^{(2,1)}} + q_1(1-q_2)r\exp\rbr{-K_1^{(2,1)}},\\
\mu_1(\T_{11}) &\coloneqq  (K_3^{(2,1)})^2\Upsilon_4^2, \quad\quad \nu_2(\T_{11}) \coloneqq  \Upsilon_5 \quad\quad \nu_3(\T_{11}) \coloneqq  \Upsilon_6, \quad\quad \|\mE[\T_{11}]\| \lesssim \Upsilon_6.
\end{align*}
Here, $\Upsilon_4$, $\Upsilon_5$, $\Upsilon_6$ are defined in \eqref{f:19}, \eqref{f:23}, and $\{K_i^{(2,1)}\}_{i = 1,2,3}$ are any constant. So, $\forall t>0$,
\begin{align*}
P\big(\big\|\T_{11} - &\mE[\T_{11}]\big\|_2> t + \Upsilon_6\sqrt{\nu_1(\T_{11})}\big)\\
&\leq n_1n_2\nu_1(\T_{11}) + 2r(d_1 + d_2)\exp\bigg(-\frac{(n_1\wedge n_2)t^2}{\big(2\Upsilon_5 + 4\Upsilon_6^2 + 4\Upsilon_6^2\nu_1(\T_{11})\big) + 4\mu_1(\T_{11})t}\bigg)\\
&\leq  n_1n_2\nu_1(\I_{21}) + 2r(d_1 + d_2)\exp\bigg(-\frac{(n_1\wedge n_2)t^2}{10\Upsilon_5+ 4\mu_1(\T_{11})t}\bigg).
\end{align*}
For any $s\geq 1$, we let
\begin{align*}
K_1^{2,1} = K_2^{2,1} = \log(n_1n_2r) + s\log(d_1 + d_2), \quad K_3^{2,1} = 1.
\end{align*}
Then, $\Upsilon_4 \asymp \Upsilon_3$. Noting that $q = q_1\vee q_2$ and $q' = q_1q_2$, we can let
\begin{align*}
\epsilon_4\asymp \sqrt{\frac{s(d_1 + d_2)\log\rbr{r(d_1 + d_2)}}{n_1\wedge n_2}}\vee \frac{s(d_1 +d_2)\cbr{\log\rbr{r(d_1 + d_2)}}^{1+q-q'}}{n_1\wedge n_2},
\end{align*}
and then have
\begin{align*}
P\rbr{\|\T_{11} - \mE[\T_{11}]\|_2\geq \epsilon_4 \Upsilon_6}\lesssim \frac{1}{(d_1 + d_2)^s}.
\end{align*}
Combining the above inequality with \eqref{f:18}, $P(\|\T_{11}\|_2\gtrsim \Upsilon_6)\lesssim 1/(d_1 + d_2)^s$. We plug back into \eqref{f:17}, combine with \eqref{f:16}, and know Lemma \ref{lem:G:4} holds for $\T_1$ with parameters $\nu_1(\T_1) = \beta\Upsilon_3$ and $\nu_2(\T_1) = \beta^2\Upsilon_6$. Finally we apply Lemma \ref{lem:G:4} and obtain that $\forall t>0$
\begin{align*}
P\rbr{\T_1 > t}\lesssim 2r(d_1 + d_2)\exp\rbr{-\frac{mt^2}{4\nu_2(\T_1)+ 4\nu_1(\T_1)t}}.
\end{align*}
For any $s\geq 1$, we let
\begin{align}\label{Ups_7}
\epsilon_5 &\asymp \sqrt{\frac{s(d_1 + d_2)\log\rbr{r(d_1 + d_2)}}{m}} \vee \frac{s(d_1 + d_2)\cbr{\log\rbr{r(d_1 + d_2)}}^{1 + \frac{q - q'}{2} }}{m},\nonumber\\
\Upsilon_7 &=  \|\Vb\|_2^{q_2(1-q_1)} + \|\Ub\|_2^{q_1(1-q_2)},
\end{align}
and have
\begin{align*}
P\rbr{\T_1 > \beta\epsilon_5\Upsilon_7}\lesssim \frac{1}{(d_1 + d_2)^s}.
\end{align*}
This completes the proof for the first part.

\noindent\textit{Proof of $\T_2$.} We apply Lemma \ref{lem:G:5} to bound $\T_2$. We check all conditions of Lemma \ref{lem:G:5}. By definition of $\Hb_2$ in \eqref{f:14},
\begin{align*}
\nabla^2\bar\mL_2(\Ub, \Vb)
= \frac{1}{n_1^2n_2^2}\sum_{(\xb, \zb)\in \D}\sum_{(\xb', \zb')\in \D'}\Hb_2\big((\xb, \zb), (\xb', \zb')\big).
\end{align*}
We first bound $\|\mE[\Hb_2]\|_2$ as follows:
\begin{align}\label{f:24}
\|\mE[\Hb_2]\|_2&\lesssim \beta\|\mE\bigg[\begin{pmatrix}
\bQ & \bS\\
\bS^T & \bR
\end{pmatrix}\bigg]\|_2 \nonumber\\
&\lesssim \beta \max_{\|\ba\|_F^2+\|\bb\|_F^2=1}\bigg|\sum_{p=1}^r\mE\big[\phi_1''(\bu_p^T\xb)\phi_2(\bv_p^T\zb)(\ba_p^T\xb)^2\big] \nonumber\\
& \quad+ 2\sum_{p=1}^r\mE\big[\phi_1'(\bu_p^T\xb)\phi_2'(\bv_p^T\zb)\xb^T\ba_p\bb_p^T\zb\big]  +\sum_{p=1}^r\mE\big[\phi_1(\bu_p^T\xb)\phi_2''(\bv_p^T\zb)(\bv_p^T\zb)^2\big]\bigg| \nonumber\\
&\lesssim  \beta \max_{\|\ba\|_F^2+\|\bb\|_F^2=1} \bigg|(1-q_1)\sum_{p=1}^r\mE\big[|\bv_p^T\zb|^{q_2}\xb^T\ba_p\ba_p^T\xb\big] + \sum_{p=1}^r\mE\big[|\xb^T\ba_p\bb_p^T\zb|\big] \nonumber\\
& \quad + (1-q_2)\sum_{p=1}^r\mE\big[|\bu_p^T\xb|^{q_1}\zb^T\bb_p\bb_p^T\zb\big]\bigg| \nonumber\\
&\leq \beta\max_{\|\ba\|_F^2+\|\bb\|_F^2=1}\bigg((1-q_1)\sum_{p=1}^r\|\bv_p\|_2^{q_2}\|\ba_p\|_2^2+\sum_{p=1}^r\|\ba_p\|_2\|\bb_p\|_2 \nonumber\\
& \quad+ (1-q_2)\sum_{p=1}^r\|\bu_p\|_2^{q_1}\|\bb_p\|_2^2\bigg) \nonumber\\
&\leq \beta\big((1-q_1)\|\Vb\|_2^{q_2} + 1 + (1-q_2)\|\Ub\|_2^{q_1}\big) \nonumber\\
&\leq\beta\Upsilon_7.
\end{align}
For the condition (a) in Lemma \ref{lem:G:5}, we have shown in \eqref{d:3} that
\begin{align*}
\|\Hb_2\|_2\lesssim \beta\big((1-q_1)\xb^T\xb \max_{p\in[r]}|\zb^T\bv_p|^{q_2} + (1-q_2)\zb^T\zb\max_{p\in[r]}|\xb^T\bu_p|^{q_1} + \|\xb\|_2\|\zb\|_2\big).
\end{align*}
Thus, similar to \eqref{f:20},
\begin{multline*}
P\rbr{\|\Hb_2\|_2\gtrsim \beta K_3^{(2,2)}\rbr{d_1(K_2^{(2,2)})^{\frac{q_2(1-q_1)}{2}}\|\Vb\|_2^{q_2(1-q_1)} + d_2(K_1^{(2,2)})^{\frac{q_1(1-q_2)}{2}}\|\Ub\|_2^{q_1(1-q_2)}}}\\
\leq 2\exp\rbr{-(d_1 \wedge d_2) K_3^{(2,2)}} + (1-q_1)q_2r\exp(-K_2^{(2,2)})+ (1-q_2)q_1\exp(-K_1^{(2,2)}).
\end{multline*}
For the condition (b) in Lemma \ref{lem:G:5}, 
\begin{align*}
\|\mE[\Hb_2\Hb_2^T]\|_2\lesssim \beta^2\|\mE[\T_{11}]\|_2\stackrel{\eqref{f:18}}{\lesssim}\beta^2\Upsilon_6.
\end{align*}
For the condition (c) in Lemma \ref{lem:G:5}, we use Lemma \ref{lem:G:3} and obtain
\begin{align*}
\max_{\|\ba\|_F^2+\|\bb\|_F^2=1} \mE\big[\big(\begin{pmatrix}
\ba^T & \bb^T
\end{pmatrix}&\Hb_2\begin{pmatrix}
\ba\\
\bb
\end{pmatrix}\big)^2\big]\\
&\lesssim \beta^2\max_{\|\ba\|_F^2+\|\bb\|_F^2=1}\mE\big[\big(\begin{pmatrix}
\ba^T & \bb^T
\end{pmatrix}\begin{pmatrix}
\bQ & \bS\\
\bS^T & \bR
\end{pmatrix}\begin{pmatrix}
\ba\\
\bb
\end{pmatrix}\big)^2\big]\\
&\lesssim  \beta^2\max_{\|\ba\|_F^2+\|\bb\|_F^2=1}\mE\bigg[\bigg(\sum_{p=1}^r(1-q_1)|\zb^T\bv_p|^{q_2}\xb^T\ba_p\ba_p^T\xb \\
& \quad+ \sum_{p=1}^r|\xb^T\ba_p\bb_p^T\zb|  + \sum_{p=1}^r(1-q_2)|\xb^T\bu_p|^{q_1}\zb^T\bb_p\bb_p^T\zb\bigg)^2\bigg]\\
&\lesssim  \beta^2\big((1-q_1)\|\Vb\|_2^{2q_2} + 1 + (1-q_2)\|\Ub\|_2^{2q_1}\big)\\
&\lesssim  \beta^2\Upsilon_7^2.
\end{align*}
Thus, conditions of Lemma \ref{lem:G:5} hold for $\T_2$ with parameters (up to constants)
\begin{align*}
\mu_1(\T_2) &\coloneqq  \beta K_3^{(2,2)}\rbr{d_1(K_2^{(2,2)})^{\frac{q_2(1-q_1)}{2}}\|\Vb\|_2^{q_2(1-q_1)} + d_2(K_1^{(2,2)})^{\frac{q_1(1-q_2)}{2}}\|\Ub\|_2^{q_1(1-q_2)}},\\
\nu_1(\T_2) &\coloneqq  \exp\rbr{-(d_1 \wedge d_2) K_3^{(2,2)}} + (1-q_1)q_2r\exp(-K_2^{(2,2)})+ (1-q_2)q_1\exp(-K_1^{(2,2)}),\\
\nu_2(\T_2) &\coloneqq  \beta^2\Upsilon_6, \quad\quad \nu_3(\T_2) \coloneqq  \beta\Upsilon_7, \quad\quad \|\mE[\Hb_2]\| \lesssim \beta\Upsilon_7.
\end{align*}
For any $s\geq 1$, we let $K_1^{(2,2)} = K_2^{(2,2)} = 2\log n_1n_2r + s\log(d_1 + d_2)$, $K_3^{(2,2)}=1$, and
\begin{align*}
\epsilon_6\asymp \sqrt{\frac{s(d_1 + d_2)\log\rbr{r(d_1 + d_2)}}{n_1\wedge n_2}} \vee \frac{s(d_1 + d_2)\cbr{\log\rbr{r(d_1 +d_2)}}^{1 + \frac{q - q'}{2}}}{n_1\wedge n_2},
\end{align*}
and then have
\begin{align*}
P\big( \T_2 \gtrsim \beta\epsilon_6\Upsilon_7\big)\lesssim \frac{1}{(d_1 +d_2)^s}.
\end{align*}
We finish the proof by noting that the first term of $\epsilon_6$ is the dominant.
\end{proof}

\begin{lemma}\label{lem:E:9}

Under conditions of Lemma \ref{lem:B:3}, we have
\begin{align*}
\J_3 \lesssim \beta^3 r^{\frac{3(1-q)}{2}}\rbr{\|\tVb\|_F^{3q} + \|\tUb\|_F^{3q}}\rbr{\|\Ub - \tUb\|_F^{1-q/2} + \|\Vb - \tVb\|_F^{1-q/2}}.
\end{align*}

\end{lemma}

\begin{proof}
By definition of $\J_3$,
\begin{align}\label{f:25}
\|\mE[\nabla^2\mL_1(\Ub, \Vb)] - &\mE[\nabla^2\mL_1(\tUb, \tVb)]\|_2 \nonumber\\
&=  \bigg\|\mE\bigg[A\begin{pmatrix}
\bd - \bd'\\
\bp - \bp'
\end{pmatrix}\begin{pmatrix}
\bd - \bd'\\
\bp - \bp'
\end{pmatrix}^T\bigg] - \mE\bigg[A^\star\begin{pmatrix}
\bd^\star - \bd'^\star\\
\bp^\star - \bp'^\star
\end{pmatrix}\begin{pmatrix}
\bd^\star - \bd'^\star\\
\bp^\star - \bp'^\star
\end{pmatrix}^T\bigg]\bigg\|_2 \nonumber\\
&\leq  \bigg\|\mE\bigg[A\bigg(\begin{pmatrix}
\bd - \bd'\\
\bp - \bp'
\end{pmatrix}\begin{pmatrix}
\bd - \bd'\\
\bp - \bp'
\end{pmatrix}^T - \begin{pmatrix}
\bd^\star - \bd'^\star\\
\bp^\star - \bp'^\star
\end{pmatrix}\begin{pmatrix}
\bd^\star - \bd'^\star\\
\bp^\star - \bp'^\star
\end{pmatrix}^T\bigg)\bigg]\bigg\|_2 \nonumber\\
& \quad + \bigg\|\mE\bigg[(A-A^\star)\begin{pmatrix}
\bd^\star - \bd'^\star\\
\bp^\star - \bp'^\star
\end{pmatrix}\begin{pmatrix}
\bd^\star - \bd'^\star\\
\bp^\star - \bp'^\star
\end{pmatrix}^T\bigg]\bigg\|_2 \nonumber\\
& \eqqcolon \|\J_{31}\|_2 + \|\J_{32}\|_2.
\end{align}
For $\J_{31}$,
\begin{multline*}
\|\J_{31}\|_2\lesssim \beta^2\bigg(\bigg\|\mE\bigg[\begin{pmatrix}
\bd\\
\bp
\end{pmatrix}\begin{pmatrix}
\bd\\
\bp
\end{pmatrix}^T - \begin{pmatrix}
\td\\
\tp
\end{pmatrix}\begin{pmatrix}
\td\\
\tp
\end{pmatrix}^T\bigg]\bigg\|_2 \\+ \bigg\|\mE\begin{pmatrix}
\bd\\
\bp
\end{pmatrix}\mE\begin{pmatrix}
\bd\\
\bp
\end{pmatrix}^T - \mE\begin{pmatrix}
\td\\
\tp
\end{pmatrix}\mE\begin{pmatrix}
\td\\
\tp
\end{pmatrix}^T\bigg\|_2\bigg).
\end{multline*}
We only bound the first term. The second term has the same bound using the equation $\mE[\xb]\mE[\xb]^T = \mE[\xb\xb'^T]$ for any variable $\xb'$ independent from $\xb$. Note that
\begin{align}\label{f:26}
\bigg\|\mE\bigg[\begin{pmatrix}
\bd\\
\bp
\end{pmatrix}\begin{pmatrix}
\bd\\
\bp
\end{pmatrix}^T - &\begin{pmatrix}
\td\\
\tp
\end{pmatrix}\begin{pmatrix}
\td\\
\tp
\end{pmatrix}^T\bigg]\bigg\|_2 \nonumber\\
= & \max_{\|\ba\|_F^2+\|\bb\|_F^2=1} \bigg|\sum_{i,j=1}^r\mE\big[\big(\phi_1'(\bu_i^T\xb)\phi_2(\bv_i^T\zb)\phi_1'(\bu_j^T\xb)\phi_2(\bv_j^T\zb) \nonumber\\
& \quad - \phi_1'(\tuT_i\xb)\phi_2(\tvT_i\zb)\phi_1'(\tuT_j\xb)\phi_2(\tvT_j\zb)\big)\cdot (\xb^T\ba_i\ba_j^T\xb)\big] \nonumber\\ 
&\quad + 2\sum_{i,j=1}^r\mE\big[\big(\phi_1'(\bu_i^T\xb)\phi_2(\bv_i^T\zb)\phi_1(\bu_j^T\xb)\phi_2'(\bv_j^T\zb) \nonumber\\
& \quad- \phi_1'(\tuT_i\xb)\phi_2(\tvT_i\zb)\phi_1(\tuT_j\xb)\phi_2'(\tvT_j\zb)\big) \cdot(\xb^T\ba_i\bb_j^T\xb)\big] \nonumber\\
&\quad + \sum_{i,j=1}^r\mE\big[\big(\phi_1(\bu_i^T\xb)\phi_2'(\bv_i^T\zb)\phi_1(\bu_j^T\xb)\phi_2'(\bv_j^T\zb) \nonumber\\
&\quad - \phi_1(\tuT_i\xb)\phi_2'(\tvT_i\zb)\phi_1(\tuT_j\xb)\phi_2'(\tvT_j\zb)\big)\cdot (\zb^T\bb_i\bb_j^T\zb)\big]\bigg|.
\end{align}
We focus on the first term in the above equality. By simple calculations using the boundedness and Lipschitz continuity of $\phi_i, \phi_i'$,
\begin{align*}
\big|\phi_1'(\bu_i^T\xb)\phi_2(\bv_i^T\zb)&\phi_1'(\bu_j^T\xb)\phi_2(\bv_j^T\zb) - \phi_1'(\tuT_i\xb)\phi_2(\tvT_i\zb)\phi_1'(\tuT_j\xb)\phi_2(\tvT_j\zb)\big|\\
&\leq |\phi_1'(\bu_i^T\xb) - \phi_1'(\tuT_i\xb)|\cdot |\zb^T\tv_i\tvT_j\zb|^{q_2} + |\zb^T(\bv_i - \tv_i)|\cdot|\tvT_j\zb|^{q_2}\\
&\quad  + |\phi_1'(\bu_j^T\xb) - \phi_1'(\tuT_j\xb)|\cdot|\zb^T\tv_i\tvT_j\zb|^{q_2} + |\zb^T(\bv_j - \tvT_j)|\cdot|\tvT_i\zb|^{q_2}.
\end{align*}
Plugging the above inequality back into \eqref{f:26}, dealing with other terms similarly, and applying Lemma \ref{lem:G:7} by noting $\sigma_r(\tUb) \wedge \sigma_r(\tVb)\geq 1$,
\begin{align}\label{f:27}
&\bigg\|\mE\bigg[\begin{pmatrix}
\bd\\
\bp
\end{pmatrix}\begin{pmatrix}
\bd\\
\bp
\end{pmatrix}^T - \begin{pmatrix}
\td\\
\tp
\end{pmatrix}\begin{pmatrix}
\td\\
\tp
\end{pmatrix}^T\bigg]\bigg\|_2 \nonumber\\
&\lesssim  \max_{\|\ba\|_F^2+\|\bb\|_F^2=1} \sum_{i,j=1}^r \|\bu_i - \tu_i\|_2^{1-\frac{q_1}{2}}\|\tv_i\|_2^{q_2}\|\tv_j\|_2^{q_2}\|\ba_i\|_2\|\ba_j\|_2 + \sum_{i,j=1}^r\|\bv_i - \tv_i\|_2\|\tv_j\|_2^{q_2}\|\ba_i\|_2\|\ba_j\|_2 \nonumber\\
& \quad + \sum_{i,j=1}^r\|\bu_i - \tu_i\|_2^{1-\frac{q_1}{2}}\|\tu_j\|_2^{q_1}\|\tv_i\|^{q_2}\|\ba_i\|_2\|\bb_j\|_2 + \sum_{i,j=1}^r\|\bv_i - \tv_i\|_2\|\tu_j\|^{q_1}\|\ba_i\|_2\|\bb_j\|_2 \nonumber\\
&\quad  + \sum_{i,j=1}^r\|\bv_i - \tv_i\|_2^{1-\frac{q_2}{2}}\|\tv_j\|_2^{q_1}\|\tu_i\|_2^{q_1}\|\ba_j\|_2\|\bb_i\|_2 + \sum_{i,j=1}^r\|\bu_i - \tu_i\|_2\|\tv_j\|_2\|\ba_j\|_2\|\bb_i\|_2 \nonumber\\
& \quad + \sum_{i,j=1}^r\|\bv_i - \tv_i\|_2^{1-\frac{q_2}{2}}\|\tu_i\|_2^{q_1}\|\tu_j\|_2^{q_1}\|\bb_i\|_2\|\bb_j\|_2 + \sum_{i,j=1}^r\|\bu_i - \tu_i\|_2\|\tu_j\|_2^{q_1}\|\bb_i\|_2\|\bb_j\|_2 \nonumber\\
&=  \max_{\|\ba\|_F^2+\|\bb\|_F^2=1} \bigg(\sum_{i=1}^r\big(\|\bu_i - \tu_i\|_2^{1-\frac{q_1}{2}}\|\tv_i\|_2^{q_2} + \|\bv_i - \tv_i\|_2\big)\|\ba_i\| + \sum_{j=1}^r\big(\|\bv_j - \tv_j\|_2^{1-\frac{q_2}{2}}\|\tu_j\|_2^{q_1} \nonumber\\
& \quad + \|\bu_j - \tu_j\|_2\big)\|\bb_j\|_2\bigg) \cdot \bigg(\sum_{i=1}^r\|\tv_i\|_2^{q_2}\|\ba_i\|_2 + \sum_{j=1}^r\|\tu_j\|_2^{q_1}\|\bb_j\|_2\bigg) \nonumber\\
&\leq \sqrt{\|\Ub - \tUb\|_F^2 + \|\Vb - \tVb\|_F^2 + \sum_{i=1}^r\|\bu_i - \tu_i\|_2^{2-q_1}\|\tv_i\|_2^{2q_2} + \|\bv_i - \tv_i\|_2^{2-q_2}\|\tu_i\|_2^{2q_1}} \nonumber\\
& \quad \cdot \sqrt{\sum_{i=1}^r\|\tv_i\|_2^{2q_2} + \|\tu_i\|_2^{2q_1}} \nonumber\\
&\leq  \big(\|\Ub - \tUb\|_F + \|\Vb - \tVb\|_F + \|\Ub - \tUb\|_2^{1-\frac{q_1}{2}} + \|\Vb - \tVb\|_2^{1-\frac{ q_2}{2}}\big)\Upsilon_2^\star,
\end{align}
where $\Upsilon_2^\star$ is defined in the same way as $\Upsilon_2$ in \eqref{f:12} but calculated using $\tUb, \tVb$. Next, we bound $\J_{32}$ in \eqref{f:25}. Since $\psi$ is Lipschitz continuous,
\begin{multline*}
|A - A^\star|
\lesssim\beta^3|\phi_1(\Ub^T\xb)^T\phi_2(\Vb^T\zb) - \phi_1(\tUb^T\xb)^T\phi_2(\tVb^T\zb)| \\
+ |\phi_1(\Ub^T\xb')^T\phi_2(\Vb^T\zb') - \phi_1(\tUb^T\xb')^T\phi_2(\tVb^T\zb')|.
\end{multline*}
Thus,
\begin{align*}
\|\J_{32}\|_2&\lesssim \beta^3\bigg\|\mE\bigg[ \big|\phi_1(\Ub^T\xb)^T\phi_2(\Vb^T\zb) - \phi_1(\tUb^T\xb)^T\phi_2(\tVb^T\zb) \big|\begin{pmatrix}
\bd^\star - \bd'^\star\\
\bp^\star - \bp'^\star
\end{pmatrix}\begin{pmatrix}
\bd^\star - \bd'^\star\\
\bp^\star - \bp'^\star
\end{pmatrix}^T \bigg]\bigg\|_2\\
&\lesssim  \beta^3 \sqrt{\mE\big[\rbr{\phi_1(\Ub^T\xb)^T\phi_2(\Vb^T\zb) - \phi_1(\tUb^T\xb)^T\phi_2(\tVb^T\zb)}^2\big]}\cdot\\
&\quad\max_{\|\ba\|_F^2+\|\bb\|_F^2=1}\sqrt{\mE[(\tdT\ba + \tpT\bb)^4]}.
\end{align*}
For the first term,
\begin{align*}
\mE\big[\big(&\phi_1(\Ub^T\xb)^T\phi_2(\Vb^T\zb) -  \phi_1(\tUb^T\xb)^T\phi_2(\tVb^T\zb)\big)^2\big]\\
&\lesssim  \mE\big[\big|\big(\phi_1(\Ub^T\xb) - \phi_1(\tUb^T\xb)\big)^T\phi_2(\tVb^T\zb)\big|^2\big] + \mE\big[\big|\big(\phi_2(\Vb^T\zb) - \phi_2(\tVb^T\zb)\big)^T\phi_1(\tUb^T\xb)\big|^2\big]\\
&\lesssim  \mE\big[\big(\sum_{p=1}^r|(\bu_p - \tu_p)^T\xb|\cdot|\tvT_p\zb|^{q_2}\big)^2\big] + \mE\big[\big(\sum_{p=1}^r|(\bv_p - \tv_p)^T\zb|\cdot|\tuT_p\xb|^{q_1}\big)^2\big]\\
&\lesssim  \sum_{p=1}^r\|\bu_p - \tu_p\|_2^2\sum_{p=1}^r\|\tv_p\|_2^{2q_2} + \sum_{p=1}^r\|\bv_p - \tv_p\|_2^2\sum_{p=1}^r\|\tu_p\|_2^2\\
&\lesssim  \|\Ub - \tUb\|_F^2\|\tVb\|_F^{2q_2}r^{1-q_2} + \|\Vb - \tVb\|_F^2\|\tUb\|_F^{2q_1}r^{1-q_1}.
\end{align*}
For the second term, from \eqref{f:11} we see $\max_{\|\ba\|_F^2+\|\bb\|_F^2=1}\sqrt{\mE[(\tdT\ba + \tpT\bb)^4]} \lesssim \Upsilon_2^\star$. Combining with the above two displays, and \eqref{f:27} and \eqref{f:25},
\begin{align*}
\|\mE[\nabla^2\mL_1(\Ub, \Vb)] - &\mE[\nabla^2\mL_1(\tUb, \tVb)]\|_2\\
 &\lesssim\beta^3\Upsilon_2^\star\big(\|\Ub - \tUb\|_2^{1-\frac{q_1}{2}} + \|\Vb - \tVb\|_2^{1-\frac{ q_2}{2}} + \|\Ub - \tUb\|_F\|\tVb\|_F^{q_2}r^{\frac{1-q_2}{2}} \\
 &\quad+ \|\Vb - \tVb\|_F\|\tUb\|_F^{q_1}r^{\frac{1-q_1}{2}}\big)\\
 &\lesssim\beta^3(\Upsilon_2^\star)^{3/2}\big(\|\Ub - \tUb\|_F^{1-\frac{q_1}{2}} + \|\Vb - \tVb\|_F^{1-\frac{ q_2}{2}}\big).
\end{align*}
This completes the proof.
\end{proof}

\begin{lemma}\label{lem:E:10}

Under conditions of Lemma \ref{lem:B:4}, we have
\begin{align*}
\T_3 \lesssim \beta^2r^{\frac{1-q}{2}}\rbr{\|\tVb\|_F^{2q} + \|\tUb\|_F^{2q}}\rbr{\|\Ub - \tUb\|_F^{1-q/2} + \|\Vb - \tVb\|_F^{1-q/2}}.
\end{align*}

\end{lemma}

\begin{proof}
We follow the same proof sketch as Lemma \ref{lem:E:9}. By definition of $\T_3$,
\begin{align*}
\|\mE[\nabla^2\mL_2(\Ub, \Vb)] &- \mE[\nabla^2\mL_2(\tUb, \tVb)]\|_2\\
&=  \bigg\|\mE\bigg[B\begin{pmatrix}
\bQ - \bQ' & \bS - \bS'\\
\bS^T - \bS'^T & \bR - \bR'
\end{pmatrix}\bigg] - \mE\bigg[B^\star\begin{pmatrix}
\tQ - \tQ' & \tS - \tS'\\
\tS^T - \tS'^T & \tR - \tR'
\end{pmatrix}\bigg]\bigg\|_2\\
&\leq  \bigg\|\mE\bigg[B\bigg(\begin{pmatrix}
\bQ - \bQ' & \bS - \bS'\\
\bS^T - \bS'^T & \bR - \bR'
\end{pmatrix} - \begin{pmatrix}
\tQ - \tQ' & \tS - \tS'\\
\tS^T - \tS'^T & \tR - \tR'
\end{pmatrix}\bigg)\bigg]\bigg\|_2\\
& \quad + \bigg\|\mE\bigg[(B - B^\star)\begin{pmatrix}
\tQ - \tQ' & \tS - \tS'\\
\tS^T - \tS'^T & \tR - \tR'
\end{pmatrix}\bigg]\bigg\|_2 \\
&\eqqcolon \|\T_{31}\|_2 + \|\T_{32}\|_2.
\end{align*}
For $\T_{31}$,
\begin{align*}
\T_{31} &\lesssim \beta \bigg\|\mE\bigg[\begin{pmatrix}
\bQ - \tQ & \bS - \tS\\
\bS^T - \tS^T & \bR - \tR
\end{pmatrix}\bigg]\bigg\|_2\\
&\lesssim \beta \max_{\|\ba\|_F^2+\|\bb\|_F^2=1}\bigg|\sum_{p=1}^r\mE\big[\big(\phi_1''(\bu_p^T\xb)\phi_2(\bv_p^T\zb) - \phi_1''(\tuT_p\xb)\phi_2(\tvT_p\zb)\big)(\ba_p^T\xb)^2\big]\\
& \quad + 2\sum_{p=1}^r\mE\big[\big(\phi_1'(\bu_p^T\xb)\phi_2'(\bv_p^T\zb) - \phi_1'(\tuT_p\xb)\phi_2'(\tvT_p\zb)\big)(\xb^T\ba_p\bb_p^T\zb)\big]\\
& \quad + \sum_{p=1}^r\mE\big[\big(\phi_1(\bu_p^T\xb)\phi_2''(\bv_p^T\zb) - \phi_1(\tuT_p\xb)\phi_2''(\tvT_p\zb)\big)(\bb_p^T\zb)^2\big] \bigg|\\
&\lesssim \beta\max_{\|\ba\|_F^2+\|\bb\|_F^2=1}\bigg((1-q_1)\sum_{p=1}^r\mE\big[\big(|(\bu_p - \tu_p)^T\xb|\cdot|\tvT_p\zb|^{q_2} + |(\bv_p - \tv_p)^T\zb|\big)\xb^T\ba_p\ba_p^T\xb\big]\\
& \quad+ (1-q_2)\sum_{p=1}^r\mE\big[\big(|(\bv_p - \tv_p)^T\zb|\cdot|\tuT_p\xb|^{q_1} + |(\bu_p - \tu_p)^T\xb|\big)\zb^T\bb_p\bb_p^T\zb\big]\\
& \quad+ \sum_{p=1}^r\mE\big[\big(|\phi_1'(\bu_p^T\xb) - \phi_1'(\tuT_p\xb)| + |\phi_2'(\bv_p^T\zb) - \phi_2'(\tvT_p\zb)|\big)\cdot|\xb^T\ba_p\bb_p^T\zb|\big]\bigg)\\
&\lesssim \beta \max_{\|\ba\|_F^2+\|\bb\|_F^2=1}\bigg((1-q_1)\sum_{p=1}^r\big(\|\bu_p - \tu_p\|_2\|\tv_p\|_2^{q_2} + \|\bv_p - \tv_p\|_2\big)\|\ba_p\|_2^2 \\
& \quad+ (1-q_2)\sum_{p=1}^r\big(\|\bv_p - \tv_p\|_2\|\tu_p\|_2^{q_1}+ \|\bu_p - \tu_p\|_2\big)\|\bb_p\|_2^2\\
& \quad+ \sum_{p=1}^r\|\bu_p - \tu_p\|_2^{1-\frac{q_1}{2}}\|\ba_p\|_2\|\bb_p\|_2  + \sum_{p=1}^r\|\bv_p - \tv_p\|_2^{1-\frac{q_2}{2}}\|\ba_p\|_2\|\bb_p\|_2\bigg)\\
&\lesssim \beta\bigg((1-q_1)(\|\Ub - \tUb\|_2\|\tVb\|_2^{q_2} + \|\Vb - \tVb\|_2) \\
& \quad+ (1-q_2)(\|\Vb - \tVb\|_2\|\tUb\|_2^{q_1} + \|\Ub - \tUb\|_2)+ \|\Ub - \tUb\|_2^{1-\frac{q_1}{2}} + \|\Vb - \tVb\|_2^{1-\frac{q_2}{2}}\bigg)\\
&\lesssim\beta(\|\Ub - \tUb\|_2^{1-\frac{q_1}{2}}\|\tVb\|_2^{q_2(1-q_1)} + \|\Vb - \tVb\|_2^{1-\frac{q_2}{2}}\|\tUb\|_2^{q_1(1-q_2)})\\
&\lesssim\beta(\|\Ub - \tUb\|_2^{1-\frac{q_1}{2}}+ \|\Vb - \tVb\|_2^{1-\frac{q_2}{2}})\Upsilon_7^\star,
\end{align*}
where $\Upsilon_7^\star$ has the same form as $\Upsilon_7$ defined in \eqref{Ups_7} but is calculated using $\tUb, \tVb$. For $\T_{32}$, we use the Lipschitz continuity of $1/(1+\exp(x))$, and simplify analogously to $\J_{32}$. We obtain
\begin{align*}
\|\T_{32}\|_2\lesssim \beta^2(\|\Ub - \tUb\|_F\|\tVb\|_F^{q_2}r^{\frac{1-q_2}{2}} + \|\Vb - \tVb\|_F\|\tUb\|_F^{q_1}r^{\frac{1-q_1}{2}})\Upsilon_7^\star.
\end{align*}
Combining the above three displays,
\begin{align*}
\|\mE[\nabla^2\mL_2(\Ub, \Vb)] - \mE[\nabla^2\mL_2(\tUb&, \tVb)]\|_2 \\
&\lesssim \beta^2\Upsilon_7^\star\big(\|\Ub - \tUb\|_2^{1-\frac{q_1}{2}}+ \|\Vb - \tVb\|_2^{1-\frac{q_2}{2}} \\
&\quad+ \|\Ub - \tUb\|_F\|\tVb\|_F^{q_2}r^{\frac{1-q_2}{2}} + \|\Vb - \tVb\|_F\|\tUb\|_F^{q_1}r^{\frac{1-q_1}{2}}\big)\\
&\lesssim  \beta^2\Upsilon_7^\star\sqrt{\Upsilon_2^\star}\big(\|\Ub - \tUb\|_F^{1-\frac{q_1}{2}} + \|\Vb - \tVb\|_F^{1-\frac{ q_2}{2}}\big).
\end{align*}
We complete the proof.
\end{proof}

\section{Auxiliary Results}\label{appen:7}

\begin{lemma}[Lemma D.4 in \citet{Zhong2018Nonlinear}]\label{lem:G:1}

Let $\Ub\in\mR^{d\times r}$ be a full-column rank matrix. Let $g: \mR^k\rightarrow [0, \infty).$ Define $\barkap(\Ub) =\prod_{p=1}^r\frac{\sigma_p(\Ub)}{\sigma_r(\Ub)} $, then we have
\begin{align*}
\mE_{\xb\in\mN(0, I_d)}g(\Ub^T\xb)\geq \frac{1}{\barkap(\Ub)}\cdot\mE_{\zb\sim\mN(0,I_r)}g(\sigma_r(\Ub)\zb).
\end{align*}

\end{lemma}

\begin{lemma}[Concentration of quadratic form and norm]\label{lem:G:2}

Suppose $\xb_1, \xb_2, \ldots, \xb_n\stackrel{iid}{\sim} \mN(0, I_{d})$ and $\Ub\in\mR^{d\times r}$, then $\forall t>0$
\begin{enumerate}[label=(\alph*),topsep=2pt]
\setlength\itemsep{0em}
\item $P\rbr{\big|\frac{1}{n}\sum_{i=1}^n\xb_i^T\Ub\Ub^T\xb_i - \|\Ub\|_F^2\big|>t}\leq 2\exp(-\frac{nt^2}{4\|\Ub\Ub^T\|_F^2 + 4\|\Ub\|_2^2t})$.
\item $P\rbr{\max_{i\in[n]} |\xb_i^T\Ub\Ub^T\xb_i - \|\Ub\|_F^2| > t}\leq 2n\exp(-\frac{t^2}{4\|\Ub\Ub^T\|_F^2 + 4\|\Ub\|_2^2t})$.
\item $P\big(\big|\frac{1}{n}\sum_{i=1}^n\xb_i^T\Ub\Ub^T\xb_i - \|\Ub\|_F^2\big|>5\sqrt{\frac{s\log d}{n}}\|\Ub\|_F^2\big)\leq \frac{2}{d^s}$, $\forall s\geq 1$.
\item $P\rbr{\max_{i\in[n]} \xb_i^T\Ub\Ub^T\xb_i> (\|\Ub\|_F + 2\sqrt{s\log n}\|\Ub\|_2)^2}\leq \frac{1}{n^{s-1}}$, $\forall s\geq 1$.
\item $P(\xb^T\Ub\Ub^T\xb\geq 6K\|\Ub\|_F^2)\leq \exp(-\frac{\|\Ub\|_F^2K}{\|\Ub\|_2^2})$, $\forall K\geq 1$.
\item $P(\max_{i\in[n]} \big|\|\xb_i\|_2 - \sqrt{d}\big|>t)\leq  2n\exp(-t^2/2)$.
\item $P(\max_{i\in [n]}\big| |\xb_i^T\bu| - \sqrt{\frac{2}{\pi}}\|\bu\|_2\big|>t)\leq 2n\exp(-\frac{t^2}{4\|\bu\|_2^2})$, $\forall \bu\in\mR^d$.
\end{enumerate}

\end{lemma}

\begin{proof}
Result in (a) directly comes from the Chernoff bound and Remark 2.3 in \citet{Hsu2012tail}. We use union bound and (a) to prove (b). (c), (d) and (e) are directly from (a) and (b). (f) is from the Chapter 3 in \citet{Vershynin2018High}. (g) is due to the fact that $|\xb^T\bu|$ is sub-Gaussian variable.
\end{proof}

\begin{lemma}[Expectation of product of quadratic form]\label{lem:G:3}

Suppose $\xb \sim \mN(0, I_d)$, $\Ub\in\mR^{d\times r}$, $\ba, \bb\in\mR^{d}$, then
\begin{enumerate}[label=(\alph*),topsep=2pt]
\setlength\itemsep{0.3em}
\item $\mE[\xb^T\Ub\Ub^T\xb\cdot|\xb^T\ba|]\lesssim \|\Ub\|_F^2\|\ba\|_2$.
\item $\mE[\xb^T\Ub\Ub^T\xb\cdot|\xb^T\ba\bb^T\xb|]\lesssim\|\Ub\|_F^2\|\ba\|_2\|\bb\|_2$.
\item suppose $\Ub_i\in\mR^{d\times r_i}$ for $i\in[4]$, $\mE\big[\prod_{i=1}^{4}\xb^T\Ub_i\Ub_i^T\xb\big]\lesssim\prod_{i=1}^4\|\Ub_i\|_F^2$.
\end{enumerate}

\end{lemma}

\begin{proof}
Note that
\begin{align*}
\mE[\xb^T\Ub\Ub^T\xb\cdot|\xb^T\ba|]&\leq  \sqrt{\mE[(\xb^T\Ub\Ub^T\xb)^2]}\sqrt{\mE[\xb^T\ba\ba^T\xb]}\\
&=  \sqrt{2\TR(\Ub\Ub^T\Ub\Ub^T) + \TR(\Ub\Ub^T)^2} \cdot \|\ba\| \lesssim \|\Ub\|_F^2\|\ba\|.
\end{align*}
This shows the part (a). (b) can be showed similarly using the H\"older's inequality twice. For (c),
\begin{align*}
\mE\big[\prod_{i=1}^{4}\xb^T\Ub_i\Ub_i^T\xb\big]&\leq \prod_{i=1}^{4}\sqrt[4]{\mE\big[(\xb^T\Ub_i\Ub_i^T\xb)^4\big]}\\
&=  \prod_{i=1}^4\big(\|\Ub_i\|_F^8 + 32\|\Ub_i\|_F^2\|\Ub_i\Ub_i^T\Ub_i\|_F^2 + 12 \|\Ub_i\Ub_i^T\|_F^4 \\
& \qquad + 12\|\Ub_i\|_F^4\|\Ub_i\Ub_i^T\|_F^2 + 48\|\Ub_i\Ub_i^T\Ub_i\Ub_i^T\|_F^2\big)^{\frac{1}{4}}\\
&\lesssim \prod_{i=1}^4\|\Ub_i\|_F^2.
\end{align*}
Here the first inequality is due to the H\"older's inequality and the second equality is from Lemma 2.2 in \citet{Magnus1978moments}.
\end{proof}

\begin{lemma}[Extension of Lemma E.13 in \citet{Zhong2018Nonlinear}]\label{lem:G:4}
Let $\D = \{(\xb, \zb)\}$ be a sample set, and let $\Omega = \{(\xb_k, \zb_k)\}_{k = 1}^m$ be a collection of samples of $\D$, where each $(\xb_k, \zb_k)$ is sampled with replacement from $\D$ uniformly. Independently, we have another sets $\D' = \{(\xb', \zb')\}$ and $\Omega' = \{(\xb_k', \zb_k')\}_{k = 1}^m$. For any pair $(\xb, \zb)$ and $(\xb', \zb')$, we have a matrix $\Ab\big((\xb, \zb), (\xb', \zb')\big)\in\mR^{d_1\times d_2}$. Define $\Hb = \frac{1}{m^2}\sum_{k, l=1}^{m}\Ab \rbr{(\xb_k, \zb_k), (\xb_l', \zb_l')}$. If the following conditions hold with $\nu_1, \nu_2$ not depending on $\D$, $\D'$:
\begin{enumerate}[label=(\alph*),topsep=2pt]
	\setlength\itemsep{0.3em}
	\item $\|\Ab\big((\xb, \zb), (\xb', \zb')\big)\|_2\leq \nu_1$, $\forall (\xb, \zb) \in \D, (\xb', \zb') \in \D'$,
	\item ${\begin{aligned}[t]
		\big\|\frac{1}{|\D||\D'|}&\sum_{(\xb, \zb)\in \D}\sum_{(\xb', \zb')\in \D'}\Ab\big((\xb, \zb), (\xb', \zb')\big)\Ab\big((\xb, \zb), (\xb', \zb')\big)^T\big\|_2\\
		&\vee\big\|\frac{1}{|\D||\D'|}\sum_{(\xb, \zb)\in \D}\sum_{(\xb', \zb')\in \D'}\Ab\big((\xb, \zb), (\xb', \zb')\big)^T\Ab\big((\xb, \zb), (\xb', \zb')\big)\big\|_2\leq \nu_2,
		\end{aligned}}$
\end{enumerate}
then $\forall t>0$,
\begin{align*}
P\big(\big\|\Hb - \frac{1}{|\D||\D'|}\sum_{(\xb, \zb)\in \D}\sum_{(\xb', \zb')\in \D'}\Ab\big((\xb, \zb), (\xb', \zb')\big)\big\|_2\geq t\big) 	\leq (d_1+d_2)\exp(-\frac{mt^2}{4\nu_2 + 4\nu_1t}).
\end{align*}
\end{lemma}

\begin{proof}
For any integer $k$, we define $\bar{k}$ to be the remainder of $k/m$ such that $1\leq \bar{k}\leq m$ (i.e. $\bar{m} = m$). Then we can express $\Hb$ as
\begin{align*}
\Hb = \frac{1}{m}\sum_{k=0}^{m-1}\rbr{\frac{1}{m}\sum_{l=1}^{m}\Ab\rbr{(\xb_l, \zb_l), (\xb_{\overline{l+k}}', \zb_{\overline{l+k}}')}} \eqqcolon \frac{1}{m}\sum_{k=0}^{m-1}\Hb_k.
\end{align*}
Note that $\Hb_k$ is the sum of $m$ independent samples, and for any $k=0,1,...,m-1$, they have the same distribution with conditional expectation
\begin{align*}
\mE[\Hb_k \mid \D, \D'] = \frac{1}{|\D||\D'|}\sum_{(\xb, \zb)\in \D}\sum_{(\xb', \zb')\in \D'}\Ab\rbr{(\xb, \zb), (\xb', \zb')}.
\end{align*}
Therefore,
\begin{align*}
P(\|\Hb - \mE[\Hb]\|_2> t \mid \D, \D') \leq & P(\frac{1}{m}\sum_{k=0}^{m-1}\|\Hb_k - \mE[\Hb_k]\|_2>t \mid \D, \D')\\
\leq & \inf_{s>0} e^{-st}\mE[\exp(\frac{s}{m}\sum_{k=0}^{m-1}\|\Hb_k - \mE[\Hb_k]\|_2) \mid \D, \D']\\
\leq &\inf_{s>0} e^{-st}\frac{1}{m}\sum_{k=0}^{m-1}\mE[\exp(s\|\Hb_k - \mE[\Hb_k]\|_2) \mid \D, \D'] \\
= &\inf_{s>0} e^{-st}\mE[\exp(s\|\Hb_0 - \mE[\Hb_0]\|_2) \mid \D, \D'].
\end{align*}
By the proof of Corollary 6.1.2 in \citet{Tropp2015introduction}, the right hand side satisfies
\begin{align*}
\inf_{s>0} e^{-st}\mE[\exp(s\|\Hb_0 - \mE[\Hb_0]\|_2) \mid \D, \D'] \leq (d_1+d_2)\exp(-\frac{mt^2}{4\nu_2 + 4\nu_1t}).
\end{align*}
Combining the above two displays and using the fact that $P(\mA) = \mE\sbr{\b1_{\mA}} = \mE\sbr{\mE\sbr{\b1_{\mA} \mid \D, \D'}}$ for any event $\mA$, we finish the proof.
\end{proof}

\begin{lemma}[Extension of Lemma E.10 in \citet{Zhong2018Nonlinear}]\label{lem:G:5}
Let $\D = \{(\xb_i, \zb_j) \sim \mF : i \in[n_1], j\in[n_2]\}$ be a sample set with size $n_1n_2$ and each pair $(\xb, \zb)$ follows the same distribution $\mF$; similarly but independently, let $\D' = \{(\xb'_i, \zb'_j) \sim \mF' : i \in[n_1], j\in[n_2]\}$ be another sample set. Let $\Ab\rbr{(\xb, \zb), (\xb', \zb')}\in\mR^{d_1\times d_2}$ be a random matrix corresponding to $(\xb, \zb)\in \D$, $(\xb', \zb')\in \D'$, and let $\Hb = \frac{1}{n_1^2n_2^2}\sum_{(\xb, \zb)\in \D}\sum_{(\xb', \zb')\in \D'}\Ab\rbr{(\xb, \zb), (\xb', \zb')}$. Suppose the following conditions hold with $\mu_1, \nu_1, \nu_2, \nu_3$:
\begin{enumerate}[label=(\alph*),topsep=2pt]
\setlength\itemsep{0.3em}
\item $P\rbr{\|\Ab\rbr{(\xb, \zb), (\xb', \zb')}\|_2\geq \mu_1}  \leq  \nu_1$,
\item ${\begin{aligned}[t]
\big\|\mE[\Ab\big((\xb, \zb), (\xb', \zb')\big)&\Ab\big((\xb, \zb), (\xb', \zb')\big)^T]\big\|_2 \\
 & \vee \big\|\mE[\Ab\big((\xb, \zb), (\xb', \zb')\big)^T\Ab\big((\xb, \zb), (\xb', \zb')\big)]\big\|_2 \leq \nu_2,
\end{aligned}}$
\item $\max_{\|\bu\|_2=\|\bv\|_2=1} \big(\mE\big[\rbr{\bu^T\Ab\big((\xb, \zb), (\xb', \zb')\big)\bv}^2\big]\big)^{1/2} \leq \nu_3$,
\end{enumerate}
then $\forall t>0$,
\begin{multline*}
P\rbr{\|\Hb - \mE[\Hb]\|_2>t + \nu_3\sqrt{\nu_1}}\\
\leq n_1^2n_2^2\nu_1 + (d_1+d_2)\exp\rbr{-\frac{(n_1\wedge n_2) t^2}{\big(2\nu_2 + 4\|\mE[\Hb]\|_2^2 + 4\nu_3^2\nu_1\big) + 4\mu_1 t}}.
\end{multline*}

\end{lemma}

\begin{proof}
For simplicity we suppress the evaluation point of $\Ab$. Let
 $\bar{\Ab} = \Ab\cdot\b1_{\|\Ab\|_2\leq \mu_1}$ and $\bar{\Hb} =  \frac{1}{n_1^2n_2^2}\sum_{(\xb, \zb)\in \D}\sum_{(\xb', \zb')\in \D'} \bar{\Ab}$. Then,
\begin{align*}
\|\Hb - \mE[\Hb]\|_2\leq \|\Hb - \bar{\Hb}\|_2 + \|\bar{\Hb} - \mE[\bar{\Hb}]\|_2 + \|\mE[\bar{\Hb}] - \mE[\Hb]\|_2.
\end{align*}
For the first term,
\begin{align*}
P(\|\Hb - \bar{\Hb}\|_2 = 0)\geq P(\Ab = \bar{\Ab}, \forall (\xb, \zb)\in \D, (\xb', \zb')\in \D')\geq 1 - n_1^2n_2^2\nu_1.
\end{align*}
For the third term,
\begin{align*}
\|\mE[\bar{\Hb}] - \mE[\Hb]\|_2
& = \|\mE[\Ab\cdot\b1_{\|\Ab\|_2>\mu_1}]\|_2 \\
& = \max_{\|\bu\|_2=\|\bv\|_2=1} \mE\big[\bu^T\Ab\bv\cdot\b1_{\|\Ab\|_2>\mu_1}\big]\\
& \leq \max_{\|\bu\|_2=\|\bv\|_2=1} \sqrt{\mE[(\bu^T\Ab\bv)^2]} \sqrt{P(\|\Ab\|_2>\mu_1)} \\
& \leq \nu_3\sqrt{\nu_1}.
\end{align*}
For the second term, without loss of generality, we assume $n_1\leq n_2$.
For any integer $k$, we let $k = s_1n_1 + \bar{k}$,
where integer $s_1\geq 0$ and remainder $\bar{k}$ satisfies $1\leq \bar{k}\leq n_1$.
We also let $k = s_2n_2 + \tilde{k}$,
where integer $s_2\geq 0$ and $\tilde{k}$ satisfies $1\leq \tilde{k}\leq n_2$.
Then
\begin{align*}
\bar{\Hb} = \frac{1}{n_2^2}\sum_{k=0}^{n_2-1}\sum_{l=0}^{n_2-1}\frac{1}{n_1}\sum_{j=0}^{n_1-1}\bigg(\underbrace{\frac{1}{n_1}\sum_{i=1}^{n_1}\bar{\Ab}\big((\xb_i, \zb_{\tilde{i+k}}), (\xb_{\overline{i+j}}', \zb'_{\tilde{\overline{i+j}+l}})\big)}_{\bar\Hb_{k,l,j}}\bigg).
\end{align*}
Based on this decomposition, we see that $\bar{\Hb}_{k,l,j}$ is a sum of $n_1$ i.i.d.~random matrices and that $\{\bar{\Hb}_{k,l,j}\}$ have the same distribution. Similar to the proof of Lemma \ref{lem:G:4}, we have
\begin{align*}
P\rbr{\|\bar{\Hb} - \mE[\bar{\Hb}]\|_2>t}\leq \inf_{s>0} e^{-st}\mE[\exp(s\|\bar{\Hb}_{0,0,0} - \mE[\bar{\Hb}_{0,0,0}]\|_2)].
\end{align*}
We apply Corollary 6.1.2 in \citet{Tropp2015introduction}.
Note that $\|\bar{\Ab} - \mE[\bar{\Ab}]\|_2\leq 2\mu_1$ and
\begin{align*}
\|\mE[\bar{\Ab}\bar{\Ab}^T] - \mE[\bar{\Ab}]\mE[\bar{\Ab}^T]\|_2
&  \leq \|\mE[\Ab\Ab^T]\|_2 + \|\mE[\bar{\Ab}]\|_2^2 \\
& \leq \nu_2 + (\|\mE[\Hb]\|_2 + \nu_3\sqrt{\nu_1})^2\\
& \leq \nu_2 + 2\|\mE[\Hb]\|_2^2 + 2\nu_3^2\nu_1.
\end{align*}
A similar bound holds for $\|\mE[\bar{\Ab}^T\bar{\Ab}] - \mE[\bar{\Ab}^T]\mE[\bar{\Ab}]\|_2$. Therefore
\begin{multline*}
\inf_{s>0} e^{-st}\mE[\exp(s\|\bar{\Hb}_{k,l,j} - \mE[\bar{\Hb}_{k,l,j}]\|_2)]\\
\leq (d_1+d_2)\exp\rbr{-\frac{n_1 t^2}{\big(2\nu_2 + 4\|\mE[\Hb]\|_2^2 + 4\nu_3^2\nu_1\big) + 4\mu_1 t}}.
\end{multline*}
Putting everything together finishes the proof.
\end{proof}

\begin{lemma}\label{lem:G:6}
Let $\D = \{(\xb_i, \zb_j) \sim \mF : i \in[n_1], j\in[n_2]\}$. Let $\Ab\rbr{\xb, \zb}\in\mR^{d_1\times d_2}$ be a random matrix corresponding to $(\xb, \zb) \in \D$ and let $\Hb = \frac{1}{n_1n_2}\sum_{(\xb, \zb)\in \D}\Ab\rbr{\xb, \zb}$. Suppose the following conditions hold with $\mu_1, \nu_1, \nu_2, \nu_3$:
\begin{enumerate}[label=(\alph*),topsep=2pt]
\setlength\itemsep{0.3em}
\item $P\rbr{\|\Ab(\xb, \zb)\|_2\geq \mu_1} \leq  \nu_1$,
\item ${\begin{aligned}[t]
\big\|\mE\big[\Ab(\xb, \zb)\Ab(\xb, \zb)^T\big]\big\|_2 \vee \big\|\mE\big[\Ab(\xb, \zb)^T\Ab(\xb, \zb)\big]\big\|_2 & \leq \nu_2,
\end{aligned}}$
\item $\max_{\|\bu\|_2=\|\bv\|_2=1} \big(\mE\big[\rbr{\bu^T\Ab(\xb, \zb)\bv}^2\big] \big)^{1/2} \leq \nu_3$,
\end{enumerate}
then $\forall t>0$,
\begin{multline*}
P\rbr{\|\Hb - \mE[\Hb]\|_2>t + \nu_3\sqrt{\nu_1}}\\
\leq n_1n_2\nu_1 + (d_1+d_2)\exp\rbr{-\frac{(n_1\wedge n_2) t^2}{\big(2\nu_2 + 4\|\mE[\Hb]\|_2^2 + 4\nu_3^2\nu_1\big) + 4\mu_1 t}}.
\end{multline*}

\end{lemma}

\begin{proof}
The result follows directly from Lemma \ref{lem:G:5}.
\end{proof}

\begin{lemma}\label{lem:G:7}

Suppose $\xb\sim\mN(0, I_d)$, $\phi \in \{\text{sigmoid}, \text{tanh}, \text{ReLU}\}$. For any $\bu, \tu, \ba, \bb \in\mR^d$,
\begin{align*}
\mE[|\phi'(\bu^T\xb) - \phi'(\tuT\xb)|\cdot|\xb^T\ba\bb^T\xb|]\leq \bigg(\sqrt{\frac{\|\bu - \tu\|_2}{\|\tu\|_2}}\bigg)^{q}\|\bu - \tu\|_2^{1-q}\|\ba\|_2\|\bb\|_2,
\end{align*}
where $q = 1$ if $\phi$ is ReLU and $q = 0$ otherwise.
\end{lemma}

\begin{proof}
By H\"older's inequality,
\begin{align*}
\mE[|\phi'(\bu^T\xb) - \phi'(\tuT\xb)|\cdot|\xb^T\ba\bb^T\xb|]&\leq  \sqrt{\mE[(\phi'(\bu^T\xb) - \phi'(\tuT\xb))^2\xb^T\ba\ba^T\xb]}\sqrt{\mE[\xb^T\bb\bb^T\xb]}.
\end{align*}
If $\phi\in\{\text{sigmoid}, \text{tanh}\}$, we finish the proof by using the Lipschitz continuity of $\phi'$ and Lemma \ref{lem:G:3}. If $\phi$ is ReLU, we apply Lemma E.17 in \citet{Zhong2018Nonlinear} to complete the proof.
\end{proof}

\bibliography{paper}

\end{document}